\documentclass[letterpaper,11pt]{article}
\usepackage[letterpaper,top=1in,bottom=1in,left=1in,right=1in]{geometry}

\usepackage{graphicx}% Include figure files
\usepackage{dcolumn}% Align table columns on decimal point
\usepackage{bm}% bold math
\usepackage{bbm}
\usepackage{hyperref}
\usepackage{amsmath,amssymb,amsfonts,amsthm}
\usepackage{color}
\usepackage[ruled,linesnumbered]{algorithm2e}
\usepackage{caption, subcaption}
\usepackage{color-edits}
\usepackage{thm-restate}
\usepackage{authblk}
\usepackage{nicefrac}
\usepackage{multirow,natbib}
\usepackage{color-edits}
\usepackage{titlesec}
\usepackage{parskip}

\hypersetup{
  colorlinks   = true, %Colours links instead of ugly boxes
  urlcolor     = red, %Colour for external hyperlinks
  linkcolor    = black, %Colour of internal links
  citecolor   = black %Colour of citations, could be ``red''
}

\addauthor{ev}{blue}
\addauthor{joon}{red}
\addauthor{kk}{magenta}

\newtheorem{lemma}{Lemma}[section]

\newtheorem{assumption}{Assumption}

\newcommand\numberthis{\addtocounter{equation}{1}\tag{\theequation}}

\titleformat*{\paragraph}{\bfseries}

\def \A {\mathcal{A}}
\def \C {\mathcal{C}}

\def \O {\mathcal{O}}
\def \P {\mathcal{P}}

\def \I {\mathcal{I}}

\def \S {\mathcal{S}}
\def \E {\mathcal{E}}

\def \eps {\epsilon}
\def \del {\delta}

\def \nat {\mathbb{N}}
\def \rr {\mathbb{R}}

\DeclareMathOperator*{\expec}{\mathbb{E}}

\newcommand{\innp}[1]{\left\langle #1 \right\rangle}
\newcommand{\cbra}[1]{\left\{ #1 \right\}}
\newcommand{\sbra}[1]{\left[ #1 \right]}
\newcommand{\rbra}[1]{\left( #1 \right)}

\newcommand{\ind}{\mathbbm{1}}
\newcommand{\bunderbrace}[2]{%
  \begin{array}[t]{@{}c@{}}
  \underbrace{#1}\\
  #2
  \end{array}
}

\allowdisplaybreaks

\def \alg {\textup{ALG}}
\def \opt {\textup{OPT}}
\def \ct {c_t}
\def \cti {c_{t,i}}
\def \qti {q_{t,i}}
\def \qtz {q_{t,0}}

\def \pt {p_t}
\def \dti {d_{t,i}}
\def \edti {\overline{d}_{t,i}}
\def \ed {\overline{d}}
\def \rti {r_{t,i}}
\def \erti {\overline{r}_{t,i}}
\def \profit {\textup{Profit}}
\def \eprofit {\overline{\textup{Profit}}}
\def \lti {\ell_{t,i}}
\def \lt {\ell_t}
\def \elti {\overline{\ell}_{t,i}}
\def \elt {\overline{\ell}_t}
\def \lossF {\widehat{f}}
\def \lossH {\widehat{h}}
\def \uniform {\mathsf{Uniform}}
\def \envbase {\E_\textup{base}}
\def \envcstarpstar {\E_{(c^\star,p^\star)}}
\def \eeAbase {\mathbb{E}^\A_\textup{base}}
\def \eeAcstarpstar {\mathbb{E}^\A_{(c^\star,p^\star)}}
\def \ppAbase {\mathbb{P}^\A_\textup{base}}
\def \ppAcstarpstar {\mathbb{P}^\A_{(c^\star,p^\star)}}

\begin{document}

\title{Bandit Profit-Maximization for Targeted Marketing}

\author[1]{Joon Suk Huh}
\author[2,3]{Ellen Vitercik}
\author[1]{Kirthevasan Kandasamy}
\affil[1]{Computer Science Department, UW--Madison}
\affil[2]{Management Science and Engineering Department, Stanford University}
\affil[3]{Computer Science Department, Stanford University}
\date{\vspace{0ex}}

\maketitle

% Paper body
\begin{abstract}
    We study a sequential profit-maximization problem, optimizing for both price and ancillary variables like marketing expenditures. Specifically, we aim to maximize profit over an arbitrary sequence of multiple demand curves, each dependent on a distinct ancillary variable, but sharing the same price.
    A prototypical example is targeted marketing, where a firm (seller) wishes to sell a product over multiple markets.
    The firm may invest different marketing expenditures for different markets to optimize customer acquisition, but must maintain the same price across all markets.
    Moreover, markets may have heterogeneous demand curves, each responding to prices and marketing expenditures differently.
    The firm's objective is to maximize its gross profit, the total revenue minus marketing costs. 

    Our results are near-optimal algorithms for this class of problems in an adversarial bandit setting, where demand curves are arbitrary non-adaptive sequences, and the firm observes only noisy evaluations of chosen points on the demand curves. For $n$ demand curves (markets), we prove a regret upper bound of $\widetilde{\O}\big(nT^{\nicefrac{3}{4}}\big)$ and a lower bound of $\Omega\big((nT)^{\nicefrac{3}{4}}\big)$ for monotonic demand curves, and a regret bound of $\widetilde{\Theta}\big(nT^{\nicefrac{2}{3}}\big)$ for demand curves that are monotonic in price and concave in the ancillary variables. %Interestingly, even when demand curves do not depend on ancillary variables, the regret lower bound of $\Omega(nT^{\nicefrac{2}{3}})$ is necessary. Thus, our algorithm solves the targeted marketing problem for free when demand curves are concave in ancillary variables.
    % Not sure add this to ArXiv V2
\end{abstract}
\section{Introduction}\label{sec:intro}
The design of revenue-maximizing mechanisms is one of the most important problems in economics. This problem is appealingly simple in the single-item setting: it boils down to choosing a revenue-maximizing price $p$. When the market is characterized by a demand curve $d(p)$, as in Figure~\ref{fig:1a}, then the revenue is $p\cdot d(p)$. Thus, finding the revenue-maximizing price under a known, fixed demand curve is straightforward. In reality, however, a firm (seller) will have only noisy, incomplete information about the demand curve.
As a result, an explosion of research has studied the more realistic setting where the demand curve is unknown, and the firm must learn the agents’ willingness to pay from repeated interactions~\citep[e.g.,]{kleinberg2003value,besbes2009dynamic,den2014simultaneously,besbes2015surprising,cheung2017dynamic,misra2019dynamic,den2015dynamic}.

However, this body of literature at the intersection of machine learning and mechanism design has not taken into account a critical lever of power that the firm has in many markets: the firm can shift the demand curve through advertising as illustrated in Figure~\ref{fig:1b},
resulting in higher revenue. This
phenomenon is known as the \emph{advertising elasticity of demand}~\citep{png2022managerial,choi2020online}. 

Moreover, different markets respond differently to advertising and/or price.
For instance, advertising for a luxury car will likely have more impact in affluent markets, while for non-luxury cars, advertising may be more impactful in emerging markets.
Thus, if a firm discovers which markets are more profitable over time, they can concentrate their advertising expenditures on those markets, as illustrated in Figure~\ref{fig:1c}.
In most real-world settings, the firm cannot practice non-anonymous (discriminatory) pricing, i.e., it must choose a
\emph{common price} for all markets; otherwise, buyers could buy across markets to pay a lower price.
Hence, the firm faces a complex, multi-dimensional profit maximization problem, where they must optimize the advertising costs in each market 
and the common price to simultaneously maximize revenue while minimizing advertising costs.

In this work, we formalize this problem with a bandit-learning model where the firm interacts with buyers from $n$ local markets over a series of $T$ timesteps. On each round $t \in [T]$, the firm chooses a common, anonymous price $p_t$, and for each market $i \in [n]$, it chooses an ancillary cost $c_{t,i}$, representing, for example, advertising spend in the market. 
Each local market $i \in [n]$ has demand $d_{t,i}$, which is a function of the price $p$ and the ancillary cost $c_{t,i}$. Thus, the firm's revenue on round $t$ is $\sum_{i = 1}^n p_t d_{t,i}(c_{t,i}, p_t) - c_{t,i}.$ The goal is to minimize \emph{regret}, which is the difference between the firm's cumulative revenue and that of the optimal price and ancillary costs in hindsight.

\paragraph{Variants.}
In Section~\ref{sec:otherprobs}, we also present variants of this targeted marketing problem for which the above framework applies.
The first is the \textit{subscription problem}, where new customers subscribe to a service on each round $t$,
and may choose to stay for future rounds.
Hence, there is a \textit{memory} effect where the demands at round $t$ can depend on past marketing costs and prices.
In each round, there is an influx of new customers depending on the current price level and marketing expenditures, and a fraction of existing customers canceling their membership.
The firm wishes to maximize its revenue by maintaining a large active customer pool.
The second variant is the \textit{promotional credit problem}, where a service provider (e.g., cloud services) wishes to attract users by giving promotional credits to users depending on whether they belong to a population segment (e.g., students, developers).
The third variant is a \textit{profit-maximizing A/B test}, where the firm conducts a series of experiments (e.g., presenting one of $k$ webpages) over $n$ segments of the population in conjunction with non-discriminatory pricing to maximize their profit. 

\begin{figure}[t]
    \centering
    \begin{subfigure}[t]{0.3\textwidth}
        \centering
        \includegraphics[width=\textwidth]{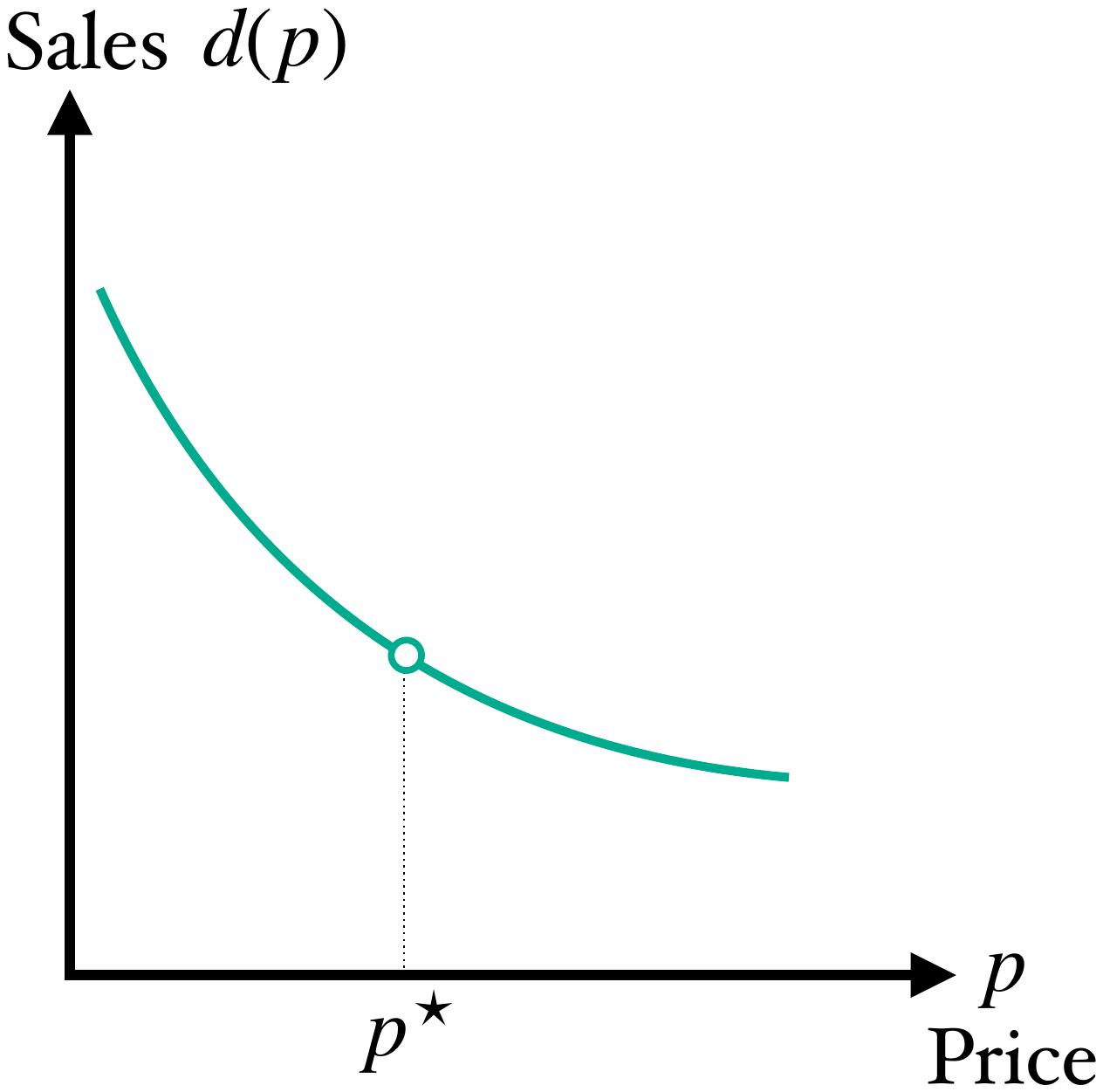}
        \caption{Single market with a fixed demand curve.}
        \label{fig:1a}
    \end{subfigure}
    \hfill
    \begin{subfigure}[t]{0.3\textwidth}
        \centering
        \includegraphics[width=\textwidth]{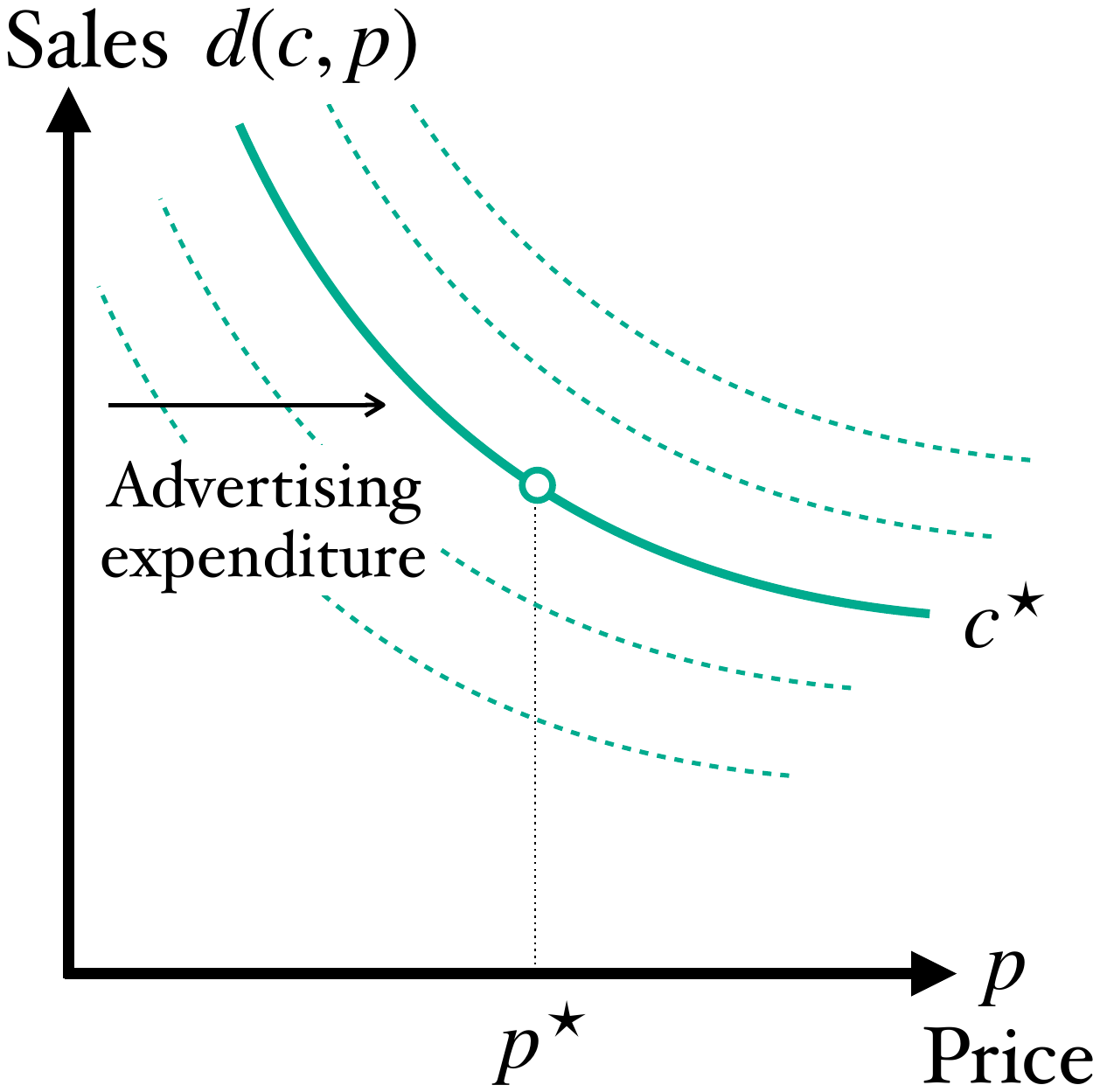}
        \caption{Single market with a shifting demand curve.}
        \label{fig:1b}
    \end{subfigure}
    \hfill
    \begin{subfigure}[t]{0.3\textwidth}
        \centering
        \includegraphics[width=\textwidth]{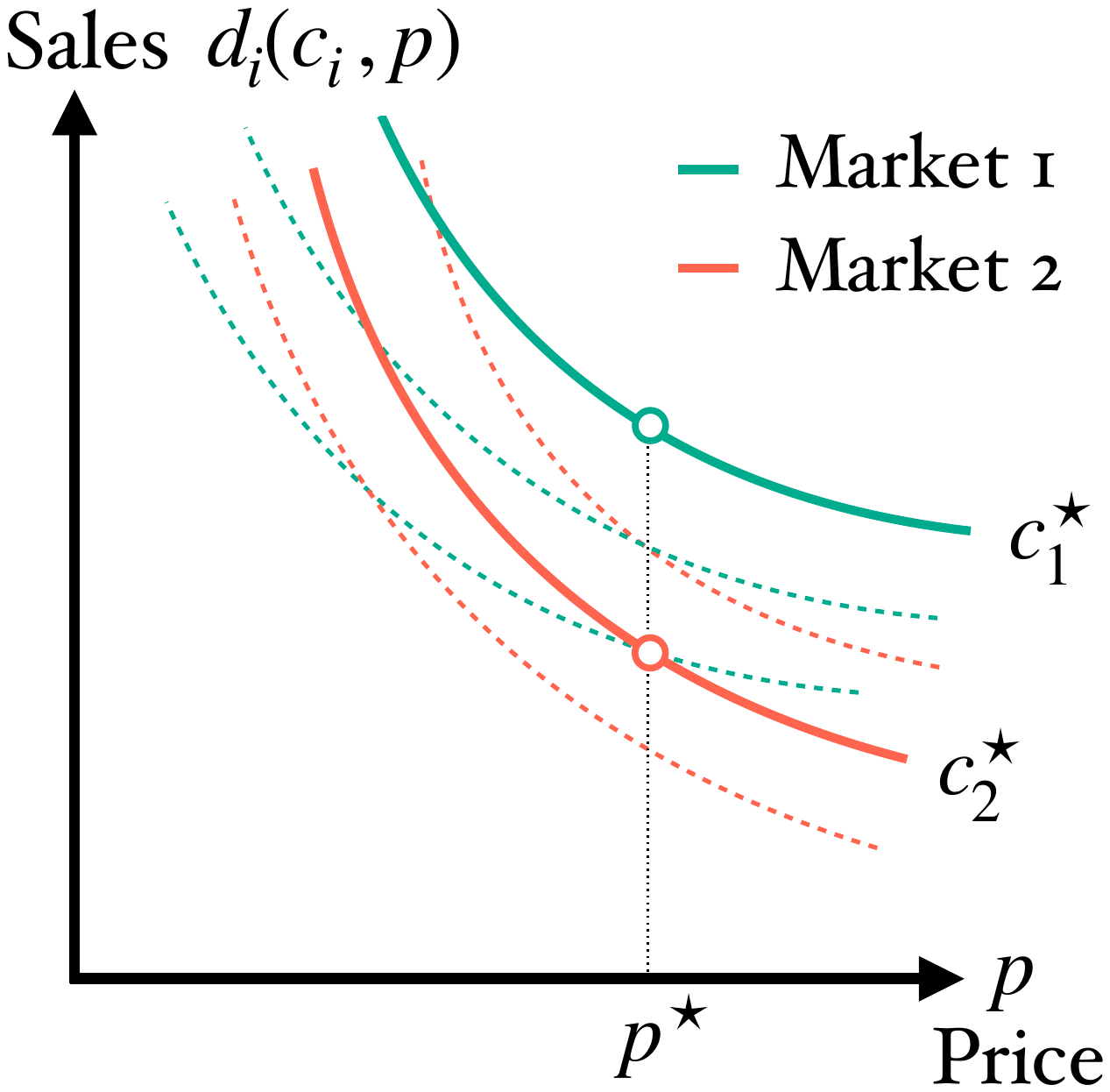}
        \caption{Multiple markets with shifting demand curves.}
        \label{fig:1c}
    \end{subfigure}
    \caption{A landscape of profit-maximization problems.
    In (a), we wish to maximize revenue under some demand curve $d(p)$, which boils down to choosing a price $p$ which maximizes $p\cdot d(p)$.
    In (b), advertising can shift the demand curve, and the goal is to maximize the profit $p\cdot d(c,p) - c$, i.e., revenue minus advertising cost $c$.
    The setting of this work is illustrated in (c), where we have $n$ different markets, and we wish to
    choose advertising costs $c_1,\dots,c_n$ and a
    \emph{common price} $p$ to maximize the total profit
    $\sum_{i}p\cdot d_i(c_i,p) - c_i$.
    The demand curves are unknown \textit{a priori}, 
    and we are interested in learning the optimal price and costs via repeated interactions.
    }
    \label{fig:1}
    \vspace{-1em}
\end{figure}

\subsection{Our contributions}
\label{sec:contributions}
Our main contributions in this work are as follows.
\begin{enumerate}
    \item We formalize profit maximization in the adversarial bandit setting where the firm only observes the realized demands (i.e., sales) for the chosen price and marketing costs. In particular, we study two different, yet natural assumptions on the demand curves: (1) Monotonic demands (Assumption \ref{assump:monotonic}), where the demands for each market are monotonically increasing in the marketing expenditure and decreasing in price. (2) Cost-concave demands (Assumption \ref{assump:costconcave}), where the demand is a concave function of the marketing expenditure, modeling diminishing returns in the effectiveness of the marketing campaign.
    \vspace{0.5em}
    \item
    We provide two profit-maximizing algorithms for the targeted marketing problem, under the two different assumptions. Our regret bounds are linear in the number of local markets, thus avoiding the curse of dimensionality. Specifically, we prove regret bounds of $\widetilde{\O}\!\rbra{nT^{\nicefrac{3}{4}}}$ for monotonic demands and $\widetilde{\O}\!\rbra{nT^{\nicefrac{2}{3}}}$ for cost-concave demands. 
    Surprisingly, for cost-concave demands, our regret bound matches, up to logarithmic factors, well-known upper and lower bounds for pricing without shifting demand curves, i.e. single-item pricing~\citep{kleinberg2003value}.
    Our theoretical results are summarized in Table~\ref{tab:results}.
    \vspace{0.5em}
    \item We prove matching and nearly-matching regret lower bounds. Under monotonic demands, we show any algorithm has to incur $\Omega\!\rbra{(nT)^{\nicefrac{3}{4}}}$ regret, meaning that our algorithm is optimal up to an $n^{\nicefrac{1}{4}}$ term. For cost-concave demands, the lower bound is $\Omega\!\rbra{nT^{\nicefrac{2}{3}}}$, which follows almost directly from the lower bound proved by \citet{kleinberg2003value}.
    \vspace{0.5em}
    \item We show that our algorithms can be adapted to solve the aforementioned variants of the targeted marketing problem, without modifying their internal logic. In Section \ref{sec:otherprobs} and Appendix~\ref{app:otherprobs}, we formalize these variants and outline our reductions.
\end{enumerate}

\paragraph{Key challenges and insights.}
A key technical challenge in our problem is that the firm needs to choose a common price across all markets.
If this were not the case, i.e., if the firm could choose a non-anonymous price $p_{t,i}$ for each market $i$, the targeted marketing problem would reduce to $n$ separate, yet simple, two-variable problems: maximizing $p_{t,i}\cdot \edti(c_{t,i},p_{t,i})-c_{t,i}$, for each market $i$.
At the same time, naively treating this as a  $(n+1)$--dimensional bandit optimization problem, say by applying an  algorithm such as EXP3 \citep{auer1995gambling}, fails to exploit the problem structure,
and leads to a regret bound that is exponential in the number of local markets $n$.
In this work, we show that one can carefully decompose the targeted marketing problem into $n$ local cost optimizations, and one price optimization problem, leading to a regret bound that is linear in the number of markets.
Interestingly, this is essentially the same bound as the non-anonymous pricing case.

\begin{table}[t]
\center
\begin{tabular}{|c|c|c|}
\hline
Demand types                                                                           & Regret bounds & References                                                                \\ \hline
\begin{tabular}[c]{@{}c@{}}Non-shifting \\ demands\end{tabular}                        & $\widetilde{\Theta}(nT^{\nicefrac{2}{3}})$           & \begin{tabular}[c]{@{}c@{}}\citet{kleinberg2003value}\\ Theorem \ref{thm:lb-concave}\end{tabular} \\ \hline
\multirow{2}{*}{\begin{tabular}[c]{@{}c@{}}Monotonic\\ (Assumption 1)\end{tabular}}    & Upper bound: $\O\!\rbra{nT^{\nicefrac{3}{4}}\log T}$  & Theorem \ref{thm:alg1-regret}, Alg. 1 (Ours)                                               \\ \cline{2-3} 
                                                                                       & Lower bound: $\Omega\!\rbra{(nT)^{\nicefrac{3}{4}}}$ & Theorem \ref{thm:lb-monotonic}  (Ours)                                                             \\ \hline
\multirow{2}{*}{\begin{tabular}[c]{@{}c@{}} Cost-concave\\ (Assumption 2)\end{tabular}} & Upper bound: $\O\!\rbra{nT^{\nicefrac{2}{3}}\log T}$ & Theorem \ref{thm:alg2-regret}, Alg. 2 (Ours)                                               \\ \cline{2-3} 
                                                                                       & Lower bound: $\Omega\!\rbra{nT^{\nicefrac{2}{3}}}$ &
                                                                                       \begin{tabular}[c]{@{}c@{}}Theorem \ref{thm:lb-concave}   
                                                                                       %(Adapted from \\ \citet{kleinberg2003value})
                                                                                       \end{tabular}
                                                                                                                                                \\ \hline
\end{tabular}
\vspace{0.5em}
\caption{Regret bounds of bandit profit-maximization problems, where our regret is defined in \eqref{eq:regret} Here, ``Non-shifting demands'' stands for the case when demand curves do not depend on ancillary variables, which is the usual profit-maximization without marketing contexts.}
\label{tab:results}
\end{table}
\paragraph{Algorithm and regret upper bound proof overview.}
We summarize the main ideas underlying the design of our algorithms: Algorithm~\ref{alg:1} in Section~\ref{subsec:alg-monotnic} for monotonic demands and Algorithm~\ref{alg:2} in Section~\ref{subsec:alg-concave} for cost-concave demands.
\begin{enumerate}
    \item \label{itm:algintuition}
    First, let us assume that the firm is aware
    of the demand curves $\dti$ on each round. 
    Instead of viewing this as an $n+1$ dimensional optimization problem, we can decompose the firm's profit $\sum_i (\pt \dti(\cti,\pt) -\cti)$ as follows.
    For a fixed price, we can find the best costs via $n$ separate one-dimensional optimization problems.
    We can evaluate each price in this manner and then choose the optimal price (along with the costs).
    Our algorithms are designed to exploit this decomposable structure. This decomposition into smaller optimization problems is statistically advantageous as an algorithm needs to track significantly fewer parameters.
    \vspace{0.5em}
    \item To leverage the above structure for monotonic demands, our algorithm--which builds on EXP3--maintains one distribution for the price; then, for each price in a discretized set, it maintains $n$ distributions, one for each cost. The main challenge in applying the intuition in~(\ref{itm:algintuition}) is in the design of appropriate objectives to update these distributions. Directly using unbiased estimators of local profits as objective functions---like in EXP3---tends to exploit too aggressively.
    The problem is that
    these unbiased estimators tend to have large variances,
    resulting in potentially large ($\Omega(T)$) regret. To address this issue, we carefully design (biased) objectives that have lower variance and encourage exploration.
    \vspace{0.5em}
    \item For cost-concave demands, we use the same decomposition, but leverage ideas from bandit convex optimization \cite{flaxman2005online,hazan2014bandit}. In particular, we use a similar update rule for the price distributions but adopt a modified version of the kernelized exponential weights update rule~\citep{bubeck2017kernel} for cost distributions.
\end{enumerate}

\paragraph{Regret lower bound proof overview.}
We prove a regret lower bound of $\Omega\!\rbra{(nT)^{\nicefrac{3}{4}}}$ for the targeted marketing problem under Assumption \ref{assump:monotonic}, where the demands are monotonically increasing (decreasing) with respect to the marketing expenditure (the common price).
Following standard approaches~\citep{bubeck2012regret,lattimore2020bandit},
we reduce the adversarial bandit problem to the stochastic setting and apply hypothesis testing arguments to construct the lower bound.
The main challenge is in the design of alternatives for hypothesis testing that are statistically indistinguishable but have large differences in profits. 
Due to the structure of our problem, we require the demand curves in our alternatives to be monotonically decreasing in price, monotonically increasing in cost, and be coupled via the common price.

\subsection{Related work}
\paragraph{Dynamic pricing.}
Bandit profit-maximization without shifting demands has long been studied in the context of dynamic pricing~\citep{den2015dynamic}. One of the seminal works in this area is \citet{kleinberg2003value}, who studied a setting where a single new buyer appears at each round who will purchase the product if the price $p_t$ is lower than her value $v_t$. This conforms to our setting, where the demand function is $d_t(p)=\ind\!\sbra{v_t\geq p}$. They proved a regret lower bound of $\Omega(T^{2/3})$, and showed that a straightforward adaptation of EXP3 achieves $\widetilde{\O}(T^{2/3})$ regret.

The setting of \citet{kleinberg2003value} was later generalized to more general classes of non-parametric demands~\citep{besbes2009dynamic,wang2021multimodal,cheung2017dynamic,misra2019dynamic,perakis2023dynamic} and parametric demand models~\citep{keskin2014dynamic,den2014simultaneously,besbes2015surprising,javanmard2017perishability,javanmard2019dynamic,xu2021logarithmic}. However, except for \citet{kleinberg2003value}, all work assumed that the underlying distribution of demands is fixed over time. 
In this work, we allow the underlying demand distributions to be any sequence over time, reflecting drifts that may occur in real markets.

\paragraph{Bandits for marketing.}
Besides the pricing problem, various marketing problems have been studied in the bandit setting. \citet{schwartz2017customer} studied how the effectiveness of online advertisements can be enhanced by the real-world deployment of a multi-armed bandit algorithm. \citet{urban2014morphing} reformulated the traditional A/B test and suggested an online algorithm for A/B test in the profit-maximizing scenario. \citet{sawant2018contextual} developed a contextual multi-armed bandit algorithm for marketing by harnessing the underlying causal effects of marketing. Finally, it is worth mentioning that there is a series of works on morphing websites for customer acquisition \citep{hauser2009website,hauser2014website,urban2014morphing,liberali2022morphing}, related to one of the variations (Profit-maximizing A/B test) of the targeted marketing problem.

To the best of our knowledge, no prior work studies the targeted marketing problem akin to our setting. However, it is worth mentioning that in some recent work, \citet{jain2023effective} apply Thompson sampling to a pricing problem, while also optimizing for additional variables such as promotions and advertising.
There are three clear differences between this work and ours.
First, while their focus is on personalized (non-anonymous) prices, our focus is on anonymous prices; in our setting, this couples the local markets making the problem challenging.
Second, their work is in the stochastic setting under a parametric model, whereas we are in an adversarial setting under a non-parametric model.
Third, unlike them, our focus is on developing efficient algorithms that exploit the decomposable structure of the problem.
\section{Problem setting}\label{sec:setting-results}
In this section, we formalize the problem setting: Profit maximization over multiple markets with adversarial bandit feedback.

\paragraph{Notation.}
For any $n\in\nat$, let $[n]:=\{1,\dots,n\}$. Let $\ind[\cdot]$ be the indicator function. For any set $S$ measurable in a probability space, let $\Delta(S)$ denote the set of probability distributions over $S$. For a compact set $S\subseteq \rr^m$ or a finite set $S$, let $\uniform(S)$ denote the uniform distribution over $S$.
Unless the scope is explicitly specified, any expectation $\expec$ is taken over all randomness.

\paragraph{Problem setting.}
Our online learning problem is formulated as follows. Let there be $n$ local markets indexed by $i\in[n]$.
On each round, a firm (seller) chooses
$(\ct,\pt):=(c_{t,1},\dots,c_{t,n},\pt)\in[0,1]^{n+1}$, 
where $\pt$ denotes the \emph{common price} on round $t$ and $\cti$ denotes the marketing expenditure for market $i$.
At the end of the round, the firm observes the local demands (e.g., volume of sales)
$d_t:=(d_{t,1},\dots,d_{t,n})\in[0,1]^n$ in each market, which depend on the chosen $(\ct, \pt)$.
By normalizing prices, marketing expenditures, and demands, we assume that their ranges are in $[0,1]$. 

\emph{Environment. } An environment is a sequence of mappings $\{D_t\}_{t\in\nat}$, chosen possibly by an oblivious adversary. Here $D_t:[0,1]^{n+1}\rightarrow\Delta([0,1]^n)$
maps the chosen costs and price to a distribution over demands.
If the firm chooses costs and price $(\ct, \pt)$ on round $t$,
the realized demands $d_t:=\{d_{t,i}\}_{i=1}^n$
are simply random variables whose joint distribution is $D_t(\ct, \pt)$.

We assume that the expected demand in market $i$ on round $t$ depends \emph{only} on the price $\pt$ and the marketing expenditure $\cti$ for that market\footnote{%
This condition restricts only the expected demands; in particular, it does not require the demands $d_{t,i}$ and $d_{t,j}$ of two markets $i,j\in[n]$ to be conditionally independent given $\ct,\pt$.
It is also considerably weaker than assuming that the demand $\dti$ for market $i$ on round $t$ is a deterministic function of $\cti$ and $\pt$,
where $\{\dti\}_{t,i}$ are chosen adversarially.
}.
To state this explicitly,
let $\ct:=(c_{t,1},\dots,c_{t,i}, \dots, c_{t,n})$
and $\ct':=(c'_{t,1},\dots,c_{t,i}, \dots, c'_{t,n})$
be two sets of marketing costs that differ in all but the $i^{\rm{th}}$ market.
We then have $\expec_{d_t\sim D_t(\ct, \pt)}[\dti] = \expec_{d_t\sim D_t(\ct', \pt)}[\dti]$ for all such $\ct$ and $\ct'$ and all $\pt$.
We will denote this expected value by $\edti(\cti, \pt)$
where we view $\edti:[0,1]^2\rightarrow[0,1]$ as a function that maps the cost for a market and the price to the expected demand for that market.
Note that only the realized values $\{\dti\}_{i\in[n]}$ and not the expected values $\edti(\cti,\pt)$ are revealed at the end of the round.

\emph{Algorithm. } At the beginning of round $t$, the firm has the history of
previous prices, costs, and observed demands $\{(c_\tau, p_\tau, d_\tau)\}_{\tau=1}^{t-1}$.
An algorithm for the firm can be viewed as a map from this history to a  distribution $q_t\in\Delta([0,1]^{n+1})$ over the $n$ marketing expenditures and price; then the firm samples $(c_t,p_t)\sim q_t$ and executes $(c_t,p_t)$.

\paragraph{Regret.} The random variable representing the total profit made by the firm after $T$ rounds is
\begin{align*}
    \alg_T:=\sum_{t,i\in[T]\times[n]}\profit_{t,i},
     \quad\quad\text{where}\ \ \profit_{t,i}:=\pt\cdot\dti-\cti.
     \numberthis \label{eqen:algT}
\end{align*}
We compare the expectation of $\alg_T$ with the best profit $\opt_T$ achievable in expectation by the optimal fixed marketing expenditures and a common price in hindsight.
\begin{align*}
    \opt_T:=\sup_{(c,p)\in[0,1]^{n+1}}\sum_{t,i\in[T]\times[n]}\eprofit_{t,i}(c_i,p), \quad\text{where }\; \eprofit_{t,i}(c_i,p):=p\cdot\edti(c_i,p)-c_i.
\end{align*}
Hence, the algorithm's expected regret after round $T$ is defined as
\begin{align}
    R_T:=\opt_T-\mathbb{E}\big[\alg_T\big],
    \label{eq:regret}
\end{align}
where the expectation is taken over the randomness of the environment and the algorithm's choices.
We wish to bound $R_T$ over any sequence of mappings $\{D_t\}_{t}$ chosen by an oblivious adversary, 
which induces a sequence of expected demand functions $\{\edti\}_{t,i}$.

\paragraph{Assumptions.}
We present bandit profit-maximization algorithms for the targeted marketing problem under two assumptions on the expected demand functions $\{\edti\}_{t,i}$, which in turn imply conditions on $\{D_t\}_t$. The first assumes that $\edti(c_i,p)$ is monotonic with respect to $c_i$ and $p$. This captures the natural intuition that the demand increases with marketing and decreases with price.

\begin{assumption}[\textbf{Monotonic demands}]
    \emph{For each $t\in[T]$ and $i\in[n]$, the expected demand function $\edti(c_i,p)$ is non-decreasing in $c_i$ and non-increasing in $p$.}
    \label{assump:monotonic}
\end{assumption}

Another natural assumption is that $\edti$ is concave with respect to $c_i$. In other words, demand exhibits diminishing returns as we increase marketing costs. While it is natural to assume that the demand is also non-decreasing with these costs, it is not necessary for our analysis.

\begin{assumption}[\textbf{Cost-concave demands}]
    \emph{For each $t\in[T]$ and $i\in[n]$, the expected demand function $\edti(c_i,p)$ is concave in $c_i$ and non-increasing in $p$.}
    \label{assump:costconcave}
\end{assumption}

In Section \ref{sec:monotonic}, we present a no-regret algorithm, upper bound, and lower bound for Assumption~\ref{assump:monotonic}.
In Section \ref{sec:cost-concave}, we do the same for Assumption~\ref{assump:costconcave}.
\section{Targeted marketing with monotonic demands}\label{sec:monotonic}

In this section, we first present our algorithm
in Section~\ref{subsec:alg-monotnic}. In Section \ref{subsec:ub-monotnic}, 
we upper bound its regret, and in Section \ref{subsec:lb-monotinic}, we prove a nearly-matching lower bound on the regret.
%%%%%%%%%%%%%%%%%%%%%%%%%%%%%%%%%%%%%%%%%%%%%%%%%%%%%%%%%%%%%%%%%%%%%%%%%%%%%%%%%%%%%%%%%%
\subsection{Algorithm for monotonic demands}\label{subsec:alg-monotnic}

\paragraph{EXP3 review.}
We begin with a brief review of the EXP3 algorithm for adversarial multi-armed bandits~\citep{auer2002nonstochastic}.
A learner sequentially chooses one of $K$ arms over a series of rounds. Pulling arm $i$ on round $t$ incurs loss $\ell_{t}(i)$. The learner only observes the loss $\ell_{t}(a_t)$ for the arm $a_t$ pulled on round $t$.
The goal is to minimize regret with respect to the best arm in hindsight: $\sum_{t=1}^T \ell_{t}(a_t) - \min_{i\in[K]}\sum_{t=1}^T\ell_{t}(i)$.
EXP3 maintains a distribution $q_t=(q_{t}(1),\dots,q_{t}(K))\in\Delta([K])$ over the $K$ arms, and samples an arm $a_t$ from $q_t$. At the end of the round, it updates the distribution as
follows: $q_{t+1}(i) \propto q_{t}(i)\cdot\exp(-\eta\,\widehat{\ell_t}(i))$, where $\widehat{\ell_t}(i):=\ind[a_t=i]\ell_t(i)/q_t(a_t)$.
This update reduces the probability that arm $a_t$ is selected in a future round by an amount depending on the observed loss $\ell_t(a_t)$ and the probability $q_t(a_t)$ of selecting $a_t$.
If an arm consistently achieves large losses, it will be heavily discounted and be chosen rarely in the future.
EXP3 achieves $\widetilde{\O}\big(\sqrt{KT}\big)$ regret.

A straightforward, yet inefficient solution to the targeted marketing problem is to discretize the price and cost space and apply EXP3.
If the size of the discretization is $K$ along each dimension, then there are $K^{n+1}$ arms, so EXP3's regret when competing with the best price and costs in the discretization is $\widetilde{\O}\big(\sqrt{K^{n+1} T}\big)$.
If we additionally account for the regret due to discretization, which is $\O(T/K)$, and optimize for $K$, the regret is $\widetilde{\O}\big(T^\frac{n+2}{n+3}\big)$, which scales poorly with $n$.

\paragraph{Our method.}
To improve beyond the above approach, we exploit our problem's structure and decompose the problem into several simpler problems.
To illustrate, assume the firm knows the
expected demand curves $\{\edti\}_{i\in[n]}$. To optimize the expected profit 
$\sum^n_{i=1}p\, \edti(c_i,p)-c_i$ over the costs $\{c_i\}_{i\in[n]}$ and price $p$, we can fix a  candidate price $p$ and optimize $p\,d_i(c_i,p)-c_i$ for each $c_i$. We can repeat this for each $p$ and output the optimal price and the corresponding costs.
While the demand curves $\{\edti\}_{i\in[n]}$ are unknown in the bandit setting, our algorithm leverages this intuition.

\begin{algorithm}[t]
    \caption{Algorithm for monotonic demands}\label{alg:1}
    \textbf{Inputs:} learning rate $\eta > 0$,
    bias control parameter $\gamma>0$, discretization size $K\in\nat$.\\
    Let $\I_K:=\{0,K^{-1},2K^{-1},\dots,1\}\subset[0,1]$ be a discretized unit interval.\\
    $q_{1,0}(p)\leftarrow\uniform(\I_K)$, $q_{1,i}(c_i|p)\leftarrow\uniform(\I_K)$ for each $i\in[n],\,p\in\I_K$.\\
    \For{$t=1,\dots,T$}{
        Sample $\pt\sim \qtz$.
        Then sample $\cti\sim \qti(\,\cdot\,|\,\pt)$ for each $i\in[n]$.\label{line:alg1-sample}\\
        From observations $\{\dti\}_{i\in[n]}$ compute $\{\lti\}_{i\in[n]}$ according to \eqref{eq:lti}.\\
        For each $i\in[n], p\in \I_K$, update cost distributions $q_{t+1,i}(c_i|\,p)\propto \qti(c_i|\,p)\cdot \exp\rbra{-\eta\,\lossF_{t,i}(c_i,p)}$, where $\lossF_{t,i}(c_i,p)$ is defined in \eqref{eq:alg1-fti}.\label{line:alg1-qti}\\
        For each $p\in \I_K$, update price distribution $q_{t+1,0}(p)\propto \qtz(p)\cdot\exp\rbra{-\eta\,\lossH_t(p)}$, where $\lossH_t(p)$ is defined in \eqref{eq:alg1-ht}.\label{line:alg1-qt0}
    }
\end{algorithm}

To describe our algorithm (Algorithm~\ref{alg:1}), let $K\in \nat$ be the discretization size we will choose shortly.
Let $\I_K:=\{0,K^{-1},2K^{-1},\dots,1\}$ be a uniformly discretized unit interval of size $K$.
We maintain a price distribution $\qtz\in\Delta(\I_K)$, and
cost distributions $\qti(\,\cdot\,|\,p)\in\Delta(\I_K)$ for each  $p\in\I_K$ and $i\in[n]$.
In this approach, we need to track $K$ parameters for $\qtz$, and $K$ parameters for $\qti(\,\cdot\,|\,p)$ for each $i\in[n]$ and $p\in \I_K$, leading to a total of $K + nK^2$ parameters. This is significantly fewer than the naive application of EXP3 explained above, where we would need to track $K^{n+1}$ parameters.

On each round $t$, in line~\ref{line:alg1-sample}
of Algorithm~\ref{alg:1}, a price $p_t\in\I_K$ is first sampled from  $\qtz(\,\cdot\,)$.
We then sample the cost $\cti$ for each market $i\in[n]$ from $\qti(\,\cdot\,|\,\pt)$.
Finally, the algorithm updates $\qti$ to $q_{t+1,i}$ for $i\in\{0\}\cup[n]$ using the observed demand $\{d_{t,i}\}_{i\in[n]}$.

\paragraph{Loss functions.}
We first define the normalized (observed) losses on round $t$, which we will use in updating the distributions $\qtz, \qti$. We have:
\begin{align*}
    &\lti:=\frac{1}{2}\left(1-\profit_{t,i}\right),
    \hspace{0.5in}\lt:=\sum_{i\in[n]}\lti,\numberthis\label{eq:lti}
\end{align*}
Here, $\lti$ is the observed loss for market $i$ on round $t$,
while $\lt$ is the total loss in round $t$.
By normalizing, we ensure
$\lti,\,\elti(\cdot,\cdot)\in[0,1]$ and $\lt,\,\elt(\cdot,\cdot)\in[0,n]$, a technical condition required for the updates.
Both $\lti$ and $\lt$ can be computed as the realized demands $\dti$ are observed (see~\eqref{eqen:algT}).

For what follows, we will also find it useful to define the true losses $\elti(c_i,p)$ and $\elt(c, p)$ below.
Note that $\elti$ and $\elt$ are unknown, since our observations $\{\dti\}_i$ and hence $\{\lti\}$ are stochastic, but more importantly since we observe realized $\dti,\lti$ values at our chosen $(\cti,\pt)$. 
We have:
\begin{align*}
    \elti(c_i,p):=\frac{1}{2}\left(1-\eprofit_{t,i}(c_i,p)\right),\hspace{0.25in} \elt(c,p):=\elt((c_1,\dots,c_n),p):=\sum_{i\in[n]}\elti(c_i,p).\numberthis\label{eq:elti}
\end{align*}

We now present the loss functions $\lossF_{t,i}$ and $\lossH_{t}$ used in lines~\ref{line:alg1-qti} and~\ref{line:alg1-qt0} to update the distributions.
We first describe them, and then later motivate our design.
Let $\gamma>0$ be a bias control parameter whose value we will specify later.
First, $\lossF_{t,i}$ serves as an estimator for $\elti$, and is used to update $\qti$: 
\begin{align*}
    \lossF_{t,i}(c_i,p):=\frac{\lti\,\ind[\cti=c_i,\pt=p]}{\qti(\cti|\pt)(\qtz(\pt)+\gamma)}.
    \numberthis\label{eq:alg1-fti}
\end{align*}
Here, $\cti,\pt\in\I_K$ were chosen in Line \ref{line:alg1-sample} of Algorithm \ref{alg:1}.
A straightforward calculation reveals that,
as $\gamma\to 0$, we have $\expec_t\!\big[\lossF_{t,i}(c_i,p)\big]\to\elti(c_i,p)$ for any $c_i,p\in\I_K$, 
where the expectation $\expec_t$ is with respect to the environment's stochasticity and the algorithm's random choice in line~\ref{line:alg1-sample}, both on round $t$.
Next, consider $\lossH_{t}(\cdot)$ for updating the pricing distribution $\qtz$, which is given by,
\begin{align*}
    \lossH_t(p):=\frac{1}{n}\rbra{\frac{\lt\,\ind\!\sbra{\pt=p}}{\qtz(\pt)+\gamma}}+\eta|\I_K|\rbra{\frac{1}{\gamma}-\frac{1}{\qtz(p)+\gamma}}
    \numberthis\label{eq:alg1-ht}.
\end{align*}
Here, $\lt$ is as defined in~\eqref{eq:lti}.
As before,
a straightforward calculation reveals that,
as $\gamma\to 0$, we have
\begin{align*}
    \mathbb{E}_t\sbra{\frac{\lt\,\ind\!\sbra{\pt=p}}{\qtz(\pt)+\gamma}}\to\sum_{c_i\in\I_K\,\forall\,i\in[n]}q_{t,1}(c_1|p)\cdots q_{t,n}(c_n|p)\,\elt((c_1,\dots,c_n),p),
\end{align*}
for any $p\in\I_K$ (see~\eqref{eq:elti}).
In other words, the first term of $\lossH_t(p)$ serves as an estimate for the expected loss $\sum_i\elti(c_i, p)$, when the $c_i$ values are also sampled from the $\qti$ distributions. Meanwhile, the second term is an exploration bonus term \footnote{Or equivalently, an exploitation penalty.} as it is small for less explored prices: $p\text{ s.t. }\qtz(p)\ll 1$.

\emph{Design choices. } We will now motivate the design of~\eqref{eq:alg1-fti} and~\eqref{eq:alg1-ht}.
On the one hand, choosing a small $\gamma$ reduces the
bias in both $\lossF_{t,i}$ and $\lossH_t$. However, this also results in a potentially large variance, as $(\inf_{p\in\I_K}\qtz(p))^{-1}$ may be large.
In particular, since we have decomposed the problem into smaller optimization problems, there is no ``variance cancellation'' effect that usually arises in the standard EXP3 analysis~\citep{auer2002nonstochastic}.
The role of the second term in~\eqref{eq:alg1-ht}, and $\gamma$ in the denominator of~\eqref{eq:alg1-fti} and~\eqref{eq:alg1-ht} is to induce a more favorable bias-variance trade-off in the decomposed problem.

An alternative interpretation of this design choice is in terms of the exploration-exploitation trade-off.
For this, recall---from the overview of EXP3---that the purpose behind the update in line~\ref{line:alg1-qt0} is to reduce the probability that the same price $p_t$ would be chosen in a future round, by an amount depending on the observed loss $\elt$. However, this loss also depends on the costs $\cti$ chosen. Hence, even if $\elti$ is large, we should not be quick to dismiss $\pt$ since there may be other costs that yield a high profit at the same price.
The loss $\lossH_{t}$ ensures that we are more liberal with exploring a chosen price $\pt$ even if it incurred large losses; in particular, as the second term in the RHS of~\eqref{eq:alg1-ht} is applied to all prices, and not just $\pt$, its effect is that $\pt$ is not discounted as heavily as it otherwise would have.
The role of $\gamma$, from a mathematical perspective, is to ensure that $\lossH_{t}$ is bounded.

%%%%%%%%%%%%%%%%%%%%%%%%%%%%%%%%%%%%%%%%%%%%%%%%%%%%%%%%%%%%%%%%%%%%%%%%%%%%%%%%%%%%%%%%%%
\subsection{Regret upper bound for monotonic demands}\label{subsec:ub-monotnic}

We will now state and prove the following upper bound on the regret for Algorithm \ref{alg:1}.
Due to space constraints, we will defer the proofs of some intermediate technical results to the Appendix.

\begin{restatable}{theorem}{AlgOneRegret}
    Assume that the expected demands $\{\edti\}_{i,t}$ satisfy Assumption~\ref{assump:monotonic}. For any $\eta>0$ and $K\in\nat$, when $\gamma:=\eta$, the regret~\eqref{eq:regret} of Algorithm \ref{alg:1} satisfies,
    \begin{align*}
        R_T\leq \O\rbra{n\eta K^2T+\frac{n}{\eta}\log K+\frac{nT}{K}}.
    \end{align*}
    By choosing $K\in\Theta(T^{\nicefrac{1}{4}})$ and $\eta\in\Theta(T^{-\nicefrac{3}{4}})$, Algorithm \ref{alg:1} guarantees $R_T\in\O(nT^{\nicefrac{3}{4}}\log T)$.
    \label{thm:alg1-regret}
\end{restatable}

\noindent\textbf{Proof of Theorem \ref{thm:alg1-regret}.}
First, we introduce some notation. For a conditional distribution $q(c|p)$ over the costs in $\I_K$, given a price $p\in\I_K$ and a function $g(c,p)$ over $\I_K^2$, let $\innp{q,g}_p:=\sum_{c\in\I_K}q(c|p)g(c,p)$. Similarly, for a distribution $q(p)$ over prices in $\I_K$, and a function $g(p)$ over $\I_K$, let $\innp{q,g}:=\sum_{p\in\I_K}q(p)g(p)$. Throughout this proof, we set $\gamma:=\eta$ as stated in the theorem, but occasionally do not cancel out factors like $\gamma/\eta$  for convenience of exposition of Section \ref{sec:otherprobs}.

\paragraph{Step 1. [Preparation]}
We begin with a standard regret analysis of the exponential weights (a.k.a Hedge) algorithm~\citep{arora2012multiplicative}. Algorithm 1 runs exponential weights updates over effective losses $\{\lossF_{t,i}\}_{t\in[T]}$ for each $i\in[n]$ and $\{\lossH_t\}_{t\in[T]}$.
The following lemma states regret bounds for $\{\lossF_{t,i}\}_{t\in[T]}$ and $\{\lossH_t\}_{t\in[T]}$ when competing with any $(c^\star_i,p^\star)\in\I_K^2$.
For completeness, we give its proof, which is
similar to the standard exponential weights analysis, in Appendix \ref{proof:lem:alg1-exp-bounds}.

\begin{restatable}{lemma}{LemAlgOneEXPBounds}(Based on~\citet{arora2012multiplicative})
    For each $i\in[n]$, any $(c^\star_i,p^\star)\in\I_K^2$ and $\eta>0$,
    \begin{align*}
        &\sum_{t\in[T]}\innp{\qti\,,\lossF_{t,i}}_{p^\star}-\sum_{t\in[T]}\lossF_{t,i}(c^\star_i,p^\star)\leq \eta\sum_{t\in[T]}\innp{\qti\,,\lossF^2_{t,i}}_{p^\star}+\frac{1}{\eta}\log|\I_K|,\numberthis\label{eq:alg1-exp-bounds-1}\\
        &\sum_{t\in[T]}\innp{\qtz\,,\lossH_t}-\sum_{t\in[T]}\lossH_t(p^\star)\leq \eta\sum_{t\in[T]}\innp{\qtz\,,\lossH^2_t}+\frac{1}{\eta}\log|\I_K|,
        \numberthis\label{eq:alg1-exp-bounds-2}
    \end{align*}
    where $\qti$, $\qtz$ are distributions given in Lines \ref{line:alg1-qti} and \ref{line:alg1-qt0} in Algorithm \ref{alg:1}, respectively.
    \label{lem:alg1-exp-bounds}
\end{restatable}
The bounds in Lemma \ref{lem:alg1-exp-bounds} hold uniformly over all choices of $(c^\star,p^\star)\in\I^{n+1}_K$.
Moreover, $\lossF_{t,i}$ and $\lossH_t$ are random quantities as they depend on random variables realized up to round $t-1$.
In subsequent steps, we will take expectations of these bounds with respect to all randomness. 

\paragraph{Step 2. [Regret decomposition]}
First, using the definition of losses in~\eqref{eq:lti} and~\eqref{eq:elti}, we will decompose our regret $R_T$~\eqref{eq:regret} as follows,
\begin{align*}
    R_T&:=\sup_{(c^\star,p^\star)\in[0,1]^{n+1}}\sum_{t,i\in[T]\times[n]}\eprofit_{t,i}(c^\star_i,p^\star)-\expec\bigg[\sum_{t,i\in[T]\times[n]}\profit_{t,i} \bigg]\\
    &=2\Big(\sum_{t\in[T]}\expec[\lt]-\inf_{(c^\star,p^\star)\in[0,1]^{n+1}}\sum_{t\in[T]}\elt(c^\star,p^\star)\Big)\\
    &= 2\Big(\sum_{t\in[T]}\expec[\lt]-\min_{(c^\star,p^\star)\in\I_K^{n+1}}\sum_{t\in[T]}\elt(c^\star,p^\star)\Big)\\&\hspace{0.75in}
    +2\Big(\min_{(c^\star,p^\star)\in\I_K^{n+1}}\sum_{t\in[T]}\elt(c^\star,p^\star)-\inf_{(c^\star,p^\star)\in[0,1]^{n+1}}\sum_{t\in[T]}\elt(c^\star,p^\star) \Big)\\
    &:=\max_{(c^\star,p^\star)\in\I_K^{n+1}}R_T(c^\star,p^\star)+\widetilde{R}_T(n,K).
    \numberthis \label{eq:alg1-regretdecomp}
\end{align*}
In the third step, we have added and subtracted $\min_{(c^\star,p^\star)\in\I_K^{n+1}}\sum_{t\in[T]}\elt(c^\star,p^\star)$.
In the fourth step,  we have defined $R_T(c^\star,p^\star):=2\rbra{\sum_{t\in[T]}\expec\sbra{\lt}-\sum_{t\in[T]}\elt(c^\star,p^\star)}$
to be the regret of our algorithm relative to a given set of costs and prices $(c^\star,p^\star)$ in the discretization.
Moreover, 
\[
\widetilde{R}_T(n,K) := 2\Big(\min_{(c^\star,p^\star)\in\I_K^{n+1}}\sum_{t\in[T]}\elt(c^\star,p^\star)-\inf_{(c^\star,p^\star)\in[0,1]^{n+1}}\sum_{t\in[T]}\elt(c^\star,p^\star) \Big)
\]
denotes the residual regret due to discretization.
In the following lemma,
we bound $\widetilde{R}_T(n,K)$ via a simple argument.
Its proof is given in Appendix \ref{proof:lem:alg1-discretization-error}.
\begin{restatable}{lemma}{LemAlgOneDiscretizationError}
    Under Assumption \ref{assump:monotonic}, $\widetilde{R}_T(n,K)\leq \frac{2nT}{K}$.
    \label{lem:alg1-discretization-error}
\end{restatable}

Next, we will focus on deriving a uniform upper bound for $R_T(c^\star,p^\star)$ over all $(c^\star,p^\star)\in\I_K^{n+1}$.
In step 3, we will lower bound the loss $\sum_{t\in[T]}\elt(c^\star,p^\star)$ of any fixed $(c^\star, p^\star)$, and in step 4, we will upper bound the algorithm's loss $\sum_{t\in[T]}\expec[\lt]$.

\paragraph{Step 3. [Lower-bounding the comparator loss]} 
In this step, we will prove the following lower bound on $\sum_{t\in[T]}\elt(c^\star,p^\star)$, which we refer to as the comparator loss.

\begin{lemma}[Comparator loss bound]
    For any $(c^\star,p^\star)\in\I^{n+1}_K$, 
    \begin{align*}
        \frac{1}{n}\sum_{t\in[T]}\elt(c^\star,p^\star)\geq \sum_{t\in[T]}\expec\Big[\innp{\qtz\,,\lossH_t}\Big]-\frac{\eta|\I_K|T}{\gamma}-\frac{2}{\eta}\log|\I_K|-4\eta|\I_K|^2T.
    \end{align*}
    \label{lem:alg1-comparator-loss-bound}
    \vspace{-1em}
\end{lemma}
\begin{proof}[Proof of Lemma \ref{lem:alg1-comparator-loss-bound}]
    We will use the results in Lemma~\ref{lem:alg1-exp-bounds} to prove Lemma~\ref{lem:alg1-comparator-loss-bound}.
    The following lemma bounds some quantities in~\eqref{eq:alg1-exp-bounds-1} in expectation.

    \begin{restatable}{lemma}{LemAlgOneStepThreeOne}
        For any $(c^\star,p^\star)\in\I^{n+1}_K$ and $(t,i)\in[T]\times[n]$,
        \begin{align*}
            \expec\!\sbra{\lossF_{t,i}(c^\star_i,p^\star)}\leq\elti(c^\star_i,p^\star),\hspace{0.5in}
            \expec\!\sbra{\innp{\qti\,,\lossF^2_{t,i}}_{p^\star}}\leq \expec\!\sbra{\frac{|\I_{K}|}{\qtz(p^\star)+\gamma}}.
        \end{align*}
        \label{lem:alg1-step3-1}
        \vspace{-1em}
    \end{restatable}
    \noindent As discussed in~\eqref{eq:alg1-fti}, we see that $\lossF_{t,i}$ serves as a proxy for the unobserved $\elti$. The second bound still has the expectation in the RHS as $\qtz(p^\star)$ depends on (random) past history $\{(c_\tau,p_\tau,\ell_\tau)\}^{t-1}_{\tau=1}$.
    
    The second bound in Lemma~\ref{lem:alg1-step3-1}, may be very large due to the $\qtz$ term in the denominator. Fortunately, our design of
    $\lossH_t(p)$, specifically the second term in the RHS of~\eqref{eq:alg1-ht} ensures that this term will be canceled out.
    To see this more explicitly, let us take the expectations of the bound in \eqref{eq:alg1-exp-bounds-1}
    for each $i$, use Lemma \ref{lem:alg1-step3-1}, and sum them over $i\in[n]$. This simple algebraic manipulation leads to the following lemma. A detailed calculation is given in Appendix \ref{proof:lem:alg1-step3-2}.
    \begin{restatable}{lemma}{LemAlgOneStepThreeTwo}
        For any $(c^\star,p^\star)\in\I^{n+1}_K$,
        \begin{align*}
            \expec\Bigg[\sum_{t\in[T]}\rbra{\frac{1}{n}\frac{\lt\,\ind[\pt=p^\star]}{\qtz(\pt)+\gamma}-\frac{\eta|\I_K|}{\qtz(p^\star)+\gamma}}\Bigg]\leq\sum_{t\in[T]}\frac{1}{n}\elt(c^\star,p^\star)+\frac{1}{\eta}\log|\I_K|.
        \end{align*}
        \label{lem:alg1-step3-2}
        \vspace{-1em}
    \end{restatable}
    We see that the LHS of the above bound contains the same quantity as the RHS of the second bound of Lemma~\ref{lem:alg1-step3-1}.
    At the same time, the LHS of the above is \textit{identical} to $\expec\big[\sum_{t\in[T]}\lossH_t(p^\star)\big]-\eta|\I_K|T/\gamma$, (see \eqref{eq:alg1-ht} for the definition of $\lossH_t$). Hence, we can rewrite the above bound as follows:
   \begin{align*}
        \expec\Bigg[\sum_{t\in[T]}\lossH_t(p^\star)\Bigg]-\frac{\eta|\I_K|T}{\gamma}\leq\frac{1}{n}\sum_{t\in[T]}\elt(c^\star,p^\star)+\frac{1}{\eta}\log|\I_K|.
    \end{align*}
    Next, we eliminate $\expec\big[\sum_t\lossH_t(p^\star)\big]$ in the above by adding the above and the expectation of bound \eqref{eq:alg1-exp-bounds-2}, that is, $\expec\!\big[\sum_t\langle \qtz,\lossH_t \rangle\big]-\expec\!\big[\sum_t\lossH_t(p^\star)\big]\leq \eta \expec\!\big[\sum_t\langle \qtz,\lossH^2_t\rangle\big]+\frac{1}{\eta}\log|\I_K|$. Hence, we obtain
    \begin{align*}
        \expec\Bigg[\sum_{t\in[T]}\innp{\qtz\,,\lossH_t}\Bigg]-\frac{\eta|\I_K|T}{\gamma}\leq\frac{1}{n}\sum_{t\in[T]}\elt(c^\star,p^\star)+\frac{2}{\eta}\log|\I_K|+\eta\expec\Bigg[\sum_{t\in[T]}\innp{\qtz\,,\lossH^2_t}\Bigg].
        \numberthis\label{eq:alg1-panpanult-comparator-loss-bound}
    \end{align*}
    The above inequality is almost our target bound except for $\expec\!\big[\sum_t\langle\qtz\,,\lossH^2_t\rangle\big]$ term. We can bound this quantity as follows:
    \begin{align*}
        \innp{\qtz\,,\lossH^2_t}&\leq \sum_{p\in\I_K}\qtz(p)\cdot\rbra{\frac{\ind\!\sbra{\pt=p}}{\qtz(\pt)}+\frac{\eta|\I_K|}{\gamma}}^2 \tag{as $\lt/n\leq 1$}\\
        &=\sum_{p\in\I_K}\qtz(p)\cdot\rbra{\frac{\ind\!\sbra{\pt=p}}{q^2_{t,0}(\pt)}+\frac{2\eta|\I_K|\ind\!\sbra{\pt=p}}{\gamma\,\qtz(\pt)}+\frac{\eta^2|\I_K|^2}{\gamma^2}}\\
        &=\frac{1}{\qtz(\pt)}+\frac{2\eta|\I_K|}{\gamma}+\frac{\eta^2 |\I_K|^2}{\gamma^2}
        \leq\frac{1}{\qtz(\pt)}+3|\I_K|^2 \tag{as $\gamma:=\eta$ and $|\I_K|^2\geq|\I_K|$}.
    \end{align*}
    Hence, we have that, for each $t\in[T]$
    \begin{align*}
        \expec\!\sbra{\innp{\qtz\,,\lossH^2_t}}&\leq\expec\!\sbra{\frac{1}{\qtz(\pt)}}+3|\I_K|^2=|\I_K|+3|\I_K|^2\leq 4|\I_K|^2.
    \end{align*}
    Combining the above with \eqref{eq:alg1-panpanult-comparator-loss-bound}, we obtain the stated bound.
\end{proof}

\paragraph{Step 4. [Bounding the algorithm's loss]}
Next, in the following bound, we upper bound the algorithm's loss $\sum_{t\in[T]}\expec[\lt]$.
Its proof, which is purely algebraic, is given in Appendix \ref{proof:lem:alg1-algorithm-loss-bound}.

\begin{restatable}{lemma}{LemAlgOneAlgorithmLossBound}
    For all $t\in[T]$,
    \begin{align*}
        \frac{1}{n}\sum_{t\in[T]}\expec\!\sbra{\lt}\leq \sum_{t\in[T]}\expec\!\sbra{\innp{\qtz\,,\lossH_t}}-\frac{\eta |\I_K| T}{\gamma}+2\eta|\I_K|^2 T.
    \end{align*}
    \label{lem:alg1-algorithm-loss-bound}
\end{restatable}

\paragraph{Step 5. [Wrap up]}
Combining the lower bound of $\sum_{t\in[T]}\elt(c^\star,p^\star)$ (Lemma \ref{lem:alg1-comparator-loss-bound}) and the upper bound of $\sum_{t\in[T]}\expec[\lt]$ (Lemma \ref{lem:alg1-algorithm-loss-bound}),  gives the following bound. For all $(c^\star,p^\star)\in\I^{n+1}_K$, we have
\begin{align*}
    \frac{1}{2}R_T(c^\star,p^\star) = \expec\bigg[\sum_{t\in[T]}\lt\bigg]-\sum_{t\in[T]}\elt(c^\star,p^\star)\leq \frac{2n}{\eta}\log|\I_K|+6n\eta|\I_K|^2T.
\end{align*}
Since $|\I_K|=K+1$ and the residual regret $\widetilde{R}_T(n,K)$ is at most $2nT/K$ (Lemma \ref{lem:alg1-discretization-error}), we have that $R_T\in \O\rbra{n\eta K^2T+\frac{n}{\eta}\log K+\frac{nT}{K}}$, as claimed. 
\qedsymbol

We mention that our proof does not assume continuity of the demands $\{\edti\}_{t,i}$. However, it is straightforward to extend our analysis to non-monotonic demands that are Lipschitz continuous. 
%%%%%%%%%%%%%%%%%%%%%%%%%%%%%%%%%%%%%%%%%%%%%%%%%%%%%%%%%%%%%%%%%%%%%%%%%%%%%%%%%%%%%%%%%%
\subsection{Lower bound for monotonic demands}\label{subsec:lb-monotinic}

We will now state and prove the following lower bound for the targeted marketing problem with monotonic demands. When compared to the upper bound in Theorem~\ref{thm:alg1-regret}, we see that while we differ by an $n^{1/4}$ term, we are tight in $T$ up to $\log$ factors.

\begin{restatable}{theorem}{ThmLowerBoundMonotonic}
    Let $\A$ be an algorithm for the targeted marketing problem. Let $\{D_t\}_t$ be a sequence of mappings from chosen costs and price to a distribution of demands, that satisfy the monotonicity condition in Assumption~\ref{assump:monotonic}. Let $R_T(\A,\{D_t\}_t)$ be the expected regret achieved by $\A$ under $\{D_t\}_t$ up to round $T$. Then, we have the following:
    $\textup{inf}_\A\textup{sup}_{\{D_t\}_t} R_T\big(\A, \{D_t\}_t\big)\in \Omega((nT)^{\nicefrac{3}{4}}).$
    \label{thm:lb-monotonic}
\end{restatable}

\paragraph{Proof strategy.}
We will follow a standard recipe for proving lower bounds for adversarial bandit settings (e.g.,~\citet{bubeck2012regret}).
First, note that it is
sufficient to prove a regret lower bound of $\Omega((nT)^{\nicefrac{3}{4}})$ for a stationary stochastic environment where $D_t$ does not depend on $t$. This is because any algorithm that guarantees a regret over all non-stationary demands also guarantees the same regret over stationary demands.
Second, to prove a lower bound for stochastic environments, we use a standard hypothesis testing argument, where we construct a set of statistically indistinguishable environments, but with large differences in the optimal profit.

The main challenge in following this recipe is in the construction of the alternatives.
To do so, we construct a baseline environment where the optimal profit is $0$,
and a set of alternative environments, where the optimal profit is sufficiently larger than $0$.
The demand curves in 
each alternative is identical to the baseline except near a single cost-price tuple $(c^\star,p^\star)$. By repeatedly playing $(c^\star,p^\star)$, which is optimal for this alternative environment, one can generate a large profit.
By varying $(c^\star,p^\star)$ we generate several alternative environments.

\paragraph{Setting and notations.}
We first construct these environments for the $n=1$ case, where there is only a single demand curve whose mean function does not depend on the round, and then generalize these environments for $n\geq 1$. We use $\E$ annotated by a subscript to denote a stationary environment. When $n=1$, a stationary environment $\E$ is a function from $(c,p)\in[0,1]^2$ to a probability distribution over demands $d\in[0,1]$. Let $\ed(c,p)$ be the expectation of $d$ given $(c,p)$. For any randomized algorithm $\A$, let $R_T(\A,\E)$ be the expected regret of $\A$ under $\E$ up to round $T$.

To define our environments, we introduce discretized intervals $\C$ and $\P$ for cost and price, respectively. Let $K>0$ be a positive integer and $\eps:=K^{-1}$. Define $\C, \P\subset[0,1]$ as follows:
\begin{align*}
    \C&=\cbra{c_0,\dots,c_K}:=\cbra{0,\frac{1}{2K},\frac{1}{K},\dots,\frac{K-1}{2K},\frac{1}{2}}=\cbra{0,\frac{\eps}{2},\dots,\frac{1}{2}-\frac{\eps}{2},\frac{1}{2}}, \\
    \P&=\cbra{p_0,\dots,p_K}:=\cbra{\frac{1}{2},\frac{1}{2}+\frac{1}{2K},\dots,\frac{1}{2}+\frac{K-1}{2K},1}=\cbra{\frac{1}{2},\frac{1}{2}+\frac{\eps}{2},\dots,1}.
\end{align*}
We construct environments such that any $(c,p)\in\C\times\P$ achieves non-negative profit, while any $(c,p)\notin\C\times\P$ has the same demand distribution as a point of $\C\times\P$ that achieves at least the same profit.
Hence, it is sufficient to consider algorithms that only choose from these discrete options.
For what follows, we define $g(c)$ for $c\in[0,1]$ as follows:
\begin{align*}
    g(c):=\min\!\big(1,\lfloor 2cK \rfloor /K\big),
\end{align*}
Note that $g(c)=2c$ for $\forall c\in\C$ and $g(c)< 2c$ otherwise.

\paragraph{Baseline environment.}
The baseline environment $\envbase$ is defined as follows.
Given that the firm chooses $(c,p)$, the random variable $d$ denoting the realized demand is drawn via 
\begin{align*}
    d:=b\,\ind\!\sbra{v\geq p}\in\{0,1\}.
    \numberthis\label{eq:base-environment}
\end{align*}
Here $b$ is a Bernoulli random variable with mean $g(c)$ and $v\in\P$ is a discrete random variable, drawn independent from $b$, whose distribution is $\mathbb{P}(v\geq p)=(2p)^{-1}$ where $p\in\P$ (see Figure \ref{fig:2}). By construction, the optimal expected profit under the baseline is zero.
This can be seen via:
$\max_{c,p} \eprofit(c, p) =\max_{c,p} p\,g(c) \frac{1}{2p} - c =  0$, with the maximum attained by any $(c,p)\in\C\times\P$.

%%%%%%%%%%%%%%%%%%%%%%%%%%%%%%%%%%%%%%%%%%%%%%%%%%%%%%%%%%%%%%%%%%%%%%%%%%%%%%%%%%%%%%%%%%
\begin{figure}[t]
    \centering
    \begin{subfigure}[t]{0.44\textwidth}
        \centering
        \includegraphics[width=\textwidth]{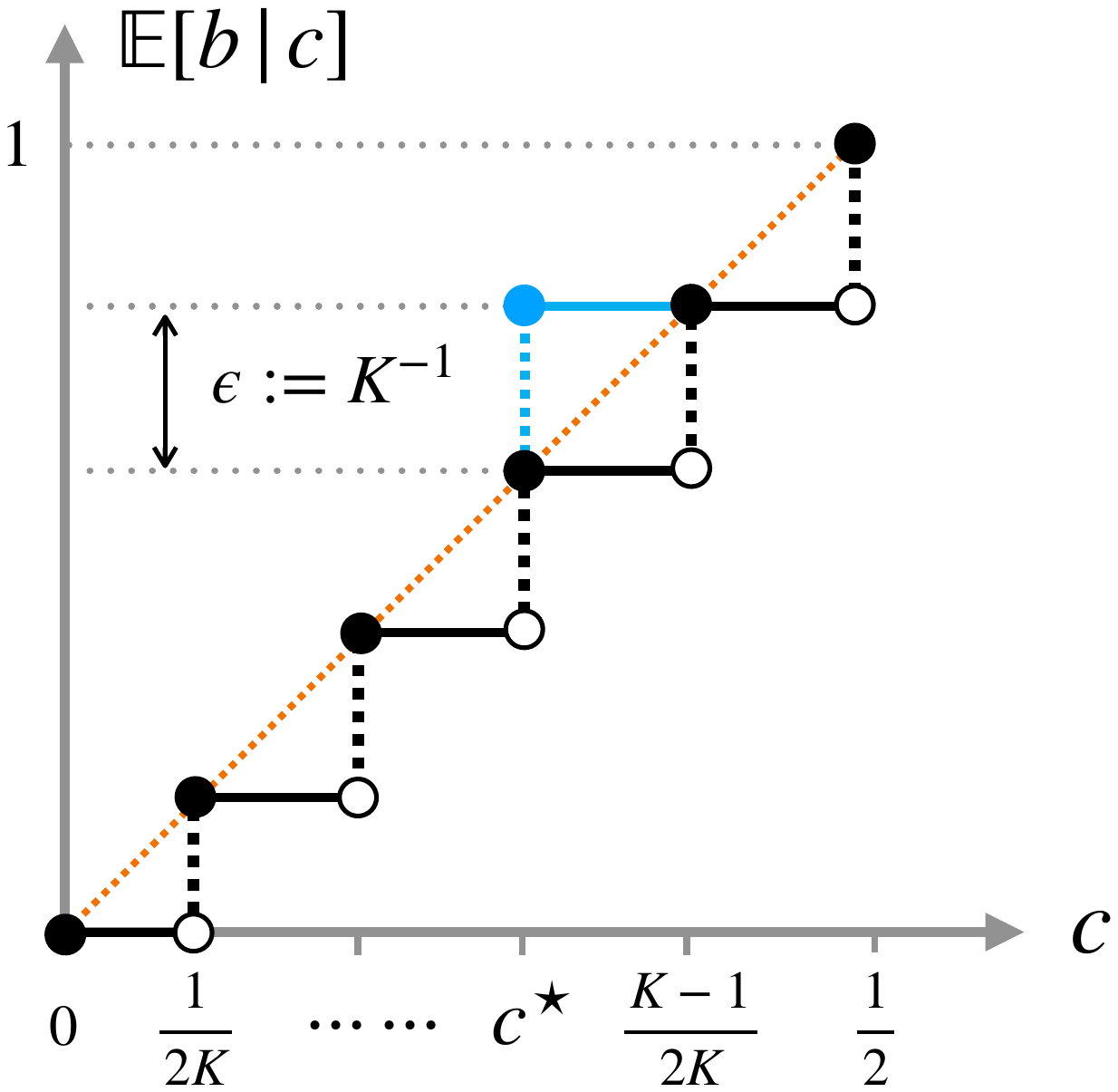}
        \caption{The blue dot indicates the bump in $\expec[b|c]$ for an alternative environment with optimal cost $c^\star$.}
        \label{fig:2a}
    \end{subfigure}
    \hspace{0.35in}
    \begin{subfigure}[t]{0.44\textwidth}
        \centering
        \includegraphics[width=\textwidth]{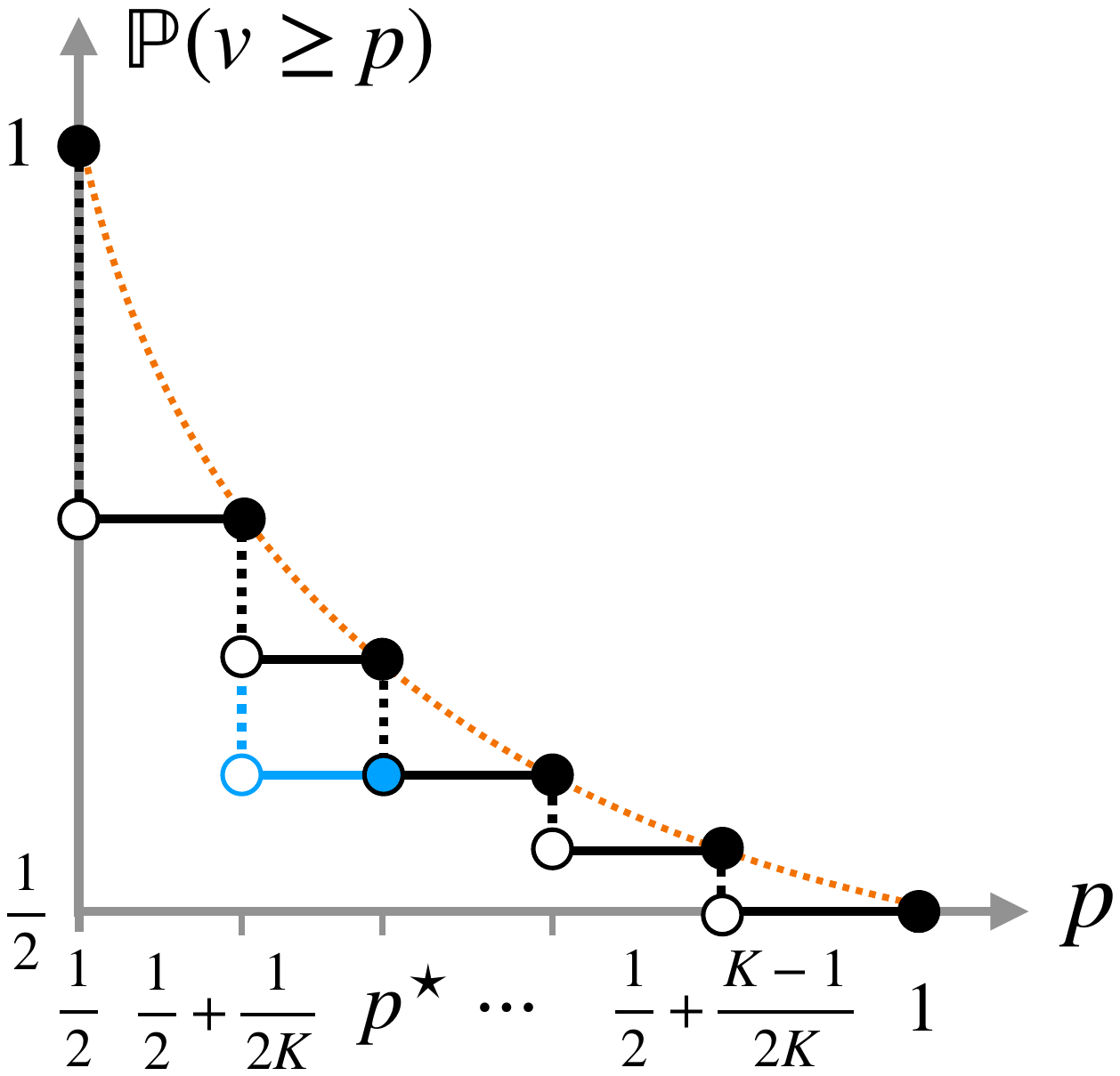}
        \caption{The blue dot indicates the dip in $\mathbb{P}(v\geq p)$ for an alternative environment with optimal price $p^\star$.}
        \label{fig:2b}
    \end{subfigure}
    \vspace{-0.1in}
    \caption{Illustrations of our baseline and alternative environments. (a) The black line with dots depicts $\expec[b]$ for the baseline defined in \eqref{eq:base-environment}, and the red line represents the function $2c$. (b) The black line with dots depicts $\mathbb{P}(v\geq p)=\expec\!\sbra{\ind[v\geq p]}$ for the baseline given in \eqref{eq:base-environment}, and the red line follows $(2p)^{-1}$.}
    \label{fig:2}
    \vspace{-1em}
\end{figure}
%%%%%%%%%%%%%%%%%%%%%%%%%%%%%%%%%%%%%%%%%%%%%%%%%%%%%%%%%%%%%%%%%%%%%%%%%%%%%%%%%%%%%%%%%%
\paragraph{Alternative environments.}
Now, we define a class of alternative environments that are statistically close to the baseline environment but have a sufficiently large positive optimal profit. 
Our design is such that the demand variables have the same distribution as the baseline, except near a single cost-price tuple $(c^\star,p^\star)\in\C\times\P$.
As it will turn out, $(c^\star,p^\star)$ will be the optimal arm for the alternative environment. Let $\S\subseteq \C\times\P$ be the set of feasible optimal arms defined as follows:
\begin{align*}
    \S:=\cbra{(c,p)\in\C\times\P \ \middle|\ \frac{2}{5}\leq c\leq \frac{9}{20},\,\frac{3}{5}\leq p\leq \frac{4}{5}}.
    \numberthis\label{eq:lb-hard-set}
\end{align*}
Note that $|\S|\in\Omega(K^2)$ as the area of square defined by $\frac{2}{5}\leq c\leq \frac{9}{20},\,\frac{3}{5}\leq p\leq \frac{4}{5}$ is constant. 

For each $(c^\star,p^\star)\in\S$, let $\envcstarpstar$ be the problem environment with $d:=b\,\ind\!\sbra{v\geq p}$ where $b$ and $v$ are independent random variables sampled as follows given $(c,p)\in[0,1]^2$.
\begin{align*}
    b\sim\begin{cases}
        \text{Ber}(g(c)+\eps),\quad\text{if }(c,p)\in [c^\star,c^\star+\eps/2)\times (p^\star-\eps/2,p^\star].\\
        \text{Ber}(g(c)),\quad\text{otherwise},
    \end{cases}
    \numberthis\label{eq:hard-environment-1d-1}
\end{align*}
and
\begin{align*}
    v\in\P,\ v\sim\begin{cases}
        \mathbb{P}(v\geq p_i)=(2p_i+\eps\,\ind[p_i=p^\star])^{-1},\\
        \qquad\text{if }(c,p)\in [c^\star,c^\star+\eps/2)\times (p^\star-\eps/2,p^\star].\\
        \mathbb{P}(v\geq p_i)=(2p_i)^{-1},\quad\text{otherwise}.
    \end{cases}
    \numberthis\label{eq:hard-environment-1d-2}
\end{align*}
Figure \ref{fig:2a} and \ref{fig:2b} illustrate $\expec[b|c]$ and $\mathbb{P}(v\geq p)$ respectively for the baseline environment. We have also shown how an alternative environment deviates around $(c^\star, p^\star)$ in blue.

The lemma below states that environment $\envcstarpstar$ satisfies Assumption \ref{assump:monotonic} and $(c^\star,p^\star)$ is the optimal arm of $\envcstarpstar$ which gives a positive profit, and that every $(c,p)\in[0,1]^2\setminus\C\times\P$ can be replaced by a point of $\C\times\P$ with the same demand distribution and at least the same profit. We give a proof of Lemma \ref{lem:lb-hard-environment-properties} in Appendix \ref{proof:lem:lb-hard-environment-properties}.

\begin{restatable}{lemma}{LemLowerBoundHardenvironmentProperties}
    For any $(c^\star,p^\star)\in\S$, let $\ed(c,p)$ be the conditional expectation of $d$. Then, $\ed(c,p)$ satisfies the following in environment $\envcstarpstar$.
    \begin{enumerate}
        \item For any $p\in[0,1]$, $\ed(c,p)$ is non-decreasing with respect to $c\in[0,1]$.
        \item For any $c\in[0,1]$, $\ed(c,p)$ is non-increasing with respect to $p\in[0,1]$.
        \item $\eprofit(c^\star,p^\star)\geq \frac{\eps}{20}>0$.
        \item $\eprofit(c,p)= 0$ for any $(c,p)\in\C\!\times\!\P\setminus\cbra{(c^\star,p^\star)}$. Moreover, for any $(c,p)\in[0,1]^2$, let $c'$ be the largest element of $\C$ that is at most $\min\{c,\frac{1}{2}\}$ and $p'$ the smallest element of $\P$ that is at least $\max\{p,\frac{1}{2}\}$. Then $d$ has the same distribution under $(c,p)$ and $(c',p')$, and $\eprofit(c',p')\geq\eprofit(c,p)$.
    \end{enumerate}
    \label{lem:lb-hard-environment-properties}
\end{restatable}

\paragraph{Environments with multiple ($n\geq 1$) markets.}
From the $n=1$ environments defined above, we define an $n$-dimensional baseline environment by repeatedly sampling $n$ independent demand variables.
Precisely, for each $t\in[T]$, let $\{\dti\}_{i\in[n]}$ be independent realizations of the demands $d$ in the baseline environment $\envbase$ \eqref{eq:base-environment}. With a slight abuse of notation, we call this $n$-dimensional baseline $\envbase$. The $n$-dimensional alternatives are similarly constructed by repeatedly sampling $n$ independent demand variables.
Precisely, for any $(c^\star,p^\star)\in\S$, $\{\dti\}_{i\in[n]}$ are independent realizations of $d$ in $\envcstarpstar$ \eqref{eq:hard-environment-1d-1}, \eqref{eq:hard-environment-1d-2}. Hence, the optimal arm is $\cti=c^\star$ for all $i\in[n]$ and $\pt=p^\star$. As before, we will call this $n$-dimensional alternative environment $\envcstarpstar$. 

For what follows, we recall the Bretagnolle-Huber inequality~\citep{bretagnolle1979estimation} which states that for any two distributions $\mathbb{P}_1$, $\mathbb{P}_2$ and event $A$, we have
\begin{align*}
    \mathbb{P}_1(A)+\mathbb{P}_2(A^c)\geq\frac{1}{2}\exp\!\rbra{-\text{KL}(\mathbb{P}_1||\,\mathbb{P}_2)}.
    \numberthis\label{eq:bregtagnolle}
\end{align*}
We are now ready to present the proof of Theorem \ref{thm:lb-monotonic}.

\begin{proof}[Proof of Theorem \ref{thm:lb-monotonic}]
For any bandit algorithm $\A$, let $\ppAbase$ and $\ppAcstarpstar$ be the distributions of $\{\dti,\cti,\pt\}_{t\in[T],\,i\in[n]}$ when $\A$ is instantiated with $\envbase$ and $\envcstarpstar$ respectively. Let $\eeAbase$, $\eeAcstarpstar$ be the expectations with respect to $\ppAbase$, $\ppAcstarpstar$ respectively. 

A key ingredient in proving lower bounds is a uniform bound on statistical indistinguishability between $\ppAbase$ and $\ppAcstarpstar$ for any algorithm $\A$.
The following lemma proves that the KL divergence between $\ppAbase$ and $\ppAcstarpstar$ is bounded above by the expected number of times that $\A$ chooses a (near) optimal arm. Its proof uses properties of our construction and is available in Appendix \ref{proof:lem:lb-kl-decomposition}.
\begin{restatable}{lemma}{LemLowerBoundKLDecomposition}
    For $t\in[T]$, $i\in[n]$ and $(c,p)\in\C\times\P$, define the following indicator random variable:
    \begin{align*}
        \chi_{t,i}(c,p):=\ind\!\sbra{(\cti,\pt)=(c,p)}.
    \end{align*}
    Then, for any bandit algorithm $\A$ that chooses from $\C^n\times\P$,
    \begin{align*}
        \emph{KL}\left(\ppAbase\,\Big|\Big|\,\ppAcstarpstar\right)\leq \sum_{t,i\in[T]\times[n]}54\eps^2\cdot\eeAbase\left[\chi_{t,i}(c^\star,p^\star)\right].
    \end{align*}
    \label{lem:lb-kl-decomposition}
    \vspace{-1.25em}
\end{restatable}

For any algorithm $\A$ and stationary environment $\E$, let $R_T(\A,\E)$ be the expected regret of $\A$ up to round $T$ under $\E$. Then, for any algorithm $\A$, we have that $R_T(\A,\E)\leq \sup_{\{D_t\}_t}R_T(\A,\{D_t\}_t)$ by H{\"o}lder's inequality. Hence, the theorem follows if $\inf_{\A}\sup_{\E}R_T(\A,\E)\in\Omega((nT)^{\nicefrac{3}{4}})$. 

To this end, let $\S\subseteq\C\times\P$ be \eqref{eq:lb-hard-set}.
As an algorithm can choose only one set of costs and a price on each round, we have
\begin{align*}
    \sum_{(c,p)\in\S}\chi_{t,i}(c,p)\leq 1, 
\hspace{0.2in} \text{for all }(t,i)\in [T]\times[n].
\end{align*}
Therefore,
$\sum_{(c,p)\in\S}\sum_{t,i\in[T]\times[n]}\eeAbase\,\chi_{t,i}(c,p)\leq nT$. This implies for some $(c^\star,p^\star)\in\S$, we have $\sum_{t,i\in[T]\times[n]}\eeAbase \,\chi_{t,i}(c^\star,p^\star)\leq nT/|\S|$. 
We will now consider the environment $\envcstarpstar$ for this $(c^\star,p^\star)$,
and apply the Bretagnolle-Huber inequality~\eqref{eq:bregtagnolle} over $\ppAcstarpstar$ and $\ppAbase$. We have,
\begin{align}
\ppAcstarpstar\Bigg(\sum_{t,i\in[T]\times[n]}\chi_{t,i}(c^\star,p^\star)\leq\frac{nT}{2}\Bigg)+\ppAbase\Bigg(\sum_{t,i\in[T]\times[n]}\chi_{t,i}(c^\star,p^\star)>\frac{nT}{2}\Bigg)\label{eq:markov}\\
    \geq\frac{1}{2}\exp\!\rbra{-\text{KL}\!\rbra{\ppAbase  \,\middle|\middle|\,\ppAcstarpstar}}\nonumber.
\end{align}
By Markov's inequality, and the observation above about the environment $\envcstarpstar$, the second term in the LHS of~\eqref{eq:markov} can be upper bounded as follows:
\begin{align*}
    \ppAbase\Bigg(\sum_{t,i\in[T]\times[n]}\chi_{t,i}(c^\star,p^\star)>\frac{nT}{2}\Bigg)\leq \frac{2}{nT}\cdot\sum_{t,i\in[T]\times[n]}\eeAbase\,\chi_{t,i}(c^\star,p^\star)\leq\frac{2}{|\S|},
\end{align*}
Moreover, by Lemma \ref{lem:lb-kl-decomposition}, we have
\begin{align*}
    \exp\rbra{-\text{KL}\!\rbra{\ppAbase  \,\middle|\middle|\,\ppAcstarpstar}}\geq \exp\Bigg(-54\eps^2\cdot\sum_{t,i\in[T]\times[n]}\eeAbase\sbra{\chi_{t,i}(c^\star,p^\star)}\Bigg)\geq\exp\rbra{-\frac{54\eps^2 nT}{|\S|}}.
\end{align*}
Therefore, we have the following lower bound
\begin{align*}
    \ppAcstarpstar\Bigg(\sum_{t,i\in[T]\times[n]}\chi_{t,i}(c^\star,p^\star)\leq\frac{nT}{2}\Bigg)\geq\frac{1}{2}\exp\left(-\frac{54\eps^2nT}{|\S|}\right)-\frac{2}{|\S|}.
\end{align*}
Finally, by item 3 of Lemma \ref{lem:lb-hard-environment-properties}, we can lower bound the expected regret of $\A$ under $\envcstarpstar$ as:
\begin{align*}
    R_T\!\left(\A,\,\envcstarpstar\right)&\geq \frac{\eps nT}{40}\cdot\ppAcstarpstar\Bigg(\sum_{t,i\in[T]\times[n]}\chi_{t,i}(c^\star,p^\star)\leq\frac{nT}{2}\Bigg)\\
    &\geq \frac{\eps nT}{80}\left(\exp\left(-\frac{54\eps^2 nT}{|\S|}\right)-\frac{4}{|\S|} \right)\\
    &\geq a_1\cdot\frac{nT}{K}\left(\exp\left(-a_2\cdot\frac{nT}{K^4}\right)-\frac{a_3}{K^2} \right)\tag{as $\eps:=K^{-1}$ and $|\S|\in\Omega(K^2)$},
\end{align*}
for some $a_1, a_2, a_3>0$. Finally, we obtain the stated bound by choosing $K=\lceil (nT)^{\nicefrac{1}{4}} \rceil$.
\end{proof}
\section{Targeted marketing with cost-concave demands}\label{sec:cost-concave}

In this section, we study the targeted marketing problem under the assumption that the expected demands $\{\edti\}_{t,i}$ are concave in $\cti$ (Assumption~\ref{assump:costconcave}).
This captures settings where there are diminishing returns to increasing spending on marketing.
We first present our algorithm
in Section~\ref{subsec:alg-concave}. In Section \ref{subsec:ub-concave}, 
we upper bound its regret, and in Section \ref{subsec:lb-concave}, we provide a matching lower bound.

%In Section \ref{subsec:alg-concave}, we provide our second algorithm (Algorithm \ref{alg:2}) and state its regret bound in Theorem \ref{thm:alg2-regret}. In Section \ref{subsec:ub-concave}, we prove the regret guarantee of Algorithm \ref{alg:2}. In Section \ref{subsec:lb-concave}, we prove a regret lower bound of $\Omega\!\rbra{nT^{\nicefrac{2}{3}}}$ for the targeted marketing problem with cost-concave demands.
%%%%%%%%%%%%%%%%%%%%%%%%%%%%%%%%%%%%%%%%%%%%%%%%%%%%%%%%%%%%%%%%%%%%%%%%%%%%%%%%%%%%%%%%%%
\subsection{Algorithm for cost-concave demands}\label{subsec:alg-concave}

%We now present our algorithm for the targeted marketing problem with cost-concave demands (Assumption \ref{assump:costconcave}).

Under Assumption~\ref{assump:costconcave},  the normalized loss functions $\elti(c_i,p)$, defined in \eqref{eq:elti}, are convex with respect to $c_i\in[0,1]$.
Hence, while we will use the same decomposed structure as in Algorithm~\ref{alg:1}, we will 
borrow techniques from the literature on bandit convex optimization~\citep{flaxman2005online,bubeck2017kernel} to perform updates for the cost distributions.
In particular, our updates are based on a simplified version of the kernelized exponential weights algorithm~\citep{bubeck2017kernel}.
%The idea is to use a smoothing kernel to propagate the observed losses $\lti$ to nearby points in order to construct a low-variance estimator, using the convexity of $\elti(\cdot,p)$.
%Using these ideas, we are able to improve the dependence on $T$ in the regret bound from $T^{3/4}$ to $T^{2/3}$. 

In each round, our algorithm for this setting, described in Algorithm~\ref{alg:2}, samples a price $\pt$ from a (discrete) distribution $\qtz(\cdot)$; then, it samples $\cti$ from $\qti(\cdot|\pt)$ for each $i\in[n]$.
This is similar to Algorithm \ref{alg:1}. The key difference is that $\qti(\cdot|p)$, which is no longer a discrete distribution, is derived by convolving another distribution $u_{t,i}(\cdot|p)$ with a smoothing kernel. To elaborate further, we first begin by describing kernels.
%which will be introduced soon.

\paragraph{Kernel.}
For $\del>0$, let $\I_\del:=[\del,1-\del]$. 
%A kernel induces a linear map from the space of distributions over $\I_\del$ to itself.
A kernel is a bivariate function $K(\cdot,\cdot):\I_\del\times\I_\del\rightarrow\mathbb{R}_{\geq 0}$ such that $\int_{\I_\del}K(x,y)dx=1$ for all $y\in\I_\del$.
This kernel induces a linear operator
$K:\Delta(\I_\del)\rightarrow\Delta(\I_\del)$,
which is a map from the space of distributions over $\I_\del$ to itself.
%Given such a kernel $K(\cdot,\cdot)$, 
With a slight abuse of notation, we denote both by $K$.
%is denoted as $K:\Delta(\I_\del)\rightarrow\Delta(\I_\del)$ induced by $K(\cdot,\cdot)$ is defined as follows.
For any input distribution $q\in\Delta(\I_\del)$, this operator outputs,
\begin{align*}
    Kq(\,\cdot\,):=\int_{\I_\del} K(\,\cdot\,,y)q(y)dy.
    \numberthis \label{eq:kerneloperator}
\end{align*}
%We say that $K(\cdot,\cdot)$ is the kernel of the linear map $K$.
We can now define the linear operator $K_\eps[q]$ used in
lines \ref{line:kernel} and \ref{line:alg2-qti} of 
Algorithm \ref{alg:2}.
% This kernel $K_\eps[q]:\Delta(\I_\del)\rightarrow\Delta(\I_\del)$, which is parametrized by
Consider any
$\eps>0$ and distribution $q\in\Delta(\I_\del)$.
We first define the following kernel
%is the linear map induced by the following kernel 
 $K_\eps[q](\cdot,\cdot):\I_\del\times\I_\del\rightarrow\mathbb{R}_{\geq 0}$.
Denote $\mu:=\expec_{X\sim q}[X]$. We have,
\begin{align*}
    K_\eps[q](x,y):=\begin{cases}
    \frac{1}{|y-\mu|}\cdot\ind\left[x\in [\min(y,\mu),\max(y,\mu)] \right],\ \ \text{if }|y-\mu|\geq\eps.\\
    \\
    \frac{1}{\eps}\cdot\ind\left[x\in[\mu-\eps, \mu]\right],\ \ \text{if }|y-\mu|<\eps,\  \mu\geq\eps+\del.  
    \numberthis\\
    \\
    \frac{1}{\eps}\cdot\ind\left[x\in[\mu, \mu+\eps]\right],\ \ \text{if }|y-\mu|<\eps,\  \mu<\eps+\del.
    \end{cases}\label{eq:kernel}
\end{align*}
In words, if $|y-\mu|\geq\eps$, $K_\eps[q](\,\cdot\,,y)$ is equivalent to the uniform pdf between $y$ and $\mu$, %$[\min(y,\mu),\max(y,\mu)]\in\I_\del$,
while otherwise it is equivalent to the uniform pdf %uniform density
on either $[\mu-\eps,\mu]$ or $[\mu,\mu+\eps]$, whichever is contained in $\I_\del$. 
The linear operator $K_\eps[q]:\Delta(\I_\del)\rightarrow\Delta(\I_\del)$
used in
Algorithm \ref{alg:2}, is induced by the above kernel as shown in~\eqref{eq:kerneloperator}.
%applies the above kernel $K_\eps[q](\cdot,\cdot)$ % This kernel $K_\eps[q]:\Delta(\I_\del)\rightarrow\Delta(\I_\del)$, which is parametrized by
A similar kernel was used by~\citet{bubeck2017kernel}.
%However, since we are in one dimension, we have a considerably simpler version of their algorithm~ and analysis.
%\citep[][Section 3]{bubeck2017kernel} in one dimension.

\begin{algorithm}[t]
    \caption{Bandit targeted marketing algorithm for cost-concave demands}\label{alg:2}
    \textbf{Inputs:} learning rate $\eta>0$, bias  parameter $\gamma$, kernel parameters $\eps>0, \,\del\in(0,1)$, discretization $K\in\nat$.\\
    Let $\I_\del:=[\del,1-\del]$, all distributions over $c_i$ are supported on $\I_\del$.\\
    Let $\I_K:=\{0,K^{-1},2K^{-1},\dots,1\}$ be discrete price levels.\\
    %\bigskip
    $q_{1,0}(p)\leftarrow\uniform(\I_K)$, $u_{1,i}(c_i|\,p)\leftarrow\uniform(\I_\del)$ for each $i\in[n],\,p\in\I_K$.\\
    $q_{1,i}(c_i|\,p)=K_\eps[u_{1,i}(\cdot\,|\,p)]u_{1,i}(c_i|\,p)$ with $K_\eps[\,\cdot\,]$ defined in \eqref{eq:kernel}.\label{line:kernel}\\
    %\bigskip
    \For{$t=1,\dots,T$}{
        Sample $\pt\sim \qtz$, $\cti\sim \qti(\,\cdot\,|\,\pt)$ for each $i\in[n]$.\label{line:alg2-sample}\\
        From observations $\{\dti\}_{i\in[n]}$ compute $\{\lti\}_{i\in[n]}$ according to \eqref{eq:lti}\\
        For each $i\in[n]$, $u_{t+1,i}(c_i|\,p)\propto u_{t,i}(c_i|\,p)\cdot \exp\rbra{-\eta\,\lossF_{t,i}(c_i,p)}$, with $\lossF_{t,i}(c_i,p)$ defined  in \eqref{eq:alg2-fti}.\label{line:alg2-uti}\\
        For each $i\in[n]$, $q_{t+1,i}(c_i|\,p)=K_\eps[u_{t+1,i}(\cdot\,|\,p)]u_{t+1,i}(c_i|\,p)$, with $K_\eps[\,\cdot\,]$ defined in \eqref{eq:kernel}\label{line:alg2-qti}\\
        Update price distributions,  $q_{t+1,0}(p)\propto \qtz(p)\cdot\exp\rbra{-\eta\,\lossH_t(p)}$, with $\lossH_t(p)$ defined in \eqref{eq:alg2-ht}. \label{line:alg2-qt0}
    }
\end{algorithm}

\paragraph{Loss functions.}
Given some $\gamma>0$, we define the
following estimator $\lossF_{t,i}(\cdot,\cdot)$ for the cost-convex loss $\elti(\cdot,\cdot)$, defined in~\eqref{eq:elti}, as follows.
%For any $\gamma>0$,
\begin{align*}
    \lossF_{t,i}(c_i,p):=\frac{\lti\,\ind[\pt=p]}{\qti(\cti|\pt)\rbra{\qtz(\pt)+\gamma}}\,K_\eps\left[u_{t,i}(\,\cdot\,|\,p)\right](\cti,c_i),\numberthis\label{eq:alg2-fti}
\end{align*}
where $\cti\in\I_\del$, $\pt\in\I_K$ are the random cost and price sampled at round $t$ in Line \ref{line:alg2-sample}. Compared to $\lossF_{t,i}$ \eqref{eq:alg1-fti} defined for Algorithm \ref{alg:1}, the above uses a kernel function $K_\eps\left[u_{t,i}(\,\cdot\,|\,p)\right](\cti,c_i)$ instead of the indicator $\ind[\cti=c_i]$. Roughly speaking, as a function of $c_i$, $K_\eps\left[u_{t,i}(\,\cdot\,|\,p)\right](\cti,c_i)$ is zero in the interval bounded by $\cti$ and $\expec_{c\sim u_{t,i}(\cdot|p)}[c]$.
Hence, so is $\lossF_{t,i}(c_i,p)$.
Therefore, this choice of $\lossF_{t,i}$ encourages the algorithm to explore this interval 
%the interval bounded by $\cti$ and $\expec_{c\sim u_{t,i}(\cdot|p)}[c]$
in future rounds. Intuitively, as convexity is a global property, searching over a wider region is potentially rewarding.

As in Algorithm \ref{alg:1}, we define a function $\lossH_t(\cdot)$ used to update $\qtz$ as follows.
This design follows a similar intuition to the one in~\eqref{eq:alg1-ht}.
Letting $\lt:=\sum_{i\in[n]}\lti$, we have,
\begin{align*}
    \lossH_t(p):=\frac{1}{n}\left(\frac{\lt\ind\!\sbra{\pt=p}}{\qtz(\pt)+\gamma}\right)+\rbra{3\eta\log\frac{e}{\eps}}\rbra{\frac{1}{\gamma}-\frac{1}{\qtz(p)+\gamma}} \numberthis\label{eq:alg2-ht}.
\end{align*}
%where .

% %\paragraph{Regret bound.}\ \ Finally, we state the regret guarantee of Algorithm \ref{alg:2} in below.
% \begin{restatable}{theorem}{AlgTwoRegret}
%     For any $\eta>0$, $K\in\nat$, Algorithm \ref{alg:2} guarantees
%     \begin{align*}
%        R_T\leq \O\rbra{n\eta KT\log T + \frac{n}{\eta}\log KT+\frac{nT}{K}},
%     \end{align*}
%     when $\gamma:=\eta\log (e/\eps)$, $\eps:=T^{-2}$ and $\del:=T^{-1}$. Especially, by choosing $K\in\Theta(T^{\nicefrac{1}{3}})$, $\eta\in\Theta(T^{-\nicefrac{2}{3}})$, Algorithm \ref{alg:2} guarantees $R_T\in\O\rbra{nT^{\nicefrac{2}{3}}\log T}$.
%     \label{thm:alg2-regret}
% \end{restatable}
%%%%%%%%%%%%%%%%%%%%%%%%%%%%%%%%%%%%%%%%%%%%%%%%%%%%%%%%%%%%%%%%%%%%%%%%%%%%%%%%%%%%%%%%%%
\subsection{Regret upper bound for cost-concave demands}\label{subsec:ub-concave}
In this section, we prove the following upper bound on the regret.

\begin{restatable}{theorem}{AlgTwoRegret}
    Assume that the expected demands $\{\edti\}_{i,t}$ satisfy Assumption~\ref{assump:costconcave}.
    Let $\eta>0$, $K\in\nat$.
    When $\eps:=T^{-2}$, $\gamma:=\eta\log (e/\eps)$, and $\del:=T^{-1}$, 
    the regret~\eqref{eq:regret} of Algorithm \ref{alg:2} satisfies
    \begin{align*}
       R_T\in \O\rbra{n\eta KT\log T + \frac{n}{\eta}\log KT+\frac{nT}{K}},
    \end{align*}
    Especially, by choosing $K\in\Theta(T^{\nicefrac{1}{3}})$, $\eta\in\Theta(T^{-\nicefrac{2}{3}})$, Algorithm \ref{alg:2} guarantees $R_T\in\O\rbra{nT^{\nicefrac{2}{3}}\log T}$.
    \label{thm:alg2-regret}
\end{restatable}

When compared to the bound in Theorem~\ref{thm:alg1-regret},
the dependence on $T$ is improved from $T^{3/4}$ to $T^{2/3}$. 
Our proof uses similar intuitions to Section~\ref{subsec:ub-monotnic}, along with some techniques adapted from~\citet{bubeck2017kernel}.
Hence, we only state the main steps and defer most details to the appendix.

\paragraph{Proof of Theorem \ref{thm:alg2-regret}.}
We first introduce some notation.
For a conditional distribution $q(c|p)$ over $\I_\del$ and function $g(c,p)$ let $\innp{q,g}_p:=\int_{\I_\del} q(c|p)g(c,p)dc$. For a distribution $q(p)$ and function $g(p)$ over $\I_K$, let $\innp{q,g}:=\sum_{p\in\I_K}q(p)g(p)$ as before. Throughout the proof, we set parameters $\gamma:=\eta\log (e/\eps)$, $\eps:=T^{-2}$ and $\del:=T^{-1}$ as stated in the theorem.

\paragraph{Step 1. [Preparation]}
We start by presenting some basic properties of the expected loss $\elti(\cdot,\cdot)$ \eqref{eq:elti} under the cost-concave demand assumption (its proof is given in Appendix \ref{proof:lem:alg2-convex-loss-properties}).
\begin{restatable}{lemma}{LemAlgTwoConvexLossProperties}
    Under Assumption \ref{assump:costconcave}, the following holds for any $\del\leq 1/2$, $p^\star\in\I_K$, $t\in[T]$ and $i\in[n]$:
    \begin{enumerate}
        \item $\elti(\,\cdot\,, p^\star)$ is $\del^{-1}$-Lipschitz over $\I_\del:=[\del,1-\del]$.
        \item For any $c^\star_i<\del$, $\elti(\del,p^\star)\leq \elti(c^\star_i,p^\star)+2\del$.
        \item For any $c^\star_i> 1-\del$, $\elti(1-\del,p^\star)\leq \elti(c^\star_i,p^\star)+2\del$.
    \end{enumerate}
    \label{lem:alg2-convex-loss-properties}
\end{restatable}
\noindent
Result (1) of the above lemma states that the losses are Lipschitz continuous in the interval $\I_\del$, while  (2) and (3) state that we do not lose much by choosing costs only in $\I_\del$ instead of $[0,1]$. The subsequent steps are analogous to that of the proof of Theorem \ref{thm:alg1-regret}: First, we decompose the regret. Then, we lower and upper bound the comparator's loss and algorithm's loss, respectively.

\paragraph{Step 2. [Regret decomposition]}
Using essentially the same argument as~\eqref{eq:alg1-regretdecomp}, we can show that the regret can be decomposed as follows:
\begin{align*}  
    R_T=\sup_{(c^\star,p^\star)\in\I_\del^n\times\I_K}R_T(c^\star,p^\star)+\widetilde{R}_T(n,\del,K).
    \numberthis\label{eq:alg2-regretdecomp}
\end{align*}
Here $R_T(c^\star,p^\star)$,  is the regret of our algorithm relative to a given set of costs and prices $(c^\star, p^\star)\in\I_\del^n\times \I_K$.
Recall that Algorithm~\ref{alg:2}  chooses costs $\cti\in\I_\del$ and prices $\pt\in\I_K$.
We have,
\begin{align*}
    R_T(c^\star,p^\star):=2\rbra{\sum_{t\in[T]}\expec\sbra{\lt}-\sum_{t\in[T]}\elt(c^\star,p^\star)}.
\end{align*}
Moreover, $\widetilde{R}_T(n,\del,K)$ is the residual
regret due to only focusing on $\I_\del^n\times \I_K$. We have,
\begin{align*}
    \widetilde{R}_T(n,\del,K) = 2\Big(\inf_{(c^\star,p^\star)\in\I_\del^n\times\I_K}\sum_{t\in[T]}\elt(c^\star,p^\star)-\inf_{(c^\star,p^\star)\in[0,1]^{n+1}}\sum_{t\in[T]}\elt(c^\star,p^\star) \Big).
\end{align*}
Following the same argument as Lemma \ref{lem:alg1-discretization-error} and using Lemma \ref{lem:alg2-convex-loss-properties}, we immediately see that $\widetilde{R}_T(n,\del,K)\in\O(nT/K+n\del T)=\O(nT/K)$, as we set $\del:=T^{-1}$ throughout this proof.

\paragraph{Step 3. [Bounding the comparator loss]}
Next, we lower bound the comparator loss $\sum_{t}\elt(c^\star,p^\star)$ where $(c^\star,p^\star)\in\I_\del^n\times\I_K$. The main result of this step is the following lemma.
\begin{restatable}{lemma}{LemAlgTwoComparatorLossBound}
    For any $(c^\star,p^\star)\in \I^n_\del\times\I_K$, we have
    \begin{align*}
        \expec\Bigg[\sum_{t\in[T]}\innp{\qtz\,,\lossH_t}\Bigg]-\frac{3 \eta T}{\gamma}\log\frac{e}{\eps}\leq\frac{1}{n}\sum_{t\in[T]}\elt(c^\star,p^\star)+\O\rbra{\eps\,T^2 + \eta|\I_K|T + \frac{1}{\eta}\log(|\I_K|T)}.
    \end{align*}
    \label{lem:alg2-comparator-loss-bound}
    \vspace{-1em}
\end{restatable}
\noindent Our proof of the above lemma uses the same intuitions as Lemma~\ref{lem:alg1-comparator-loss-bound},
along with some techniques adapted from \citet{bubeck2017kernel}. We give its detailed proof in Appendix \ref{subsec:proof-of-comparator-loss-bound}.

\paragraph{Step 4. [Bounding the algorithm's loss]}
We now upper bound the algorithm's loss via the following lemma. Its proof, given in Appendix \ref{proof:lem:alg2-algorithm-loss-bound}, uses a straightforward calculation.
\begin{restatable}{lemma}{LemAlgTwoAlgorithmLossBound}
The following bound holds for the cumulative losses,
    \begin{align*}
        \frac{1}{n}\sum_{t\in[T]}\expec\!\sbra{\lt}\leq\sum_{t\in[T]}\expec\!\sbra{\innp{\qtz\,,\lossH_t}}-\frac{3\eta T}{\gamma}\log\frac{e}{\eps}+4\eta|\I_K|T\log(e/\eps).
    \end{align*}
    \label{lem:alg2-algorithm-loss-bound}
    \vspace{-1em}
\end{restatable}

\paragraph{Step 5. [Wrap up]}
Finally, by combining Lemma \ref{lem:alg2-comparator-loss-bound}, \ref{lem:alg2-algorithm-loss-bound} and $\eps:=T^{-2}$, we obtain
\begin{align*}
    R_T(c^\star,p^\star):=2\rbra{\expec\Bigg[\sum_{t\in[T]}\lt\Bigg]-\sum_{t\in[T]}\elt(c^\star,p^\star)}\leq\O\rbra{n\eta|\I_K|T\log T+\frac{n}{\eta}\log|\I_K|T},
\end{align*}
over all $(c^\star,p^\star)\in\I_\del^n\times\I_K$. Using the regret decomposition in~\eqref{eq:alg2-regretdecomp}, the fact that $\widetilde{R}_T(n,\del,K)\in\O(nT/K)$ (see step 2), and noting that $|\I_K|=K+1$, we obtain the claimed regret bound.
\qedsymbol
%%%%%%%%%%%%%%%%%%%%%%%%%%%%%%%%%%%%%%%%%%%%%%%%%%%%%%%%%%%%%%%%%%%%%%%%%%%%%%%%%%%%%%%%%%
\subsection{Lower bound for cost-concave demands}\label{subsec:lb-concave}
In this section, we present Theorem \ref{thm:lb-concave}, a regret lower bound of $\Omega(nT^{\nicefrac{2}{3}})$, for the targeted marketing problem with cost-concave demands.
This shows that Algorithm~\ref{alg:2} is minimax optimal up to $\log$ factors.
Its proof, given in Appendix \ref{app:proof:thm:lb-concave} for completeness, is a straightforward adaptation of the $\Omega(T^{\nicefrac{2}{3}})$ lower bound for the online pricing problem in~\citet{kleinberg2003value}.

\begin{restatable}{theorem}{ThmLowerBoundConcave}
    For any algorithm $\A$ for the targeted marketing problem and a demand sequence $\{D_t\}_t$, let $R_T(\A,\{D_t\}_t)$ be the expected regret achieved by $\A$ under $\{D_t\}_t$ up to round $T$. Then,
    $\textup{inf}_\A\textup{sup}_{\{D_t\}_t} R_T\big(\A, \{D_t\}_t\big)\in \Omega(nT^{\nicefrac{2}{3}}),$
    where the supremum is taken over all $n$ market demand sequences $\{D_t\}_t$ that satisfies the cost-concave demands assumption (Assumption \ref{assump:costconcave}).
    \label{thm:lb-concave}
\end{restatable}

\section{Variations of the targeted marketing problem}\label{sec:otherprobs}

We now present variations of the targeted marketing problem, where our framework is applicable.

\vspace{0.1in}
\noindent \textbf{1. Subscription service.}\ \
Consider a firm that runs a subscription service, where new customers join on each round, while some old customers leave. Let $\dti\in[0,1]$ be the number of new users of type $i\in[n]$ who joined the service during round $t$. Let $\beta_i\in[0,1)$ be the fraction of users who keep the service in each round in market $i\in[n]$. Hence, in each round $t$, there are $\sum^t_{s=1}\beta_i^{t-s}d_{s,i}\in\big[0,\frac{1}{1-\beta_i}\big]$ active users. For $\cti\in[0,1]$, let $\frac{\cti}{1-\beta_i}$ be the marketing expenditure\footnote{As the number of active users at any time is at most $\frac{1}{1-\beta_i}$ and the price (per round) is at most $1$, 
 there is no reason to spend more than $\frac{1}{1-\beta_i}$ on marketing. Hence we write the marketing expenditure as $\frac{\cti}{1-\beta_i}$ to ensure $\cti\in [0,1]$.} spent in round $t$ to attract new users of type $i$. Then, the total profit in round $t$ is
\begin{align*}
    \profit_t:=\sum_{i\in[n]}\profit_{t,i}:=\sum_{i\in[n]}\bigg(\Big(\sum^t_{s=1}\beta_i^{t-s}p_s d_{s,i}\Big)-\frac{\cti}{1-\beta_i}\bigg).
\end{align*}
This is similar to our problem, but there is a memory effect since past customers can contribute to future profits.
Despite this added complexity,
in Appendix \ref{app:subscription}, we show that our algorithms can be adapted to obtain the same regret bounds.

\vspace{0.1in}
\noindent \textbf{2. Promotional credit.}\ \
Next, we consider the promotional credit problem, where the firm segments the population into $n$ types indexed by $i\in[n]$. On each round $t$,  $\rti\in[0,1]$ people of each type try the firm's service (where $\rti$ is chosen exogenously). The firm offers promotional credits $\cti\in[0,1]$ to each type $i$ and price $\pt\in[0,1]$. After using promotional credits, a  $\dti\in[0,1]$ fraction of the $\rti$ customers decide to purchase (or continue) the service. The total profit at round $t$ is
\begin{align*}
    \profit_t:=\sum_{i\in[n]}\profit_{t,i}:=\sum_{i\in[n]}\rti\left(\pt\dti-\cti\right).
\end{align*}
In Appendix \ref{app:promotion}, we show that our algorithms can be applied to this problem without modification.

\vspace{0.1in}
\noindent \textbf{3. Profit-maximizing A/B tests.}\ \
Suppose a firm performs a sequence of experiments where it chooses $M$ marketing alternatives for $n$ population segments to increase demand for a product. The firm wishes to perform these experiments while maintaining a common price. Assume that each alternative $m\in[M]$ costs $c_t(m)\in[0,1]$ to implement in round $t$. For example, when $M=2$, the first option ($m=1$) could be to show the product on a non-interactive webpage, while the second option ($m=2$) could be to present the product on an AI-assisted interactive webpage. In this scenario, the second option costs more as it requires more computing resources. 

For each $i\in[n]$ and $t\in[T]$, let $\dti\in[0,1]$ be a random variable representing the normalized demands made by population segment $i$ during round $t$, given a choice of alternative $m_{t,i}\in[M]$ and price $\pt\in[0,1]$.  The firm's total profit at round $t$ is 
\begin{align*}
    \profit_t:=\sum_{i\in[n]}\profit_{t,i}:=\sum_{i\in[n]}\left(\pt\dti-c_t(m_{t,i})\right).
\end{align*}
In Appendix \ref{app:profitmaxAB}, we show a slight modification of our Algorithm \ref{alg:1} yields $\widetilde{\O}\big(n\sqrt{M}T^{\nicefrac{2}{3}}\big)$ regret.
\section{Conclusion}\label{sec:conclusion}

We studied profit maximization when there are multiple markets whose demands respond differently to the price and marketing expenditure. The demand characteristics are unknown, and the goal is to design algorithms that are able to learn the optimal price and marketing costs via repeated interactions.
Our algorithms, designed to exploit the decomposable structure of this optimization problem, use carefully designed loss functions to manage the exploration-exploitation trade-off.
When compared to a naive application of an adversarial bandit algorithm, which has regret with an exponential dependence on the number of markets, our approach only has linear dependence. We complement our upper bounds with nearly matching lower bounds.
Closing the $n^{1/4}$ gap between the upper and lower bounds for monotonic demands is an interesting avenue for future research.

% Bibliography
\bibliographystyle{ACM-Reference-Format}
\bibliography{ref}

@String{Computing = "Computing" }

@String{Computer = "{IEEE} Computer" }

@String{Springer = "Springer-Verlag" }

@article{auer1995gambling,
    title={Gambling in a rigged casino: The adversarial multi-armed bandit problem},
    author={Auer, Peter and Cesa-Bianchi, Nicolo and Freund, Yoav and Schapire, Robert E},
    journal={Proceedings of IEEE 36th Annual Symposium on Foundations of Computer Science},
    pages={322--331},
    year={1995},
    organization={IEEE}
}

@article{auer2002nonstochastic,
    title={The nonstochastic multiarmed bandit problem},
    author={Auer, Peter and Cesa-Bianchi, Nicolo and Freund, Yoav and Schapire, Robert E},
    journal={SIAM Journal on Computing},
    volume={32},
    number={1},
    pages={48--77},
    year={2002},
    publisher={SIAM}
}

@inproceedings{kleinberg2003value,
    title={The value of knowing a demand curve: Bounds on regret for online posted-price auctions},
    author={Kleinberg, Robert and Leighton, Tom},
    booktitle={Proceedings of IEEE 44th Annual Symposium on Foundations of Computer Science},
    pages={594--605},
    year={2003},
    organization={IEEE}
}

@article{besbes2009dynamic,
    title={Dynamic pricing without knowing the demand function: Risk bounds and near-optimal algorithms},
    author={Besbes, Omar and Zeevi, Assaf},
    journal={Operations Research},
    volume={57},
    number={6},
    pages={1407--1420},
    year={2009},
    publisher={INFORMS}
}

@article{wang2021multimodal,
  title={Multimodal dynamic pricing},
  author={Wang, Yining and Chen, Boxiao and Simchi-Levi, David},
  journal={Management Science},
  volume={67},
  number={10},
  pages={6136--6152},
  year={2021},
  publisher={INFORMS}
}

@article{cheung2017dynamic,
    title={Dynamic pricing and demand learning with limited price experimentation},
    author={Cheung, Wang Chi and Simchi-Levi, David and Wang, He},
    journal={Operations Research},
    volume={65},
    number={6},
    pages={1722--1731},
    year={2017},
    publisher={INFORMS}
}

@article{misra2019dynamic,
    title={Dynamic online pricing with incomplete information using multiarmed bandit experiments},
    author={Misra, Kanishka and Schwartz, Eric M and Abernethy, Jacob},
    journal={Marketing Science},
    volume={38},
    number={2},
    pages={226--252},
    year={2019},
    publisher={INFORMS}
}

@article{perakis2023dynamic,
  title={Dynamic pricing with unknown nonparametric demand and limited price changes},
  author={Perakis, Georgia and Singhvi, Divya},
  journal={Operations Research},
  year={2023},
  publisher={INFORMS}
}

@article{keskin2014dynamic,
  title={Dynamic pricing with an unknown demand model: Asymptotically optimal semi-myopic policies},
  author={Keskin, N Bora and Zeevi, Assaf},
  journal={Operations Research},
  volume={62},
  number={5},
  pages={1142--1167},
  year={2014},
  publisher={INFORMS}
}

@article{den2014simultaneously,
    title={Simultaneously learning and optimizing using controlled variance pricing},
    author={den Boer, Arnoud V and Zwart, Bert},
    journal={Management Science},
    volume={60},
    number={3},
    pages={770--783},
    year={2014},
    publisher={Informs}
}

@article{besbes2015surprising,
    title={On the (surprising) sufficiency of linear models for dynamic pricing with demand learning},
    author={Besbes, Omar and Zeevi, Assaf},
    journal={Management Science},
    volume={61},
    number={4},
    pages={723--739},
    year={2015},
    publisher={INFORMS}
}

@article{javanmard2017perishability,
  title={Perishability of data: dynamic pricing under varying-coefficient models},
  author={Javanmard, Adel},
  journal={The Journal of Machine Learning Research},
  volume={18},
  number={1},
  pages={1714--1744},
  year={2017},
  publisher={JMLR. org}
}

@article{javanmard2019dynamic,
  title={Dynamic pricing in high-dimensions},
  author={Javanmard, Adel and Nazerzadeh, Hamid},
  journal={Journal of Machine Learning Research},
  volume={20},
  number={9},
  pages={1--49},
  year={2019},
  publisher={JMLR.org}
}

@article{xu2021logarithmic,
  title={Logarithmic regret in feature-based dynamic pricing},
  author={Xu, Jianyu and Wang, Yu-Xiang},
  journal={Advances in Neural Information Processing Systems},
  volume={34},
  pages={13898--13910},
  year={2021}
}

@article{den2015dynamic,
    title={Dynamic pricing and learning: {H}istorical origins, current research, and new directions},
    author={Den Boer, Arnoud V},
    journal={Surveys in Operations Research and Management Science},
    volume={20},
    number={1},
    pages={1--18},
    year={2015},
    publisher={Elsevier}
}

@article{jain2023effective,
    title={Effective Adaptive Exploration of Prices and Promotions in Choice-Based Demand Models},
    author={Jain, Lalit and Li, Zhaoqi and Loghmani, Erfan and Mason, Blake and Yoganarasimhan, Hema},
    journal={Available at SSRN 4438537},
    year={2023}
}

@article{hauser2009website,
    title={Website morphing},
    author={Hauser, John R and Urban, Glen L and Liberali, Guilherme and Braun, Michael},
    journal={Marketing Science},
    volume={28},
    number={2},
    pages={202--223},
    year={2009},
    publisher={INFORMS}
}

@article{hauser2014website,
    title={Website morphing 2.0: Switching costs, partial exposure, random exit, and when to morph},
    author={Hauser, John R and Liberali, Guilherme and Urban, Glen L},
    journal={Management Science},
    volume={60},
    number={6},
    pages={1594--1616},
    year={2014},
    publisher={INFORMS}
}

@article{sawant2018contextual,
    title={Contextual multi-armed bandits for causal marketing},
    author={Sawant, Neela and Namballa, Chitti Babu and Sadagopan, Narayanan and Nassif, Houssam},
    journal={arXiv preprint arXiv:1810.01859},
    year={2018}
}

@article{choi2020online,
    title={Online display advertising markets: A literature review and future directions},
    author={Choi, Hana and Mela, Carl F and Balseiro, Santiago R and Leary, Adam},
    journal={Information Systems Research},
    volume={31},
    number={2},
    pages={556--575},
    year={2020},
    publisher={INFORMS}
}

@article{schwartz2017customer,
    title={Customer acquisition via display advertising using multi-armed bandit experiments},
    author={Schwartz, Eric M and Bradlow, Eric T and Fader, Peter S},
    journal={Marketing Science},
    volume={36},
    number={4},
    pages={500--522},
    year={2017},
    publisher={INFORMS}
}

@article{liberali2022morphing,
    title={Morphing for consumer dynamics: Bandits meet hidden Markov models},
    author={Liberali, Gui and Ferecatu, Alina},
    journal={Marketing Science},
    volume={41},
    number={4},
    pages={769--794},
    year={2022},
    publisher={INFORMS}
}

@article{urban2014morphing,
    title={Morphing banner advertising},
    author={Urban, Glen L and Liberali, Guilherme and MacDonald, Erin and Bordley, Robert and Hauser, John R},
    journal={Marketing Science},
    volume={33},
    number={1},
    pages={27--46},
    year={2014},
    publisher={INFORMS}
}

@article{bretagnolle1979estimation,
    title={Estimation des densit{\'e}s: risque minimax},
    author={Bretagnolle, Jean and Huber, Catherine},
    journal={Zeitschrift f{\"u}r Wahrscheinlichkeitstheorie und verwandte Gebiete},
    volume={47},
    pages={119--137},
    year={1979},
    publisher={Springer}
}

@book{cover1999elements,
    title={Elements of information theory},
    author={Cover, Thomas M},
    year={1999},
    publisher={John Wiley \& Sons}
}

@book{lattimore2020bandit,
    title={Bandit algorithms},
    author={Lattimore, Tor and Szepesv{\'a}ri, Csaba},
    year={2020},
    publisher={Cambridge University Press}
}

@book{png2022managerial,
    title={Managerial economics},
    author={Png, Ivan},
    year={2022},
    publisher={Routledge}
}

@article{hazan2016introduction,
    title={Introduction to online convex optimization},
    author={Hazan, Elad and others},
    journal={Foundations and Trends{\textregistered} in Optimization},
    volume={2},
    number={3-4},
    pages={157--325},
    year={2016}
}

@article{arora2012multiplicative,
    title={The multiplicative weights update method: {A} meta-algorithm and applications},
    author={Arora, Sanjeev and Hazan, Elad and Kale, Satyen},
    journal={Theory of Computing},
    volume={8},
    number={1},
    pages={121--164},
    year={2012}
}

@article{bubeck2012regret,
    title={Regret analysis of stochastic and nonstochastic multi-armed bandit problems},
    author={Bubeck, S{\'e}bastien and Cesa-Bianchi, Nicolo and others},
    journal={Foundations and Trends{\textregistered} in Machine Learning},
    volume={5},
    number={1},
    pages={1--122},
    year={2012},
    publisher={Now Publishers, Inc.}
}

@inproceedings{flaxman2005online,
    title={Online convex optimization in the bandit setting: gradient descent without a gradient},
    author={Flaxman, Abraham D and Kalai, Adam Tauman and McMahan, H Brendan},
    booktitle={Proceedings of the 16th Annual ACM-SIAM Symposium on Discrete Algorithms},
    pages={385--394},
    year={2005}
}

@article{hazan2014bandit,
    title={Bandit convex optimization: Towards tight bounds},
    author={Hazan, Elad and Levy, Kfir},
    journal={Advances in Neural Information Processing Systems},
    volume={27},
    pages={784-792},
    year={2014}
}

@inproceedings{bubeck2017kernel,
    title={Kernel-based methods for bandit convex optimization},
    author={Bubeck, S{\'e}bastien and Lee, Yin Tat and Eldan, Ronen},
    booktitle={Proceedings of the 49th Annual ACM SIGACT Symposium on Theory of Computing},
    pages={72--85},
    year={2017}
}

% Appendix
\appendix
\section{Proofs omitted from Section \ref{sec:monotonic}}\label{app:deferred-proofs-sec3}
\subsection{Lemmas for regret upper bound proof}\label{app:deferred-proofs-sec3-ub}
\LemAlgOneEXPBounds*
\begin{proof}\label{proof:lem:alg1-exp-bounds}
    The second bound follows from essentially the same procedure that derives the first bound. Hence, we prove the first bound only. Recall that for conditional distribution $q(c\,|\,p)$ over $\I_K$ and function $g(c,p)$, $\innp{q,g}_p:=\sum_{c\in\I_K}q(c|p)g(c,p)$. Fix any $p\in\I_K$. First, rewrite $\innp{\qti,\lossF_{t,i}}_p$ as follows:
    \begin{align*}
        \innp{\qti,\lossF_{t,i}}_p=\bunderbrace{\innp{\qti,\lossF_{t,i}}_p+\frac{1}{\eta}\log\innp{\qti,\exp\rbra{-\eta\lossF_{t,i}}}_p}{A_{t,i}(p)}\ \ \bunderbrace{-\frac{1}{\eta}\log\innp{\qti,\exp\rbra{-\eta\lossF_{t,i}}}_p}{B_{t,i}(p)}.
    \end{align*}
    Then, we upper bound $A_{t,i}(p)$ as follows:
    \begin{align*}
        A_{t,i}(p)&:=\innp{\qti,\lossF_{t,i}}_p+\frac{1}{\eta}\log\innp{\qti,\exp\rbra{-\eta\lossF_{t,i}}}_p\\
        &\leq\innp{\qti,\lossF_{t,i}}_p+\frac{1}{\eta}\rbra{\innp{\qti,\exp\rbra{-\eta\lossF_{t,i}}}_p-1} \tag{as $\log x\leq x-1$ for $x\geq 0$}\\
        &=\frac{1}{\eta}\innp{\qti,\,\exp\rbra{-\eta\lossF_{t,i}}-1+\eta \lossF_{t,i}}_p \tag{as $\innp{\qti,1}_p=1$}\\
        &\leq \frac{\eta}{2}\innp{\qti,\lossF^2_{t,i}}_p \tag{as $e^{-x}-1+x\leq\frac{x^2}{2}$ for $x\geq 0$}
    \end{align*}
    Next, we rewrite $B_{t,i}(p)$ as follows:
    \begin{align*}
        B_{t,i}(p)&:=-\frac{1}{\eta}\log\innp{\qti,\exp\rbra{-\eta\lossF_{t,i}}}_p\\
        &=-\frac{1}{\eta}\log\frac{\sum_{c_i\in\I_K}\exp\rbra{-\eta\sum^t_{\tau=1}\lossF_{\tau,i}(c_i,p)}}{\sum_{c_i\in\I_K}\exp\rbra{-\eta\sum^{t-1}_{\tau=1}\lossF_{\tau,i}(c_i,p)}}\\
        &=\frac{1}{\eta}\rbra{\Phi_{t-1,i}(p)-\Phi_{t,i}(p)},
    \end{align*}
    where $\Phi_{t,i}(p):=\log\sum_{c_i\in\I_K}\exp\rbra{-\eta\sum^t_{\tau=1}\lossF_{\tau,i}(c_i,p)}$ and $\Phi_{0,i}(p):=\log \sum_{c_i\in\I_K }1=\log |\I_K|$. Using the above bounds for $A_{t,i}(p)$ and $B_{t,i}(p)$, we obtain
    \begin{align*}
        \innp{\qti,\lossF_{t,i}}_p\leq\frac{\eta}{2}\innp{\qti,\lossF^2_{t,i}}_p+\frac{1}{\eta}\rbra{\Phi_{t-1,i}(p)-\Phi_{t,i}(p)}.
    \end{align*}
    Summing the above inequality over $t\in[T]$ yields
    \begin{align*}
        \sum_{t\in[T]}\innp{\qti,\lossF_{t,i}}_p&\leq\frac{\eta}{2}\sum_{t\in[T]}\innp{\qti,\lossF^2_{t,i}}_p+\frac{1}{\eta}\rbra{\Phi_{0,i}(p)-\Phi_{T,i}(p)}\\
        &=\frac{\eta}{2}\sum_{t\in[T]}\innp{\qti,\lossF^2_{t,i}}_p+\frac{1}{\eta}\log|\I_K|-\frac{1}{\eta}\Phi_{T,i}(p).
    \end{align*}
    Finally, we upper bound $-\frac{1}{\eta}\Phi_{T,i}(p)$ as follows: For any $c^\star_i\in\I_K$,
    \begin{align*}
        \frac{1}{\eta}\Phi_{T,i}(p)=\frac{1}{\eta}\log\sum_{c_i\in\I_K}\exp\rbra{-\eta\sum_{t\in[T]}\lossF_{t,i}(c_i,p)}\geq \frac{1}{\eta}\log\exp\rbra{-\eta\sum_{t\in[T]}\lossF_{t,i}(c^\star_i,p)}=-\sum_{t\in[T]}\lossF_{t,i}(c^\star_i,p).
    \end{align*}
    By combining the above two inequalities and setting $p=p^\star$, we obtain the first bound \eqref{eq:alg1-exp-bounds-1}.
\end{proof}
%%%%%%%%%%%%%%%%%%%%%%%%%%%%%%%%%%%%%%%%%%%%%%%%%%%%%%%%%%%%%%%%%%%%%%%%%%%%%%%%%%%%%%%%%%%%%%%%%%%
\vspace{1em}
\LemAlgOneDiscretizationError*
\begin{proof}\label{proof:lem:alg1-discretization-error}
    First, recall the definition of $\widetilde{R}_T(n,K)$ below.
    \begin{align*}
        \widetilde{R}_T(n,K)&:=2\Big(\min_{(c^\star,p^\star)\in\I_K^{n+1}}\sum_{t\in[T]}\elt(c^\star,p^\star)-\inf_{(c^\star,p^\star)\in[0,1]^{n+1}}\sum_{t\in[T]}\elt(c^\star,p^\star) \Big)\\
        &=\sup_{(c^\star,p^\star)\in[0,1]^{n+1}}\sum_{t,i\in[T]\times[n]}\eprofit_{t,i}(c^\star_i,p^\star)-\max_{(c^\star,p^\star)\in\I_K^{n+1}}\sum_{t,i\in[T]\times[n]}\eprofit_{t,i}(c^\star_i,p^\star),
    \end{align*}
    where the second line follows from the definition of $\elt$ in \eqref{eq:elti}. Fix $\zeta>0$, and let $(\hat{c},\hat{p})\in[0,1]^{n+1}$ be such that
    \begin{align*}
        \sum_{t,i\in[T]\times[n]}\eprofit_{t,i}(\hat{c}_i,\hat{p})\geq\sup_{(c,p)\in[0,1]^{n+1}}\sum_{t,i\in[T]\times[n]}\eprofit_{t,i}(c_i,p)-\zeta.
    \end{align*}
    Then, by the monotonicity of $\edti(\cdot,\cdot)$ (Assumption \ref{assump:monotonic}), for any $\eps\geq 0$,
    \begin{align*}
        \edti(\hat{c}_i,\hat{p})\leq \edti(\hat{c}_i+\eps,\hat{p}-\eps),\ \ \forall i\in[n].
    \end{align*}
    Hence,
    \begin{align*}
        \eprofit_{t,i}(\hat{c}_i+\eps,\hat{p}-\eps)&=(\hat{p}-\eps)\cdot\edti(\hat{c}_i+\eps,\hat{p}-\eps)-(\hat{c}_i+\eps)\\
        &\geq (\hat{p}-\eps)\cdot\edti(\hat{c}_i,\hat{p})-(\hat{c}_i+\eps) \\
        &\geq \eprofit_{t,i}(\hat{c}_i,\hat{p})-2\eps.
    \end{align*}
    Since there exists $(c^\star,p^\star)\in\I_K^{n+1}$ such that $0\leq c^\star_i-\hat{c}_i\leq K^{-1}$, $\forall i\in[n]$ and $0\leq \hat{p}-p^\star\leq K^{-1}$, the above implies
    \begin{align*}
        \eprofit_{t,i}(c^\star_i,p^\star)&\geq \eprofit_{t,i}(\hat{c}_i,\hat{p})-2K^{-1}.
    \end{align*}
    Finally, summing the above over $(t,i)\in[T]\times[n]$, we have that
    \begin{align*}
        \sup_{(c^\star,p^\star)\in[0,1]^{n+1}}\sum_{t,i\in[T]\times[n]}\eprofit_{t,i}(c^\star_i,p^\star)\leq \sum_{t,i\in[T]\times[n]}\eprofit_{t,i}(c^\star,p^\star) +\frac{2nT}{K}+\zeta.
    \end{align*}
    Therefore, $\widetilde{R}_T(n,K)\leq\frac{2nT}{K}+\zeta$, and since $\zeta>0$ is arbitrary, $\widetilde{R}_T(n,K)\leq\frac{2nT}{K}$.
\end{proof}
%%%%%%%%%%%%%%%%%%%%%%%%%%%%%%%%%%%%%%%%%%%%%%%%%%%%%%%%%%%%%%%%%%%%%%%%%%%%%%%%%%%%%%%%%%%%%%%%%%%
\vspace{1em}
\LemAlgOneStepThreeOne*
\begin{proof}\label{proof:lem:alg1-step3-1}
    For the sake of convenience, we restate the definition of $\elti(c^\star,p^\star)$.
    \begin{align*}
        \elti(c_i^\star,p^\star):=\frac{1}{2}\left(1-\eprofit_{t,i}(c^\star_i,p^\star)\right),\quad \eprofit_{t,i}(c^\star,p^\star):=p^\star\,\edti(c^\star_i,p^\star)-c^\star_i.
    \end{align*}
    Then, the first bound follows from the following calculations:
    \begin{align*}
        \expec\!\sbra{\lossF_{t,i}(c^\star_i,p^\star)}&=    \expec\!\sbra{\frac{\lti\,\ind\!\sbra{\cti=c^\star_i,\pt=p^\star}}{\qti(\cti|\,\pt)\rbra{\qtz(\pt)+\gamma}}}\\
        &=\expec\!\sbra{\sum_{c'_i\,,\,p'}\qti(c'_i\,,p')\cdot\frac{\elti(c'_i,p')\,\ind\!\sbra{c'_i=c^\star_i,p'=p^\star}}{\qti(c'_i|\,p')\rbra{\qtz(p')+\gamma}}}\\
        &\leq\expec\!\sbra{\sum_{c'_i\,,\,p'}\qti(c'_i\,,p')\cdot\frac{\elti(c'_i,p')}{\qti(c'_i\,,p')}\cdot\ind\!\sbra{c'_i=c^\star_i,p'=p^\star}} \\
        &=\sum_{c'_i\,,\,p'}\elti(c'_i,p')\,\ind\!\sbra{c'_i=c^\star_i,p'=p^\star}\\
        &=\elti(c^\star_i,p^\star).
    \end{align*}
    And the second bound follows as
    \begin{align*}
        \expec\!\sbra{\innp{\qti\,,\lossF^2_{t,i}}_{p^\star}}
        &\leq\expec\!\sbra{\sum_{c'_i\,,\,p'}\sum_{c_i} \qti(c'_i\,,p')\,\qti(c_i|\,p^\star) \cdot\frac{\ind\!\sbra{c'_i=c_i,p'=p^\star }}{q^2_{t,i}(c'_i|\,p')\rbra{\qtz(p')+\gamma}^2   }} \tag{as $\lti\leq 1$}\\
        &=\expec\!\sbra{\sum_{c_i} \frac{\qti(c_i\,,p^\star)\, \qti(c_i|\,p^\star)}{q^2_{t,i}(c_i|\,p^\star)\rbra{\qtz(p^\star)+\gamma}^2}}\\
        &=\expec\!\sbra{\sum_{c_i} \frac{\qti(c_i\,,p^\star)}{\qti(c_i|\,p^\star)\rbra{\qtz(p^\star)+\gamma}^2}}\\
        &\leq\expec\!\sbra{\sum_{c_i} \frac{\qti(c_i\,,p^\star)}{\qti(c_i|\,p^\star)\cdot\qtz(p^\star)\cdot\!\rbra{\qtz(p^\star)+\gamma}}}\\
        &=\expec\!\sbra{\sum_{c_i} \frac{\qti(c_i\,,p^\star)}{\qti(c_i\,,p^\star)\rbra{\qtz(p^\star)+\gamma}}}\\
        &=\expec\!\sbra{\frac{\sum_{c_i\in\I_K}1}{\qtz(p^\star)+\gamma}}\\
        &=\expec\!\sbra{\frac{|\I_K|}{\qtz(p^\star)+\gamma}}.
    \end{align*}
\end{proof}
%%%%%%%%%%%%%%%%%%%%%%%%%%%%%%%%%%%%%%%%%%%%%%%%%%%%%%%%%%%%%%%%%%%%%%%%%%%%%%%%%%%%%%%%%%%%%%%%%%%
\vspace{1em}
\LemAlgOneStepThreeTwo*
\begin{proof}\label{proof:lem:alg1-step3-2}
    First, we state the expectation of \eqref{eq:alg1-exp-bounds-1} in Lemma \ref{lem:alg1-exp-bounds} below.
    \begin{align*}
        \sum_{t\in[T]}\expec\!\sbra{\innp{\qti\,,\lossF_{t,i}}_{p^\star}}-\sum_{t\in[T]}\expec\!\sbra{\lossF_{t,i}(c^\star_i,p^\star)}\leq \eta\sum_{t\in[T]}\expec\!\sbra{\innp{\qti\,,\lossF^2_{t,i}}_{p^\star}}+\frac{1}{\eta}\log|\I_K|,
    \end{align*}
    where we used the linearity of expectation. Then, by applying Lemma \ref{lem:alg1-step3-1}, we obtain the following bound.
    \begin{align*}
        \sum_{t\in[T]}\expec\!\sbra{\innp{\qti\,,\lossF_{t,i}}_{p^\star}}-\sum_{t\in[T]}\elti(c^\star_i,p^\star)\leq\sum_{t\in[T]}\expec\!\sbra{\frac{\eta|\I_{K}|}{\qtz(p^\star)+\gamma}}+\frac{1}{\eta}\log|\I_K|.
    \end{align*}
    We can further rewrite the above by unfolding $\innp{\qti\,,\lossF_{t,i}}_{p^\star}$:
    \begin{align*}
        \innp{\qti\,,\lossF_{t,i}}_{p^\star}&=\sum_{c_i}\qti(c_i|p^\star)\cdot\frac{\lti\,\ind[\cti=c_i,\pt=p^\star]}{\qti(\cti|\pt)\rbra{\qtz(\pt)+\gamma}}\\
        &=\qti(\cti|p^\star)\cdot \frac{\lti\,\ind[\pt=p^\star]}{\qti(\cti|\pt)\rbra{\qtz(\pt)+\gamma}}\\
        &=\qti(\cti|\pt)\cdot \frac{\lti\,\ind[\pt=p^\star]}{\qti(\cti|\pt)\rbra{\qtz(\pt)+\gamma}}\\
        &=\frac{\lti\ind[\pt=p^\star]}{\qtz(\pt)+\gamma},
    \end{align*}
    which gives
    \begin{align*}
        \expec\Bigg[\sum_{t\in[T]}\rbra{\frac{\lti\,\ind[\pt=p^\star]}{\qtz(\pt)+\gamma}-\frac{\eta|\I_K|}{\qtz(p^\star)+\gamma}}\Bigg]&\leq\sum_{t\in[T]}\elti(c^\star_i,p^\star)+\frac{1}{\eta}\log|\I_K|.
        \numberthis\label{eq:proof:lem:alg1-step2-2-1}
    \end{align*}
    By summing \eqref{eq:proof:lem:alg1-step2-2-1} over $i\in[n]$ and dividing it by $n$, we obtain the stated bound.
    \iffalse
    Recall the definition of $\lossH_t(p)$:
    \begin{align*}
        \lossH_t(p):=\frac{1}{n}\rbra{\frac{\lt\,\ind\!\sbra{\pt=p}}{\qtz(\pt)+\gamma}}+\eta|\I_K|\rbra{\frac{1}{\gamma}-\frac{1}{\qtz(p)+\gamma}},
    \end{align*}
    where $\lt:=\sum_{i\in[n]}\lti$. By summing \eqref{eq:proof:lem:alg1-step2-2-1} over $i\in[n]$ and substituting the definition of $\lossH_t(p^\star)$, we obtain 
    \begin{align*}
        \expec\Bigg[\sum_{t\in[T]}\lossH_t(p^\star)\Bigg]-\frac{\eta|\I_K|T}{\gamma}\leq\frac{1}{n}\sum_{t\in[T]}\elt(c^\star,p^\star)+\frac{1}{\eta}\log|\I_K|,
    \end{align*}
    which is the stated bound.
    \fi
\end{proof}
%%%%%%%%%%%%%%%%%%%%%%%%%%%%%%%%%%%%%%%%%%%%%%%%%%%%%%%%%%%%%%%%%%%%%%%%%%%%%%%%%%%%%%%%%%%%%%%%%%%
\vspace{1em}
\LemAlgOneAlgorithmLossBound*
\begin{proof}\label{proof:lem:alg1-algorithm-loss-bound}
    The lemma follows from straightforward calculations. First, for each $t\in[T]$,
    \begin{align*}
        \innp{\qtz\,,\lossH_t}-\frac{\eta|\I_K|}{\gamma}&=\frac{1}{n}\sum_{p\in\I_K}\rbra{\frac{\lt\,\ind\!\sbra{\pt=p}}{\qtz(\pt)+\gamma}\cdot \qtz(p)}-\eta|\I_K|\sum_{p\in\I_K}\frac{\qtz(p)}{\qtz(p)+\gamma}\\
        &=\frac{1}{n}\,\frac{\lt\,\qtz(\pt)}{\qtz(\pt)+\gamma}-\eta|\I_K|\sum_{p\in\I_K}\frac{\qtz(p)}{\qtz(p)+\gamma}\\
        &\geq\frac{1}{n}\,\frac{\lt\,\qtz(\pt)}{\qtz(\pt)+\gamma}-\eta|\I_K|^2\\
        &=\frac{1}{n}\!\rbra{\lt-\frac{\gamma\,\lt}{\qtz(\pt)+\gamma}}-\eta|\I_K|^2.
    \end{align*}
    Next, we take the expectation to the above:
    \begin{align*}
        \expec\bigg[\innp{\qtz\,,\lossH_t}\bigg]-\frac{\eta|\I_K|}{\gamma}&\geq \frac{1}{n}\,\expec\big[\lt\big]-\gamma\,\expec\bigg[\frac{\lt/n}{\qtz(\pt)+\gamma}\bigg]-\eta|\I_K|^2\\
        &\geq \frac{1}{n}\,\expec\big[\lt\big]-\gamma|\I_K|-\eta|\I_K|^2\\
        &\geq \frac{1}{n}\,\expec\big[\lt\big]-2\eta|\I_K|^2 \tag{as $\gamma:=\eta$ and $|\I_K|^2\geq|\I_K|$}.
    \end{align*}
    This implies the stated bound.
\end{proof}
%%%%%%%%%%%%%%%%%%%%%%%%%%%%%%%%%%%%%%%%%%%%%%%%%%%%%%%%%%%%%%%%%%%%%%%%%%%%%%%%%%%%%%%%%%%%%%%%%%%
\subsection{Lemmas for regret lower bound proof}\label{app:deferred-proofs-sec3-lb}
\LemLowerBoundHardenvironmentProperties*
\begin{proof}\label{proof:lem:lb-hard-environment-properties}
    \begin{enumerate}
        \item By construction, $d(c,p)$ is non-decreasing w.r.t. $c\in[0,1]$ when $p\notin (p^\star-\eps/2,p^\star]$. Moreover, $d(\cdot, p)=d(\cdot, p')$ for any $p, p'\in(p^\star-\eps/2,p^\star]$. Hence, to show the first part, it suffices to show the following:
        \begin{align*}
            d(c^\star-\eps/2,p^\star)=\frac{2c^\star-\eps}{2p^\star}\leq d(c^\star,p^\star)=\frac{2c^\star+\eps}{2p^\star+\eps} \leq d(c^\star+\eps/2,p^\star)=\frac{2c^\star+\eps}{2p^\star}.
        \end{align*}
        Since the latter inequality is true, we only need to check the former inequality:
        \begin{align*}
            \frac{2c^\star-\eps}{2p^\star}\leq \frac{2c^\star+\eps}{2p^\star+\eps}\iff\eps\geq 2c^\star-4p^\star,
        \end{align*}
        which is satisfied as $\eps>0$, $c^\star\leq \frac{1}{2}$ and $p^\star\geq \frac{1}{2}$.
    
        \item By construction, $d(c,p)$ is non-increasing w.r.t. $p\in[0,1]$ when $c\notin [c^\star,c^\star+\eps/2)$. Moreover, $d(c, \cdot)=d(c', \cdot)$ for any $c, c'\in[c^\star,c^\star+\eps/2)$. Hence, to show the first part, it suffices to show the following:
        \begin{align*}
            d(c^\star,p^\star+\eps/2)=\frac{2c^\star}{2p^\star+\eps}\leq d(c^\star,p^\star)=\frac{2c^\star+\eps}{2p^\star+\eps}\leq d(c^\star,p^\star-\eps/2)=\frac{2c^\star}{2p^\star-\eps},
        \end{align*}
        which leads to the following inequalities:
        \begin{align*}
            \frac{2c^\star+\eps}{2p^\star+\eps}\leq\frac{2c^\star}{2p^\star-\eps}\impliedby\eps\geq 2p^\star-4c^\star,\, p^\star>\eps/2,
        \end{align*}
        where the latter is satisfied for $(c^\star,p^\star)\in\S$ as $c^\star\geq \frac{2}{5},\, \frac{3}{5}\leq p^\star \leq \frac{4}{5}$ and $\eps>0$.
        \item A straightforward computation yields
        \begin{align*}
            \eprofit(c^\star,p^\star)&:=p^\star\cdot d(c^\star,p^\star)-c^\star=p^\star\cdot\frac{c^\star+\eps/2}{p^\star+\eps/2}-c^\star\\
            &=c^\star\rbra{\frac{p^\star}{p^\star+\eps/2}-1}+\frac{\eps p^\star}{2p^\star+\eps}\\
            &=\frac{\eps}{2}\rbra{\frac{p^\star-c^\star}{p^\star+\eps/2}}\\
            &\geq\frac{\eps}{2}\rbra{p^\star-c^\star}\geq\frac{\eps}{20},
        \end{align*}
        where the last inequality follows from $p^\star\geq\frac{3}{5}$ and $c^\star\leq\frac{1}{2}$ for any $(c^\star,p^\star)\in\S$.
        \item For any $(c,p)\in\C\times\P$ and $(c,p)\neq(c^\star,p^\star)$,
        \begin{align*}
            \eprofit(c,p)=p\cdot \frac{1}{2p}\cdot 2c - c =0.
        \end{align*}
        For the second part, note that $g(c)=g(c')$: if $c\leq\frac{1}{2}$, then $c'=\lfloor 2cK\rfloor/(2K)$ and $\lfloor 2c'K\rfloor=\lfloor 2cK\rfloor$, and if $c>\frac{1}{2}$, then $g(c)=1=g(\frac{1}{2})$. Since $\mathbf{v}$ takes values in $\P$, $\ind[\mathbf{v}\geq p]=\ind[\mathbf{v}\geq p']$. Moreover, $(c,p)\in[c^\star,c^\star+\eps/2)\times(p^\star-\eps/2,p^\star]$ if and only if $(c',p')=(c^\star,p^\star)$, so $b$ and $\mathbf{v}$ have the same distributions under $(c,p)$ and $(c',p')$. Hence $d$ has the same distribution under both, and since $p'\geq p$ and $c'\leq c$, we have $\eprofit(c',p')=p'\,\ed(c',p')-c'\geq p\,\ed(c,p)-c=\eprofit(c,p)$. The same argument, without the bump, applies to $\envbase$.
    \end{enumerate}
\end{proof}
%%%%%%%%%%%%%%%%%%%%%%%%%%%%%%%%%%%%%%%%%%%%%%%%%%%%%%%%%%%%%%%%%%%%%%%%%%%%%%%%%%%%%%%%%%%%%%%%%%%
\vspace{1em}
\LemLowerBoundKLDecomposition*
    \begin{proof}\label{proof:lem:lb-kl-decomposition}
    For each $t\in[T]$, let $o_{t}:=\{d_{\tau,i},c_{\tau,i},p_\tau\}_{\tau\in[t],i\in[n]}$. Let $\pi^\A_t(c_{t,1},\dots,c_{t,n},p_t\,|\,o_{t-1})$ be $\A$'s policy function at given round $t$, that is, the probability of $\A$ choosing $(c_{t,1},\dots,c_{t,n},p_t)\in\C^n\times\P$ given the past history $o_{t-1}$. For each $i\in[n]$, let $u(\edti|c_{t,i},p_t)$ (resp. $u'(\edti|c_{t,i},p_t)$) be the probability that $\dti=\edti\in\{0,1\}$ given $c_{t,i}$ and $p_t$ under $\E_\textup{base}$ (resp. $\E_{(c^\star,p^\star)}$).
    
    Then, 
    \begin{align*}
        \log\frac{\textup{d}\mathbb{P}^\A_{\textup{base}}}{\textup{d}\ppAcstarpstar}(o_T)&=\log\frac{\prod^T_{t=1}\pi^\A_t(c_{t,1},\dots,c_{t,n},p_t\,|\,o_{t-1})\cdot\prod^n_{i=1}u(\edti\,|\,c_{t,i},p_t)}{\prod^T_{t=1}\pi^\A_t(c_{t,1},\dots,c_{t,n},p_t\,|\,o_{t-1})\cdot\prod^n_{i=1} u'(\edti\,|\,c_{t,i},p_t)}\\
        &=\sum_{t,i\in[T]\times[n]}\log\frac{u(\edti|c_{t,i},p_t)}{u'(\edti|c_{t,i},p_t)}
    \end{align*}
    Hence,
    \begin{align*}
        \textup{KL}\Big(\ppAbase  \,\Big|\Big|\,\ppAcstarpstar\Big)&=\eeAbase\Bigg[\log\frac{\textup{d}\ppAbase}{\textup{d}\ppAcstarpstar}(\mathbf{o}_T) \Bigg]\\
        &=\sum_{t,i\in[T]\times[n]}\eeAbase\Bigg[\log\frac{u(\edti\,|\,\mathbf{c}_{t,i},\pt)}{u'(\edti\,|\,\mathbf{c}_{t,i},\pt)} \Bigg]\\
        &=\sum_{t,i\in[T]\times[n]}\eeAbase\Bigg[\sum_{c,p\in\C\times\P}\bm{\chi}_{t,i}(c,p)\cdot\log\frac{u(\edti\,|\,\mathbf{c}_{t,i},\pt)}{u'(\edti\,|\,\mathbf{c}_{t,i},\pt)} \Bigg] \\
        &\leq\sum_{t,i\in[T]\times[n]}54\eps^2\cdot\eeAbase\big[\bm{\chi}_{t,i}(c^\star,p^\star) \big], \tag{by Lemma \ref{lem:klbound}}
    \end{align*}
    where the third line is due to $\sum_{c,p\in\C\times\P}\bm{\chi}_{t,i}(c,p)=1$ for all $i\in[n]$, and the last line we used the fact that $u(\,\cdot\,|\cti,\pt)=u'(\,\cdot\,|\cti,\pt)$ whenever $\cti\neq c^\star$ or $\pt\neq p^\star$.
\end{proof}
%%%%%%%%%%%%%%%%%%%%%%%%%%%%%%%%%%%%%%%%%%%%%%%%%%%%%%%%%%%%%%%%%%%%%%%%%%%%%%%%%%%%%%%%%%%%%%%%%%%
\vspace{1em}
\begin{lemma}
For any $(c^\star,p^\star)\in\S$ and $0<\eps\leq\frac{1}{20}$,
\begin{align*}
    \textup{KL}\!\rbra{\textup{Ber}(2c^\star)\otimes\textup{Ber}\rbra{\frac{1}{2p^\star}} \,\middle|\middle|\, \textup{Ber}(2c^\star+\eps)\otimes\textup{Ber}\rbra{\frac{1}{2p^\star+\eps}}}\leq 54\eps^2.
\end{align*}
\label{lem:klbound}
\end{lemma}

\begin{proof}\label{proof:klbound}
    Let $\textup{kl}(p,q):=p\log \frac{p}{q}+(1-p)\log\frac{1-p}{1-q}$. Then,
    \begin{align*}
        \textup{KL}\!\rbra{\textup{Ber}(2c^\star)\otimes\textup{Ber}\rbra{\frac{1}{2p^\star}}\,\middle|\middle|\,\textup{Ber}(2c^\star+\eps)\otimes\textup{Ber}\rbra{\frac{1}{2p^\star+\eps}}}=\textup{kl}\rbra{2c^\star,2c^\star+\eps}+\textup{kl}\!\rbra{\frac{1}{2p^\star},\frac{1}{2p^\star+\eps}}.
    \end{align*}
    Next, we bound two terms in the RHS using $\frac{x}{1+x}\leq\log(1+x) \leq x$ for all $x>-1$.
    \begin{align*}
        \textup{kl}\!\rbra{2c^\star,2c^\star+\eps}&=2c^\star\log\frac{2c^\star}{2c^\star+\eps}+(1-2c^\star)\log\rbra{\frac{1-2c^\star}{1-2c^\star-\eps}}\\
        &=-2c^\star\log\rbra{1+\frac{\eps}{2c^\star}}+(1-2c^\star)\log\rbra{1+\frac{\eps}{1-2c^\star-\eps}}\\
        &\leq -2c^\star\rbra{\frac{\eps/2c^\star}{1+\eps/2c^\star}}+\frac{\eps(1-2c^\star)}{1-2c^\star-\eps}\\
        &=-\frac{2\eps c^\star}{2c^\star+\eps}+\frac{(1-2c^\star)\eps}{1-2c^\star-\eps}\\
        &=\frac{\eps^2}{(1-2c^\star-\eps)(2c^\star+\eps)}\\
        &\leq\frac{\eps^2}{(0.1-\eps)(0.4+\eps)}\tag{as $\frac{2}{5}\leq c^\star \leq \frac{9}{20}$}\\
        &\leq 50\eps^2, \tag{as $\eps\leq\frac{1}{20}$}
    \end{align*}
    
    \begin{align*}
        \textup{kl}\rbra{\frac{1}{2p^\star},\frac{1}{2p^\star+\eps}}&=\frac{1}{2p^\star}\log\frac{2p^\star+\eps}{2p^\star}+\rbra{1-\frac{1}{2p^\star}}\log\frac{1-\frac{1}{2p^\star}}{1-\frac{1}{2p^\star+\eps}}\\
        &=\frac{1}{2p^\star}\log\rbra{1+\frac{\eps}{2p^\star}}+\rbra{1-\frac{1}{2p^\star}}\log\rbra{1-\frac{\eps}{2p^\star(2p^\star+\eps-1)}}\\
        &\leq\frac{\eps}{4(p^\star)^2}-\rbra{1-\frac{1}{2p^\star}}\cdot\frac{\eps}{2p^\star(2p^\star+\eps-1)}\\
        &=\frac{\eps}{4(p^\star)^2}-\frac{\eps}{4(p^\star)^2}\cdot\frac{2p^\star-1}{2p^\star-1+\eps}\\
        &=\frac{\eps}{4(p^\star)^2}\rbra{1-\frac{2p^\star-1}{2p^\star-1+\eps}}\\
        &=\frac{\eps^2}{4(p^\star)^2(2p^\star-1+\eps)}\leq\frac{\eps^2}{4\rbra{\frac{3}{5}}^2\cdot\rbra{\frac{1}{5}}}\leq 4\eps^2. \tag{as $\frac{3}{5}\leq p^\star$}
    \end{align*}
\end{proof}
\section{Proofs omitted from Section \ref{sec:cost-concave}}\label{app:deferred-proofs-sec4}
%%%%%%%%%%%%%%%%%%%%%%%%%%%%%%%%%%%%%%%%%%%%%%%%%%%%%%%%%%%%%%%%%%%%%%%%%%%%%%%%%%%%%%%%%%%%%%%%%%%
We start this appendix with a regret bound for $\{\lossF_{t,i}\}_{t\in[T]}$ \eqref{eq:alg2-fti}, which is a generalization of Lemma \ref{lem:alg1-exp-bounds} to continuous settings. While this regret bound follows from a standard analysis~\citep{hazan2016introduction,bubeck2012regret} of the exponential weights algorithm, we give its proof below for completeness.
\begin{lemma}
    Let $q':=\uniform(\I_\del)$. Then, for any $q\in\Delta(\I_\del)$ with density $q(\cdot)$ and any $p^\star\in\I_K$,
    \begin{align*}
        \sum_{t\in[T]}\innp{u_{t,i},\lossF_{t,i}}_{p^\star}-\sum_{t\in[T]}\innp{q,\lossF_{t,i}}_{p^\star}\leq \frac{1}{\eta}\,\emph{KL}(q||q')+\frac{\eta}{2}\sum_{t\in[T]}\innp{u_{t,i},\lossF^2_{t,i}}_{p^\star},
        \numberthis\label{eq:alg2-continuous-EXP-bound}
    \end{align*}
    where $u_{t,i}(c_i|\,p)$ is the distribution given in Line \ref{line:alg2-uti} of Algorithm \ref{alg:2}.
    \label{lem:alg2-continuous-EXP-bound}
\end{lemma}

\begin{proof}
    Recall the definition of $u_{t,i}(c_i|\,p)$ given in line \ref{line:alg2-uti} of Algorithm \ref{alg:2} which can be rewritten as follows:
    \begin{align*}
        u_{t,i}(c_i|\,p):=\frac{\exp\rbra{-\eta\sum^{t-1}_{\tau=1}\lossF_{\tau,i}(c_i,p)}}{\int_{\I_\del}\exp\rbra{-\eta\sum^{t-1}_{\tau=1}\lossF_{\tau,i}(c_i,p)}dc_i}.
    \end{align*}
    Then, rewrite $\innp{u_{t,i},\lossF_{t,i}}_p:=\int_{\I_\del} u_{t,i}(c_i|p)\lossF(c_i,p)dc_i$ as follows:
    \begin{align*}
        \innp{u_{t,i},\lossF_{t,i}}_p=\bunderbrace{\innp{u_{t,i},\lossF_{t,i}}_p+\frac{1}{\eta}\log\innp{u_{t,i},\exp\rbra{-\eta\lossF_{t,i}}}_p}{A_{t,i}(p)}\ \ \bunderbrace{-\frac{1}{\eta}\log\innp{u_{t,i},\exp\rbra{-\eta\lossF_{t,i}}}_p}{B_{t,i}(p)}.
    \end{align*}
    An upper bound for $A_{t,i}(p)$ is given as follows:
    \begin{align*}
        A_{t,i}(p)&:=\innp{u_{t,i},\lossF_{t,i}}_p+\frac{1}{\eta}\log\innp{u_{t,i},\exp\rbra{-\eta\lossF_{t,i}}}_p\\
        &\leq\innp{u_{t,i},\lossF_{t,i}}_p+\frac{1}{\eta}\rbra{\innp{u_{t,i},\exp\rbra{-\eta\lossF_{t,i}}}_p-1} \tag{as $\log x\leq x-1$ for $x\geq 0$}\\
        &=\frac{1}{\eta}\innp{u_{t,i},\,\exp\rbra{-\eta\lossF_{t,i}}-1+\eta \lossF_{t,i}}_p \tag{as $\innp{u_{t,i},1}_p=1$}\\
        &\leq \frac{\eta}{2}\innp{u_{t,i},\lossF^2_{t,i}}_p \tag{as $e^{-x}-1+x\leq\frac{x^2}{2}$ for $x\geq 0$}
    \end{align*}
    Next, we rewrite $B_{t,i}(p)$ as follows:
    \begin{align*}
        B_{t,i}(p)&:=-\frac{1}{\eta}\log\innp{u_{t,i},\exp\rbra{-\eta\lossF_{t,i}}}_p\\
        &=-\frac{1}{\eta}\log\frac{\int_{\I_\del}\exp\rbra{-\eta\sum^t_{\tau=1}\lossF_{\tau,i}(c_i,p)}dc_i}{\int_{\I_\del}\exp\rbra{-\eta\sum^{t-1}_{\tau=1}\lossF_{\tau,i}(c_i,p)}dc_i}\\
        &=\frac{1}{\eta}\rbra{\Phi_{t-1,i}(p)-\Phi_{t,i}(p)},
    \end{align*}
    where $\Phi_{t,i}(p):=\log\int_{\I_\del}\exp\rbra{-\eta\sum^t_{\tau=1}\lossF_{\tau,i}(c_i,p)}dc_i$ and $\Phi_{0,i}(p):=\log \int_{\I_\del}dc_i=\log |\I_\del|$.
    
    Therefore,
    \begin{align*}
        \sum_{t\in[T]} \innp{u_{t,i},\lossF_{t,i}}_p &= \sum_{t\in[T]} A_{t,i}(p)-B_{t,i}(p)\\
        &\leq\frac{\eta}{2}\sum_{t\in[T]}\innp{u_{t,i},\lossF^2_{t,i}}_p+\frac{1}{\eta}\sum_{t\in[T]}\rbra{\Phi_{t-1,i}(p)-\Phi_{t,i}(p)}\\
        &\leq\frac{\eta}{2}\sum_{t\in[T]}\innp{u_{t,i},\lossF^2_{t,i}}_p-\frac{1}{\eta}\rbra{\Phi_{T,i}(p)-\Phi_{0,i}(p)}.
    \end{align*}
    
    For any $q\in\Delta(\I_\del)$ with density $q(\cdot)$, the following holds from Donsker-Varadhan inequality ~\citep{cover1999elements}:
    \begin{align*}
        \frac{1}{\eta}\rbra{\Phi_{T,i}(p)-\Phi_{0,i}(p)}&=\frac{1}{\eta}\log\frac{1}{|\I_\del|}\int_{\I_\del}\exp\rbra{-\eta\sum_{t\in[T]}\lossF_{t,i}(c_i,p)}dc_i\\
        &\geq -\sum_{t\in[T]}\innp{q,\lossF_{t,i}(c_i,p)}_p-\frac{1}{\eta}\,\text{KL}\rbra{q || q'},
    \end{align*}
    where $q':=\uniform(\I_\del)$. Therefore, we obtain the stated bound by setting $p=p^\star$.
\end{proof}
%%%%%%%%%%%%%%%%%%%%%%%%%%%%%%%%%%%%%%%%%%%%%%%%%%%%%%%%%%%%%%%%%%%%%%%%%%%%%%%%%%%%%%%%%%%%%%%%%%%
\vspace{1em}
\LemAlgTwoConvexLossProperties*
\begin{proof}\label{proof:lem:alg2-convex-loss-properties}
    Let $x<y<z \in\I_\del$ and $f:\I_\del\rightarrow\mathbb{R}$ be a convex function. Since $y=\frac{z-y}{z-x}\cdot x+\frac{y-x}{z-x}\cdot z$, we have $f(y)\leq\frac{z-y}{z-x} f(x) + \frac{y-x}{z-x} f(z)$ which is equivalent to 
    \begin{align*}
        \frac{f(y)-f(x)}{y-x}\leq\frac{f(z)-f(y)}{z-y}.
        \numberthis\label{eq:convex-slope-monotonicity}
    \end{align*}
    For any $p^\star\in\I_K$, $\del>0$, $t\in[T]$ and $i\in[n]$, $\elti(\,\cdot\,, p^\star)$ is convex over $[0,1]\supseteq \I_\del$, hence it satisfies the above inequality. Therefore, for any $\del < x\leq y \leq 1-\del$, 
    \begin{align*}
        \frac{\elti(\del, p^\star)-\elti(0, p^\star)}{\del-0}\leq\frac{\elti(y, p^\star)-\elti(x, p^\star)}{y-x}\leq\frac{\elti(1, p^\star)-\elti(1-\del, p^\star)}{1-(1-\del)}.
    \end{align*}
    Since $\elti(\,\cdot\,, p^\star)\leq 1$, the above implies
    \begin{align*}
        -\frac{1}{\del}\leq \frac{\elti(y, p^\star)-\elti(x, p^\star)}{y-x} \leq \frac{1}{\del}.
    \end{align*}
    Hence, $\left|\elti(x, p^\star)-\elti(y, p^\star)\right|\leq \del^{-1}|x-y|$ which proves the first claim.
    
    The second claim follows from \eqref{eq:convex-slope-monotonicity} with $x=c^\star_i<y=\del<z=1$ and $f(\cdot)=\elti(\cdot,p^\star)$:
    \begin{align*}
        \frac{\elti(\del,p^\star)-\elti(c^\star_i,p^\star)}{\del-c^\star}&\leq\frac{\elti(1,p^\star)-\elti(\del,p^\star)}{1-\del} \implies \\
        \frac{\elti(\del,p^\star)-\elti(c^\star_i,p^\star)}{\del}&\leq 2,
    \end{align*}
    as $\del\leq 1/2$ and $\elti(\,\cdot\,, p^\star)\leq 1$. The final claim follows similarly with $x=0<y=1-\del<z=c^\star$.
\end{proof}
%%%%%%%%%%%%%%%%%%%%%%%%%%%%%%%%%%%%%%%%%%%%%%%%%%%%%%%%%%%%%%%%%%%%%%%%%%%%%%%%%%%%%%%%%%%%%%%%%%%
\vspace{1em}
\LemAlgTwoAlgorithmLossBound*
\begin{proof}\label{proof:lem:alg2-algorithm-loss-bound}
    The stated inequality follows from:
    \begin{align*}
        \expec\!\sbra{\innp{\qtz\,,\lossH_t}}-\frac{3\eta}{\gamma}\log\frac{e}{\eps}&=\frac{1}{n}\sum_{i\in[n]}\expec\!\sbra{\sum_{p\in\I_K}\qtz(p)\cdot\frac{\lti\,\ind\!\sbra{\pt=p}}{\qtz(\pt)+\gamma}}
        -\expec\!\sbra{\sum_{p\in\I_K}\qtz(p)\cdot\frac{3\eta\log(e/\eps)}{\qtz(p)+\gamma}}\\
        &\geq\frac{1}{n}\sum_{i\in[n]}\expec\!\sbra{\sum_{p\in\I_K}\qtz(p)\cdot\frac{\lti\,\ind\!\sbra{\pt=p}}{\qtz(\pt)+\gamma}}-3\eta|\I_K|\log(e/\eps)\\
        &=\frac{1}{n}\sum_{i\in[n]}\expec\!\sbra{\frac{\lti\,\qtz(\pt)}{\qtz(\pt)+\gamma}}-3\eta|\I_K|\log(e/\eps)\\
        &=\frac{1}{n}\sum_{i\in[n]}\expec\!\sbra{\lti-\frac{\gamma\,\lti}{\qtz(\pt)+\gamma}}-3\eta|\I_K|\log(e/\eps)\\
        &\geq\frac{1}{n}\,\expec\!\sbra{\lt}-\expec\!\sbra{\frac{\gamma}{\qtz(\pt)}}-3\eta|\I_K|\log(e/\eps)\tag{as $\lti\leq 1$}\\
        &=\frac{1}{n}\,\expec\!\sbra{\lt}-4\eta|\I_K|\log(e/\eps),
    \end{align*}
    where the last line is due to $\expec[1/\qtz(\pt)]=\sum_{p\in\I_K}1=|\I_K|$ and $\gamma:=\eta\log(e/\eps)$. 
\end{proof}

%%%%%%%%%%%%%%%%%%%%%%%%%%%%%%%%%%%%%%%%%%%%%%%%%%%%%%%%%%%%%%%%%%%%%%%%%%%%%%%%%%%%%%%%%%%%%%%%%%%
\subsection{Proof of Lemma \ref{lem:alg2-comparator-loss-bound}}
\label{subsec:proof-of-comparator-loss-bound}
In this section, we give a proof of Lemma \ref{lem:alg2-comparator-loss-bound} which gives a lower bound for the comparator loss. Specifically, we prove this lemma using the following lemmas:

\begin{lemma}
    Let $u_{t,i}(c_i|\,p)$ is the distribution given in Line \ref{line:alg2-uti} of Algorithm \ref{alg:2} and $\qti(c_i|\,p)=K_\eps[u_{t,i}(\,\cdot\,|\,p)]u_{t,i}(c_i|\,p)$. Then, the following hold:
    \begin{enumerate}
        \vspace{0.5em}
        \item 
        $\displaystyle \expec\!\sbra{\innp{u_{t,i},\lossF_{t,i}}_{p^\star}}=\expec\!\sbra{\frac{\lti\,\ind[\pt=p^\star]}{\qtz(\pt)+\gamma}}.$
        \vspace{0.5em}
        \item 
        $\displaystyle \expec\!\sbra{\innp{\qti,\elti}_{p^\star}}\leq\expec\!\sbra{\innp{u_{t,i},\lossF_{t,i}}_{p^\star}}+\expec\!\sbra{\frac{\gamma}{\qtz(p^\star)+\gamma}}$.
        \vspace{0.5em}
        \item
        $\displaystyle \expec\!\sbra{\innp{u_{t,i},\lossF^2_{t,i}}_{p^\star}}\leq\expec\!\sbra{\frac{2\log\rbra{e/\eps}}{\qtz(p^\star)+\gamma}}.$
    \end{enumerate}
    \label{lem:alg2-comparator-loss-1}
\end{lemma}
\noindent Note that Lemma \ref{lem:alg2-comparator-loss-1} contains bounds for all random terms in \eqref{eq:alg2-continuous-EXP-bound} of Lemma \ref{lem:alg2-continuous-EXP-bound} except for that of $\langle q,\lossF_{t,i}\rangle_{p^\star}$ terms. To bound $\langle q,\lossF_{t,i}\rangle_{p^\star}$ terms in \eqref{eq:alg2-continuous-EXP-bound}, we exploit the convexity of $\elti(\cdot,p)$. Using techniques adapted from \citet{bubeck2017kernel} and the convex properties given in Lemma \ref{lem:alg2-convex-loss-properties}, we prove the following lemma.
\begin{lemma}
    For any $q\in\Delta(\I_\del)$ with density $q(\cdot)$, the following holds under Assumption \ref{assump:costconcave}:
    \begin{align*}
        2\expec\!\sbra{\innp{q,\lossF_{t,i}}_{p^\star}}\leq \innp{q,\elti}_{p^\star}+\expec\!\sbra{\innp{u_{t,i},\lossF_{t,i}}_{p^\star}}+\expec\!\sbra{\frac{\eta\log(e/\eps)}{\qtz(p^\star)+\gamma}}+7\eps\,\del^{-1}.
    \end{align*}
    \label{lem:alg2-comparator-loss-2}
\end{lemma}
\noindent Using the above two lemmas, we prove Lemma \ref{lem:alg2-comparator-loss-bound} as follows:

\LemAlgTwoComparatorLossBound*
\begin{proof}[Proof of Lemma \ref{lem:alg2-comparator-loss-bound}]
    First, write the expected, summed version of Lemma \ref{lem:alg2-continuous-EXP-bound}: For any $q_1,\dots, q_n \in \Delta(\I_\del)$,
    \begin{align*}
        \expec\Bigg[\sum_{t,i\in[T]\times[n]}\Big(\innp{u_{t,i},\lossF_{t,i}}_{p^\star}-\innp{q_i,\lossF_{t,i}}_{p^\star}\Big)\Bigg]\leq \frac{1}{\eta}\sum_{i\in[n]}\text{KL}(q_i||q')+\frac{\eta}{2}\,\expec\Bigg[\sum_{t,i\in[T]\times[n]}\innp{u_{t,i},\lossF^2_{t,i}}_{p^\star}\Bigg].
    \end{align*}
    Rewriting the above using the linearity of expectation, we obtain
    \begin{align*}
        \sum_{t,i\in[T]\times[n]}\expec\bigg[\innp{u_{t,i},\lossF_{t,i}}_{p^\star}\bigg]-\sum_{t,i\in[T]\times[n]}\expec\bigg[\innp{q_i,\lossF_{t,i}}_{p^\star}\bigg]\leq &\frac{1}{\eta}\sum_{i\in[n]}\text{KL}(q_i||q')\\
        &\hspace{0.3in}+\frac{\eta}{2}\,\sum_{t,i\in[T]\times[n]}\expec\bigg[\innp{u_{t,i},\lossF^2_{t,i}}_{p^\star}\bigg].
    \end{align*}
    Then, apply item 3 in Lemma \ref{lem:alg2-comparator-loss-1} to the above:
    \begin{align*}
        \sum_{t,i\in[T]\times[n]}\expec\bigg[\innp{u_{t,i},\lossF_{t,i}}_{p^\star}\bigg]-\sum_{t,i\in[T]\times[n]}\expec\bigg[\innp{q_i,\lossF_{t,i}}_{p^\star}\bigg]\leq &\frac{1}{\eta}\sum_{i\in[n]}\text{KL}(q_i||q')\\
        &\hspace{0.3in}+\frac{\eta}{2}\,\sum_{t,i\in[T]\times[n]}\expec\bigg[\frac{2\log(e/\eps)}{\qtz(p^\star)+\gamma}\bigg].
    \end{align*}
    Next, we apply Lemma \ref{lem:alg2-comparator-loss-2} to the above to obtain the following:
    \begin{align*}
        \sum_{t,i\in[T]\times[n]}\expec\bigg[\innp{u_{t,i},\lossF_{t,i}}_{p^\star}\bigg]-\sum_{t,i\in[T]\times[n]}\innp{q_i,\elti}_{p^\star}\leq &\frac{2}{\eta}\sum_{i\in[n]}\text{KL}(q_i||q')\\
        &+\eta\,\sum_{t,i\in[T]\times[n]}\expec\bigg[\frac{3\log(e/\eps)}{\qtz(p^\star)+\gamma}\bigg]+7nT\eps\del^{-1}.
    \end{align*}
    Now, we apply item 1 in Lemma \ref{lem:alg2-comparator-loss-1} to the above.
    \begin{align*}
        \sum_{t,i\in[T]\times[n]}\expec\bigg[\frac{\lti\,\ind[\pt=p^\star]}{\qtz(\pt)+\gamma}\bigg]-\sum_{t,i\in[T]\times[n]}\innp{q_i,\elti}_{p^\star}\leq &\frac{2}{\eta}\sum_{i\in[n]}\text{KL}(q_i||q')\\
        &+\eta\,\sum_{t,i\in[T]\times[n]}\expec\bigg[\frac{3\log(e/\eps)}{\qtz(p^\star)+\gamma}\bigg]+7nT\eps\del^{-1}.
    \end{align*}
    By rearranging terms and using the linearity of expectation, we obtain
    \begin{align*}
        \expec\Bigg[\sum_{t\in[T]}\bigg(\frac{\lt\,\ind\!\sbra{\pt=p^\star}}{\qtz(\pt)+\gamma}-\frac{3n\eta\log(e/\eps)}{\qtz(p^\star)+\gamma}\bigg)\Bigg]\leq\sum_{t,i\in[T]\times[n]}\innp{q_i,\elti}_{p^\star} +\frac{2}{\eta}\sum_{i\in[n]}\text{KL}(q_i||q')+7nT\eps\del^{-1},
        \numberthis\label{eq:proof-comparator-loss-bound-1}
    \end{align*}
    where $q_i\in\Delta(\I_\del)$ and $q':=\uniform(\I_\del)$. Now, recall the definition of $\lossH_t$:
    \begin{align*}
        \lossH_t(p):=\frac{1}{n}\sum_{i\in[n]}\frac{\lti\,\ind\!\sbra{\pt=p}}{\qtz(\pt)+\gamma}+\rbra{3\eta\log\frac{e}{\eps}}\rbra{\frac{1}{\gamma}-\frac{1}{\qtz(p)+\gamma}}.
    \end{align*}
    Notice that one can rewrite the LHS of \eqref{eq:proof-comparator-loss-bound-1} as $\expec\!\sbra{\sum_{t\in[T]}n\lossH_t(p^\star)}-\frac{3 n\eta T}{\gamma}\log\frac{e}{\eps}$, which results in the following:
    \begin{align*}
        \expec\Bigg[\sum_{t\in[T]}\lossH_t(p^\star)\Bigg]\leq\frac{1}{n}\sum_{t,i\in[T]\times[n]}\innp{q_i,\elti}_{p^\star}+\frac{2}{n\eta}\sum_{i\in[n]}\text{KL}(q_i||q')+\frac{3 \eta T}{\gamma}\log\frac{e}{\eps}+\frac{7T\eps}{\del},
    \end{align*}
    for any $q_1,\dots, q_n \in \Delta(\I_\del)$.
    
    Next, we bound $\innp{q_i,\elti}_{p^\star}$ as follows: For any $c^\star_i\in\I_\del$ and $s\in[0,1]$, let $q_i$ be the uniform distribution over $(1-s)c^\star_i+s\I_\del$, for some $s>0$ that will be specified later. Then, by $\del^{-1}$--Lipschitz continuity of $\elti(\,\cdot\,,p^\star)$ (see item 1 of Lemma \ref{lem:alg2-convex-loss-properties}),
    \begin{align*}
        \innp{q_i,\elti}_{p^\star}\leq 2s\del^{-1}+\elti(c^\star_i,p^\star).
    \end{align*}
    Since $\text{KL}(q_i||q')\leq\log(1/s)$ for such $q_i$s, choosing $\del:=T^{-1}$, $s:=T^{-2}$ we obtain the following bound.
    \begin{align*}
        \expec\Bigg[\sum_{t\in[T]}\lossH_t(p^\star)\Bigg]-\frac{3 \eta T}{\gamma}\log\frac{e}{\eps}\leq\frac{1}{n}\sum_{t\in[T]}\elt(c^\star,p^\star)+\O\rbra{\eps\,T^2 + \frac{1}{\eta}\log T}.
        \numberthis\label{eq:proof-comparator-loss-bound-2}
    \end{align*}

    Next, to cancel out $\expec[\sum_t\lossH_t(p^\star)]$ term in the above, we use the expected regret bound for $\{\lossH_t\}_{t\in[T]}$. Since $\lossH_t(\cdot)\geq 0$ by construction, we can reuse \eqref{eq:alg1-exp-bounds-2} in Lemma \ref{lem:alg1-exp-bounds} with $\{\lossH_t\}_{t\in[T]}$ for the cost-concave case \eqref{eq:alg2-ht}. Hence,
    \begin{align*}
        \expec\Bigg[\sum_{t\in[T]}\innp{\qtz\,,\lossH_t}\Bigg]-\expec\Bigg[\sum_{t\in[T]}\lossH_t(p^\star)\Bigg]\leq \eta\expec\Bigg[\sum_{t\in[T]}\innp{\qtz\,,\lossH^2_t}\Bigg]+\frac{1}{\eta}\log|\I_K|.
    \end{align*}
    Adding the above with \eqref{eq:proof-comparator-loss-bound-2}, we obtain the following.
    \begin{align*}
        &\expec\Bigg[\sum_{t\in[T]}\innp{\qtz\,,\lossH_t}\Bigg]-\frac{3\eta T}{\gamma}\log\frac{e}{\eps}\\
        &\leq\frac{1}{n}\sum_{t\in[T]}\elt(c^\star,p^\star)+\eta\expec\Bigg[\sum_{t\in[T]}\innp{\qtz\,,\lossH^2_t}\Bigg]+\O\rbra{\eps T^2 + \frac{1}{\eta}\log|\I_K|T}.
        \numberthis\label{eq:proof-comparator-loss-bound-3}
    \end{align*}
    Finally, we derive the stated bound by bounding $\expec[\sum_t\langle\qtz,\lossH^2_t\rangle]$ as follows: Since $\lti\leq 1$ and $\gamma:=\eta\log(e/\eps)$, we have that
    \begin{align*}
        \lossH_t(p):=\frac{1}{n}\sum_{i\in[n]}\frac{\lti\,\ind\!\sbra{\pt=p}}{\qtz(\pt)+\gamma}+\rbra{3\eta\log\frac{e}{\eps}}\rbra{\frac{1}{\gamma}-\frac{1}{\qtz(p)+\gamma}}\leq \frac{\ind[\pt=p]}{\qtz(\pt)}+3.
    \end{align*}
    Then,
    \begin{align*}
        \expec\!\sbra{\innp{\qtz\,,\lossH^2_t}}&\leq\expec\!\sbra{\sum_{p\in\I_K}\qtz(p)\cdot\rbra{\frac{\ind[\pt=p]}{q^2_{t,0}(\pt)}+\frac{6\ind[\pt=p]}{\qtz(\pt)}+9}}\\
        &\leq\expec\!\sbra{\frac{1}{\qtz(\pt)}}+6|\I_K|+9=7|\I_K|+9.
    \end{align*}
    Combining the above and \eqref{eq:proof-comparator-loss-bound-3}, we obtain the stated bound.
\end{proof}

\begin{proof}[Proof of Lemma \ref{lem:alg2-comparator-loss-1}]
    Recall the definition of $\lossF_{t,i}$ here for convenience.
    \begin{align*}
        \lossF_{t,i}(c_i,p):=\frac{\lti\,\ind[\pt=p]}{\qti(\cti|\pt)\rbra{\qtz(\pt)+\gamma}}\cdot K_\eps\Big[u_{t,i}(\,\cdot\,|\,p)\Big](\cti,c_i).
    \end{align*}
    
    For any fixed $p^\star\in\I_K$,
    \begin{align*}
        \innp{u_{t,i},\lossF_{t,i}}_{p^\star}&=\frac{\lti\,\ind\!\sbra{\pt=p^\star}}{\qti(\cti|\pt)\rbra{\qtz(\pt)+\gamma}}\cdot \bunderbrace{\int_{\I_\del} K_\eps\Big[u_{t,i}(\,\cdot\,|\,{p^\star})\Big](\cti\,,c_i)\cdot u_{t,i}(c_i|{p^\star}) dc_i}{\qti(\cti|\,{p^\star})}\\
        &=\frac{\lti\ind\!\sbra{\pt=p^\star}}{\qti(\cti|\pt)\rbra{\qtz(\pt)+\gamma}}\cdot \qti(\cti|p^\star)\\
        &=\frac{\lti\,\ind\!\sbra{\pt=p^\star}}{\qti(\cti|p^\star)\rbra{\qtz(\pt)+\gamma}}\cdot \qti(\cti|p^\star)\\
        &=\frac{\lti\,\ind[\pt=p^\star]}{\qtz(\pt)+\gamma},
    \end{align*}
    hence
    \begin{align*}
        \expec\!\sbra{\innp{u_{t,i},\lossF_{t,i}}_{p^\star}}=\expec\!\sbra{\frac{\lti\,\ind[\pt=p^\star]}{\qtz(\pt)+\gamma}},
    \end{align*}
    which is the first item of the lemma. The second item follows from the below calculations:
    \begin{align*}
        \expec\sbra{\innp{u_{t,i},\lossF_{t,i}}_{p^\star}}&=\expec\sbra{\frac{\lti\,\ind[\pt=p^\star]}{\qtz(\pt)+\gamma}}\\
        &=\expec\sbra{\sum_{p'\in\I_K}\int_{\I_\del}\qtz(p')\qti(c'_i|\,p')\cdot\frac{\elti(c'_i,p')}{\qtz(p')+\gamma}\cdot\ind\!\sbra{p'=p^\star} dc'_i}\\
        &=\expec\sbra{\int_{\I_\del}\qti(c'_i|\,p^\star)\cdot\frac{\elti(c'_i,p^\star) \qtz(p^\star)}{\qtz(p^\star)+\gamma} dc'_i}\\
        &=\expec\sbra{\int_{\I_\del}\qti(c'_i|\,p^\star)\rbra{\elti(c'_i,p^\star)-\gamma\cdot\frac{\elti(c'_i,p^\star)}{\qtz(p^\star)+\gamma}}dc'_i}\\
        &\geq\expec\sbra{\int_{\I_\del}\qti(c'_i|\,p^\star)\rbra{\elti(c'_i,p^\star)-\frac{\gamma}{\qtz(p^\star)+\gamma}}dc'_i}\\
        &=\expec\sbra{\int_{\I_\del}\qti(c'_i|\,p^\star)\cdot\elti(c'_i,p^\star)\,dc'_i}-\expec\sbra{\frac{\gamma}{\qtz(p^\star)+\gamma}}\\
        &=\expec\sbra{\innp{\qti,\elti}_{p^\star}}-\expec\sbra{\frac{\gamma}{\qtz(p^\star)+\gamma}}.
    \end{align*}
    Finally, the third item follows from the below:
    \begin{align*}
        \expec\sbra{\innp{u_{t,i},\lossF^2_{t,i}}_{p^\star}}&=\expec\Bigg[\,\frac{\lti^2\,\ind\!\sbra{\pt=p^\star}}{q^2_{t,i}(\cti|p^\star)\cdot\rbra{\qtz(\pt)+\gamma}^2}  \cdot \bunderbrace{\int_{\I_\del} K^2_\eps\Big[u_{t,i}(\,\cdot\,|\,{p^\star})\Big](\cti\,,c_i)\cdot u_{t,i}(c_i|{p^\star}) dc_i}{:=q^{(2)}_{t,i}(\cti|\,{p^\star})}\,\Bigg]\\
        &=\expec\sbra{\frac{\lti^2\,\ind\!\sbra{\pt=p^\star}\cdot q^{(2)}_{t,i}(\cti|p^\star)}{q^2_{t,i}(\cti|p^\star)\cdot\rbra{\qtz(\pt)+\gamma}^2}}\\
        &\leq\expec\sbra{\frac{\ind\!\sbra{\pt=p^\star}\cdot q^{(2)}_{t,i}(\cti|p^\star)}{q^2_{t,i}(\cti|p^\star)\cdot\rbra{\qtz(\pt)+\gamma}^2}}\tag{as $\lti\leq 1$}\\
        &=\expec\sbra{\sum_{p'\in\I_K}\int_{\I_\del}\frac{\ind\!\sbra{p'=p^\star}\cdot q^{(2)}_{t,i}(c'_i\,|\,p^\star)}{q^2_{t,i}(c'_i\,|\,p^\star)\cdot\rbra{\qtz(p')+\gamma}^2}\cdot \qti(c'_i\,|\,p')\qtz(p')dc'_i}\\
        &=\expec\sbra{\int_{\I_\del}\frac{q^{(2)}_{t,i}(c'_i\,|\,p^\star)\cdot \qti(c'_i\,|\,p^\star)\cdot \qtz(p^\star)}{q^2_{t,i}(c'_i\,|\,p^\star)\cdot\rbra{\qtz(p^\star)+\gamma}^2}dc'_i}\\
        &\leq\expec\sbra{\int_{\I_\del}\frac{q^{(2)}_{t,i}(c'_i\,|\,p^\star)}{\qti(c'_i\,|\,p^\star)\cdot\rbra{\qtz(p^\star)+\gamma}}dc'_i}\\
        &=\expec\sbra{\frac{1}{\qtz(p^\star)+\gamma}\int_{\I_\del}\frac{q^{(2)}_{t,i}(c'_i\,|\,p^\star)}{\qti(c'_i\,|\,p^\star)}dc'_i}\\
        &\leq \expec\sbra{\frac{2\log\rbra{e/\eps}}{\qtz(p^\star)+\gamma}},
    \end{align*}
    where the last line follows as
    \begin{align*}
        \int_{\I_\del}\frac{q^{(2)}_{t,i}(c'_i\,|\,p^\star)}{\qti(c'_i\,|\,p^\star)}dc'_i&=\int_{\I_\del} \frac{1}{\qti(c'_i\,|\,p^\star)}\cdot \rbra{\int_{\I_\del} K^2_\eps\Big[u_{t,i}(\,\cdot\,|\,{p^\star})\Big](c'_i,c_i)\cdot u_{t,i}(c_i\,|\,{p^\star})dc_i}dc'_i\\
        &\leq \int_{\I_\del} \frac{1}{\qti(c'_i\,|\,p^\star)}\cdot \frac{\qti(c'_i\,|\,p^\star)}{\max(|c'_i-\mu|,\eps)}dc'_i\qquad\text{where}\ \mu\in\I_\del,\\
        &= \int_{\I_\del} \frac{dx}{\max(|x-\mu|,\eps)}\\
        &= \int^{\mu-\eps}_\del \frac{dx}{\mu-x} + \int^{\mu+\eps}_{\mu-\eps} \frac{dx}{\eps} + \int^{1-\del}_{\mu+\eps}\frac{dx}{x-\mu}\\
        &\leq 2+\int^{\mu-\eps}_0 \frac{dx}{\mu-x} + \int^{1}_{\mu+\eps}\frac{dx}{x-\mu}\\
        &= 2 + \int^\mu_\eps\frac{dx}{x}+ \int^{1-\mu}_\eps\frac{dx}{x}\\
        &\leq 2\rbra{1+\int^1_\eps \frac{dx}{x}}=2\log \rbra{\frac{e}{\eps}}
    \end{align*}
    where the second line of the above is due to, for any $q\in\Delta(\I_\del)$, $K_\eps[q](x,y)\leq\frac{1}{\max(|x-\mu|,\eps)}$ for some $\mu\in\I_\del$ (see \eqref{eq:kernel}).
\end{proof}

%%%%%%%%%%%%%%%%%%%%%%%%%%%%%%%%%%%%%%%%%%%%%%%%%%%%%%%%%%%%%%%%%%%%%%%%%%%%%%%%%%%%%%%%%%%%%%%%%%%
%\LemAlgTwoStepTwoThree*
\begin{proof}[Proof of Lemma \ref{lem:alg2-comparator-loss-2}]
    First, a straightforward calculation yields
    \begin{align*}
        \expec\sbra{\innp{q,\lossF_{t,i}}_{p^\star}}&=\expec\bigg[\,\frac{\lti\,\ind\!\sbra{\pt=p^\star}}{\qti(\cti|p^\star)\cdot\rbra{\qtz(\pt)+\gamma}}\cdot\bunderbrace{\int_{\I_\del} K_\eps\Big[u_{t,i}(\,\cdot\,|\,{p^\star})\Big](\cti\,,c_i)\cdot q(c_i) dc_i}{K_\eps[u_{t,i}(\,\cdot\,|\,{p^\star})]q(\cti)}\,\bigg]\\
        &=\expec\sbra{\frac{\lti\,\ind\!\sbra{\pt=p^\star}}{\qti(\cti|p^\star)\cdot\rbra{\qtz(\pt)+\gamma}}\cdot K_\eps[u_{t,i}(\,\cdot\,|\,{p^\star})]q(\cti)}\\
        &=\expec\sbra{\sum_{p'\in\I_K}\int_{\I_\del} dc'_i\,\qtz(p')\qti(c'_i|p')\cdot\frac{\elti(c'_i,p')\ind\!\sbra{p'=p^\star}}{\qti(c'_i|p^\star)\rbra{\qtz(p')+\gamma}}\cdot K_\eps[u_{t,i}(\,\cdot\,|\,{p^\star})]q(c'_i)}\\
        &=\expec\sbra{\int_{\I_\del} dc'_i\,\qtz(p^\star)\qti(c'_i|p^\star)\cdot\frac{\elti(c'_i,p^\star)\cdot K_\eps[u_{t,i}(\,\cdot\,|\,{p^\star})]q(c'_i)}{\qti(c'_i|p^\star)\rbra{\qtz(p^\star)+\gamma}}}\\
        &\leq\expec\sbra{\int_{\I_\del} dc'_i\, \elti(c'_i,p^\star)\cdot K_\eps[u_{t,i}(\,\cdot\,|\,{p^\star})]q(c'_i)}\\
        &=\expec\sbra{\innp{K_\eps[u_{t,i}(\,\cdot\,|\,{p^\star})]q,\elti}_{p^\star}}.
    \end{align*}
    Using the definition of $K_\eps[\,\cdot\,]$ \eqref{eq:kernel}, we can rewrite $\innp{K_\eps[u_{t,i}(\,\cdot\,|\,{p^\star})]q,\elti}_{p^\star}$ as follows: Let $\mu:=\expec_{X\sim u_{t,i}(\,\cdot\,|\,{p^\star})}[X]\in[\del,1-\del]$. Then, for $U\sim\uniform([0,1])$ and $X\sim q$,
    \begin{align*}
        \innp{K_\eps[u_{t,i}(\,\cdot\,|\,{p^\star})]q,\elti}_{p^\star}=\expec_{\,U,\,X}\bigg[&\elti\!\rbra{U\mu+(1-U)X\,,\,p^\star}\cdot\ind\Big[\,|X-\mu|\geq\eps\,\Big] \\
        &+\elti\!\rbra{\mu-\eps U\,,\,p^\star}\cdot\ind\Big[\,|X-\mu|<\eps\wedge\mu\geq\eps+\del\,\Big] \\
        &+\elti\!\rbra{\mu+\eps U\,,\,p^\star}\cdot\ind\Big[\,|X-\mu|<\eps\wedge\mu<\eps+\del\,\Big]\bigg].
    \end{align*}
    Using the convexity and $\del^{-1}$--Lipschitz continuity of $\elti$ over $\I_\del$ as shown in Lemma \ref{lem:alg2-convex-loss-properties}, we infer the following facts:
    \begin{align*}
        &\elti\!\rbra{U\mu+(1-U)X,p^\star}\leq U\,\elti\!\rbra{\mu,p^\star} + (1-U)\,\elti\!\rbra{X,p^\star},\\
        &\elti(\mu-\eps U,p^\star),\ \elti(\mu+\eps U,p^\star)\leq \frac{1}{2}\rbra{\elti\!\rbra{\mu,p^\star}+\elti\!\rbra{X,p^\star}}+\frac{3\eps}{2\del},\ \ \text{when}\ |X-\mu|<\eps.
    \end{align*}
    Therefore, 
    \begin{align*}
        \innp{K_\eps[u_{t,i}(\,\cdot\,|\,{p^\star})]q,\elti}_{p^\star}\leq\expec_{\,U,\,X}\bigg[&\rbra{U\,\elti(\mu,p^\star)+(1-U)\,\elti(X,p^\star)}\cdot\ind\Big[\,|X-\mu|\geq\eps\,\Big]+\nonumber\\
        &\frac{1}{2}\rbra{\elti(\mu,p^\star)+\elti(X,p^\star)}\cdot\ind\Big[\,|X-\mu|<\eps\,\Big]\bigg]+\frac{3\eps}{2\del}.
    \end{align*}
    Then, for $\tilde{\mu}=\expec_{X\sim \qti(\,\cdot\,|\,{p^\star})}[X]\in[\del,1-\del]$,
    \begin{align*}
        \innp{K_\eps[u_{t,i}(\,\cdot\,|\,{p^\star})]q,\elti}_{p^\star}&\leq\frac{1}{2}\elti(\mu,p^\star)+\frac{1}{2}\innp{q,\elti}_{p^\star}+\frac{3\eps}{2\del}\\
        &\leq\frac{1}{2}\elti(\tilde{\mu},p^\star)+\frac{1}{2}\innp{q,\elti}_{p^\star}+\frac{7\eps}{2\del} \tag{as $|\mu-\tilde{\mu}|\leq 2\eps$}\\
        &\leq\frac{1}{2}\innp{\qti,\elti}_{p^\star}+\frac{1}{2}\innp{q,\elti}_{p^\star}+\frac{7\eps}{2\del}. \tag{by Jensen's inequality}
    \end{align*}
    Summarizing what we have obtained so far:
    \begin{align*}
        \expec\sbra{\innp{q,\lossF_{t,i}}_{p^\star}}&\leq\expec\sbra{\innp{K_\eps[u_{t,i}(\,\cdot\,|\,{p^\star})]q,\elti}_{p^\star}}\\
        &\leq\frac{1}{2}\expec\sbra{\innp{\qti,\elti}_{p^\star}}+\frac{1}{2}\expec\sbra{\innp{q,\elti}_{p^\star}}+\frac{7\eps}{2\del}\\
        &=\frac{1}{2}\expec\sbra{\innp{\qti,\elti}_{p^\star}}+\frac{1}{2}\innp{q,\elti}_{p^\star}+\frac{7\eps}{2\del},
    \end{align*}
    where the last inequality is due to $\expec\sbra{\innp{q,\elti}_{p^\star}}=\innp{q,\elti}_{p^\star}$, that is, $q$, $\elti$ does not depend on any randomness of algorithm and environment.
    
    Finally, we use item 2 of Lemma \ref{lem:alg2-comparator-loss-1} to bound $\expec\sbra{\innp{\qti,\elti}_{p^\star}}$ in the above inequality in terms of $\expec\Big[\innp{u_{t,i},\lossF_{t,i}}_{p^\star}\Big]$ as follows:
    \begin{align*}
        \expec\sbra{\innp{q,\lossF_{t,i}}_{p^\star}}
        &\leq \frac{1}{2}\,\expec\sbra{\innp{u_{t,i},\lossF_{t,i}}_{p^\star}}+\frac{1}{2}\innp{q,\elti}_{p^\star}+\frac{1}{2}\,\expec\sbra{\frac{\gamma}{\qtz(p^\star)+\gamma}}+\frac{7\eps}{2\del}\\
        &=\frac{1}{2}\,\expec\sbra{\innp{u_{t,i},\lossF_{t,i}}_{p^\star}}+\frac{1}{2}\innp{q,\elti}_{p^\star}+\frac{1}{2}\,\expec\sbra{\frac{\eta\log(e/\eps)}{\qtz(p^\star)+\gamma}}+\frac{7\eps}{2\del},
    \end{align*}
    where the last equality is due to $\gamma:=\eta\log(e/\eps)$.
\end{proof}
%%%%%%%%%%%%%%%%%%%%%%%%%%%%%%%%%%%%%%%%%%%%%%%%%%%%%%%%%%%%%%%%%%%%%%%%%%%%%%%%%%%%%%%%%%%%%%%%%%%
\subsection{Regret lower bound proof}

\ThmLowerBoundConcave*
\begin{proof}\label{app:proof:thm:lb-concave}
    We prove this by showing that any algorithm with regret $R_T$ for the targeted marketing with cost-concave demands can solve a bandit profit-maximization over any sequence of deterministic demands $d_t(p)$ that only depend on, and are non-increasing in, the price $p$, with regret at most $R_T/n$, then we use the result from \citet{kleinberg2003value}.
    
    The reduction is as follows: Define an environment such that, given any $p_t\in[0,1]$, the demands are $d_{t,1}=\cdots=d_{t,n}=d_t(p_t)$. Run the algorithm over this environment and output $p_t$. As $\edti(c,p)=d_t(p)$s are cost-concave and non-increasing in $p$, this is a valid cost-concave environment and the algorithm guarantees
    \begin{align*}
        R_T:=\sup_{(c,p)\in[0,1]^{n+1}}\sum_{t,i\in[T]\times[n]}\rbra{p\,\edti(c_i,p)-c_i}-\expec\Bigg[\sum_{t,i\in[T]\times[n]}\rbra{\pt\dti(\cti,\pt)-\cti}\Bigg].
    \end{align*}
    
    As $\edti(\cti,\pt)=d_t(\pt)$ for all $t\in[T]$ and $i\in[n]$, the first term is $\sup_{p\in[0,1]}n\sum_{t\in[T]}p\cdot d_t(p)$ and the second term is at most $n\cdot\expec\big[\sum_{t\in[T]}\pt\cdot d_t(\pt)\big]$. Therefore,
    \begin{align*}
        \sup_{p\in[0,1]}\sum_{t\in[T]}p\,d_t(p)-\expec\Big[\sum_{t\in[T]}\pt d_t(p_t)\Big]\leq\frac{R_T}{n}.
    \end{align*}
    The LHS of the above is the regret for single-demand bandit profit-maximization. For any algorithm for this problem, there exists a sequence of demands $\{d^\star_t(\cdot)\}_{t\in[T]}$ such that the given algorithm has to suffer $\Omega(T^{\nicefrac{2}{3}})$ regret~\citep{kleinberg2003value}. Hence, the given algorithm has to suffer $\Omega(nT^{\nicefrac{2}{3}})$ regret with respect to the environment where the demand in each market, given $\pt$, is $d^\star_t(p_t)$.
\end{proof}
\section{Details omitted from Section
\ref{sec:otherprobs}}
\label{app:otherprobs}
%%%%%%%%%%%%%%%%%%%%%%%%%%%%%%%%%%%%%%%%%%%%%%%%%%%%%%%%%%%%%%%%%
\subsection{Subscription service}\label{app:subscription}
\paragraph{Setting.}\ \ Consider a firm that runs a subscription service. Let $\dti\in[0,1]$ be the number of new users of type $i\in[n]$ who joined the service during round $t$. Let $\beta_i\in[0,1)$ be the fraction of users who keep the service in each round. Hence, in each round $t$, there are $\sum^t_{s=1}\beta_i^{t-s}d_{s,i}\in\big[0,\frac{1}{1-\beta_i}\big]$ active users. For $\cti\in[0,1]$, let $\frac{\cti}{1-\beta_i}$ be the marketing expenditure spent in round $t$ to attract new users of type $i$. 

Then, the stochastic and expected profits are given as follows:
\begin{align*}
    \profit_{t,i}&:=\rbra{\sum^t_{s=1}\beta_i^{t-s}p_s d_{s,i}}-\frac{\cti}{1-\beta_i},\numberthis\label{eq:subscription-profit}\\
    \eprofit_{t,i}(c_i,p)&:=\rbra{\sum^t_{s=1}\beta_i^{t-s} p\,\ed_{s,i}(c_i,p)}-\frac{c_i}{1-\beta_i}\numberthis\label{eq:subscription-eprofit}.
\end{align*}
Define $\profit_t:=\sum_{i\in[n]}\profit_{t,i}$ and $\eprofit_t(c,p):=\sum_{i\in[n]}\eprofit_{t,i}(c_i,p)$.

\paragraph{Reduction.}\ \ We show that our algorithms guarantee sublinear regrets under Assumption \ref{assump:monotonic} and \ref{assump:costconcave}. One can use Algorithm \ref{alg:1} and \ref{alg:2} to solve this problem via the following reduction: summing over $s\in[T]$, $\sum_{t\in[T]}\profit_{t,i}$ can be rewritten as follows.
\begin{align*}
    \sum_{t\in[T]}\profit_{t,i}&=\sum_{t\in[T]}\sum^t_{s=1}\beta_i^{t-s}p_s d_{s,i}-\frac{1}{1-\beta_i}\sum_{t\in[T]}\cti=\frac{1}{1-\beta_i}\sum_{t\in[T]}\widetilde{\profit}_{t,i},
\end{align*}
where
$\widetilde{\profit}_{t,i}:=\rbra{1-\beta_i^{T-t+1}}\pt\dti-\cti$. We know that $\widetilde{\profit}_{t,i}\in[-1,1]$ as $1-\beta_i^{T-t+1}\in[0,1]$.
The markets enter the total profit with different weights $(1-\beta_i)^{-1}$, so we rescale the term of market $i$ by $(1-\max_{j\in[n]}\beta_j)/(1-\beta_i)\in(0,1]$, which gives every market the common weight $(1-\max_{j\in[n]}\beta_j)^{-1}$.
It is straightforward to see that running our algorithms with the loss
\begin{align*}
    \lti:=\frac{1}{2}\rbra{1-\frac{1-\max_{j\in[n]}\beta_j}{1-\beta_i}\,\widetilde{\profit}_{t,i}}\in[0,1].
    \numberthis\label{eq:lti-subscription}
\end{align*}
 guarantees the same regrets as in the targeted marketing problem, scaled by $(1-\max_{i\in[n]}\beta_i)^{-1}$:

\paragraph{Regret bound.}\ \ Now, we state regret bounds for the subscription service problem.
\begin{restatable}{theorem}{ThmSubscription}
    Let
    \begin{align*}
        R_T:=\sup_{(c,p)\in[0,1]^{n+1}}\sum_{t\in[T]}\eprofit_t(c,p)-\expec\Bigg[\sum_{t\in[T]}\profit_t\Bigg],
    \end{align*}
    where $\profit_t$ and $\eprofit_t(c,p)$ are given by \eqref{eq:subscription-profit} and \eqref{eq:subscription-eprofit} respectively. Under Assumption \ref{assump:monotonic}, running Algorithm \ref{alg:1} with $\lti$ defined in \eqref{eq:lti-subscription} and $\gamma:=\eta$, $K\in\Theta(T^{\nicefrac{1}{4}}), \eta\in\Theta(T^{-\nicefrac{3}{4}})$ guarantees 
    \begin{align*}
        R_T\in\O\rbra{\frac{nT^{\nicefrac{3}{4}}\log T}{1-\max_{i\in[n]}\beta_i}},
    \end{align*}
    Similarly, under Assumption \ref{assump:costconcave}, running Algorithm \ref{alg:2} with the aforementioned $\lti$ guarantees
    \begin{align*}
        R_T\in\O\rbra{\frac{nT^{\nicefrac{2}{3}}\log T}{1-\max_{i\in[n]}\beta_i}},
    \end{align*}
    when $\gamma:=\eta\log(e/\eps), \eps:=T^{-2}, \del:=T^{-1}, K\in\Theta(T^{\nicefrac{1}{3}})$ and $\eta\in\Theta(T^{-\nicefrac{2}{3}})$.
\end{restatable}
%%%%%%%%%%%%%%%%%%%%%%%%%%%%%%%%%%%%%%%%%%%%%%%%%%%%%%%%%%%%%%%%%
\subsection{Promotional credit}\label{app:promotion}
\paragraph{Setting.}\ \ Next, we consider the promotional credit problem, where the firm segments the population into $n$ types indexed by $i\in[n]$. For each round $t$,  $\rti\in[0,1]$ people of each type try the firm's service, where $\rti$ is given exogenously. The firm offers promotional credits $\cti\in[0,1]$ to each type $i$ and price $\pt\in[0,1]$. After using promotional credits, a  $\dti\in[0,1]$ fraction of the $\rti$ customers decide to purchase (or continue) the service. The total profit at round $t$ is
\begin{align*}
    \profit_t:=\sum_{i\in[n]}\profit_{t,i}:=\sum_{i\in[n]}\rti\left(\pt\dti-\cti\right).
    \numberthis\label{eq:promo-profit}
\end{align*}
Let $\erti$ be the expectation of the exogenous $\rti$ and $\edti(c_{t,i},\pt)$ be the conditional expectation of $\dti$ given promotional credit $c_{t,i}$ and price $\pt$. Then, the (conditional) expected profit at round $t$ is
\begin{align*}
    \eprofit_t(c_t,\pt):=\sum_{i\in[n]}\eprofit_{t,i}(c_{t,i},\pt):=\sum_{i\in[n]}\erti\left(\pt\edti(c_{t,i},\pt)-c_{t,i}\right).
    \numberthis\label{eq:promo-eprofit}
\end{align*}

\paragraph{Regret bound.}\ \ As in the targeted marketing problem, we consider regret upper bounds in two cases: when $\edti(\cdot,\cdot)$ is monotonic with respect to both arguments (Assumption \ref{assump:monotonic}), or it is concave to the first argument, the promotional credit, and monotonic with respect to the second argument, the price (Assumption \ref{assump:costconcave}). The expected profit \eqref{eq:promo-eprofit} is a locally rescaled version of an expected profit for the targeted marketing problem: $\pt\dti(\cti,\pt)-\cti$. Hence, it is straightforward to see that our algorithms are still no-regret with respect to the profit defined above without any modification.
\begin{restatable}{theorem}{ThmPromo}
    Let
    \begin{align*}
        R_T:=\sup_{(c,p)\in[0,1]^{n+1}}\sum_{t\in[T]}\eprofit_t(c,p)-\expec\Bigg[\sum_{t \in [T]} \profit_t\Bigg].
    \end{align*}
    Under Assumption \ref{assump:monotonic}, running Algorithm \ref{alg:1} with $\lti$ defined in terms of \eqref{eq:promo-profit} guarantees $R_T\in\O(nT^{\nicefrac{3}{4}}\log T)$ when $\gamma:=\eta, K\in\Theta(T^{\nicefrac{1}{4}}), \eta\in\Theta(T^{-\nicefrac{3}{4}})$. Similarly, under Assumption \ref{assump:costconcave}, running Algorithm \ref{alg:2} with the aforementioned $\lti$ guarantees $R_T\in\O(nT^{\nicefrac{2}{3}}\log T)$ when $\gamma:=\eta\log(e/\eps), \eps:=T^{-2}, \del:=T^{-1}, K\in\Theta(T^{\nicefrac{1}{3}})$ and $\eta\in\Theta(T^{-\nicefrac{2}{3}})$.
\end{restatable}
%%%%%%%%%%%%%%%%%%%%%%%%%%%%%%%%%%%%%%%%%%%%%%%%%%%%%%%%%%%%%%%%%
\subsection{Profit-maximizing A/B test}\label{app:profitmaxAB}
\paragraph{Setting.}\ \ Suppose a firm performs a sequence of experiments where it chooses $M$ marketing alternatives for $n$ population segments to increase demand of a product. Also, the firm wants to perform these experiments while maintaining a common price for the product. Assume that each alternative $m\in[M]$ costs $c_t(m)\in[0,1]$ to implement in round $t$. For example, when $M=2$, the first option ($m=1$) could be to show the product on a non-interactive webpage, while the second option ($m=2$) could be to present the product on an AI-assisted interactive webpage. In this scenario, the second option costs more as it requires more computing resources. 

For each $i\in[n]$ and $t\in[T]$, let $\dti\in[0,1]$ be a random variable representing the normalized demands made by population segment $i$ during round $t$, given a choice of alternative $m_{t,i}\in[M]$ and price $\pt\in[0,1]$. Moreover, let $\edti(m_{t,i},p_t)$ be the conditional expectation of $\dti$ given $m_{t,i}$ and $p_t$.

The firm's total profit at round $t$ is 
\begin{align*}
    \profit_t:=\sum_{i\in[n]}\profit_{t,i}:=\sum_{i\in[n]}\left(\pt\dti-c_t(m_{t,i})\right),
    \numberthis\label{eq:ab-profit}
\end{align*}
and the expected profit at round $t$ given $m_t:=(m_{t,1},\dots,m_{t,n})\in[M]^n$ and $\pt\in[0,1]$ is 
\begin{align*}
    \eprofit_t(m_t,\pt):=\sum_{i\in[n]}\eprofit_{t,i}(m_{t,i},\pt):=\sum_{i\in[n]}\left(\pt\edti(m_{t,i},\pt)-c_t(m_{t,i})\right).
\end{align*}
The firm's regret with respect to the best choice of an alternative for each segment and price is
\begin{align*}
    R_T:=\sup_{(m,p)\in[M]^n\times[0,1]}\sum_{t\in[T]}\eprofit_t(m,p)-\expec\Bigg[\sum_{t\in[T]}\profit_t\Bigg].
\end{align*}

\paragraph{Modification.} This problem can be thought of as a version of targeted marketing with discrete ancillary variables. It turns out that we can still use Algorithm \ref{alg:1} with a small modification to solve this problem. The modification is as follows: rename variables $c_{t,i}\rightarrow m_{t,i}$, $c_i \rightarrow m_i$ in Algorithm \ref{alg:1}. For all $p\in\I_K$, set $q_{1,i}(m_i|p)\leftarrow\uniform([M])$ and let $q_{t,i}(m_i|p)$ be supported over $[M]$ for all $t\in[T]$, and replace $|\I_K|$ by $M$ in the second term of \eqref{eq:alg1-ht}. Losses are given by $\lti:=\frac{1}{2}(1-\profit_{t,i})$ where $\profit_{t,i}$ is defined in  \eqref{eq:ab-profit}.

\paragraph{Regret bound.} With this modification, the regret guarantee for this problem is as follows:
\begin{restatable}{theorem}{ThmABTest}
    Assume that $\edti(m,p)$ is non-increasing in $p$ for all $t$, $i$, and $m$. With the aforementioned modification of Algorithm \ref{alg:1}, $R_T\in\O\big(n\sqrt{M}T^{\nicefrac{2}{3}}\log(MT)\big)$ when $\gamma:=\eta\sqrt{M/K}$, $\eta\in\Theta\big(T^{-\nicefrac{2}{3}}/\sqrt{M}\,\big)$ and $K\in\Theta(T^{\nicefrac{1}{3}})$.
\end{restatable}

\begin{proof}
    By revising the proof of Theorem \ref{thm:alg1-regret} with the aforementioned modification,
    \begin{align*}
        R_T\in\O\rbra{n\eta MKT+\frac{n}{\eta}\log MK+\frac{nT}{K}}.
    \end{align*}
    The choice of $\eta$ and $K$ implies the theorem statement.
\end{proof}

\end{document}